\documentclass{article}

\usepackage{microtype}
\usepackage{graphicx}
\usepackage{subcaption}
\usepackage{booktabs} % for professional tables
\usepackage{multirow}   % For multi-row cells
\usepackage{xcolor}     % For row coloring
\usepackage{colortbl}   % For coloring table cells
\usepackage{amsmath, amssymb, amsfonts, amsthm}
\usepackage{bm} % For bold math vectors
\usepackage{siunitx} % For decimal alignment
\usepackage{etoolbox} % For bolding in S-columns
\usepackage[most]{tcolorbox}
\usepackage{enumitem}
\usepackage{natbib}
\usepackage{fontawesome5} % Essential for the icons
\usepackage{array}
\usepackage{tabularx}
\usepackage{xspace}
\usepackage{CJKutf8}
\usepackage{mdframed}
\usepackage{mathtools}
\usepackage{balance}
\usepackage{wrapfig}
\usepackage{tikz}
\usepackage{pgfplots}
\pgfplotsset{compat=1.18}
\usepackage{rotating} % Required for sidewaystable
\newmdenv[
    linecolor=gray!40,
    backgroundcolor=gray!5,
    roundcorner=5pt,
    innertopmargin=10pt,
    innerbottommargin=10pt,
    innerleftmargin=10pt,
    innerrightmargin=10pt,
    skipabove=10pt,
    skipbelow=10pt
]{examplebox}
\newcommand{\methodname}{\textsc{DibJudge}}

\DeclareMathOperator{\E}{\mathbb{E}}

\definecolor{HeaderBlue}{HTML}{2C3E50} % Dark blue for text
\definecolor{HighlightBlue}{HTML}{EBF2FA} % Soft blue for 'Ours' rows
\definecolor{BestGreen}{HTML}{145A32}    % Dark green for best scores
\definecolor{commentboxcolor}{RGB}{30, 60, 100} % Dark Blue for frame/title
\definecolor{commentboxbg}{RGB}{240, 245, 250}   % Very light blue for background
\definecolor{insightColor}{HTML}{2E86C1} % Professional Blue
\definecolor{alertColor}{HTML}{C0392B}   % Deep Red for bias/warnings
\definecolor{boxBg}{HTML}{F4F6F7}        % Very light gray-blue background

\definecolor{bestgreen}{HTML}{CFE8D0} 
\definecolor{secondyellow}{HTML}{F6F6D5}

\definecolor{bestgreen}{RGB}{192, 228, 205}
\definecolor{secondyellow}{RGB}{227, 237, 189}
\newtcolorbox{insightbox}[1][]{
    enhanced,
    colback=boxBg,
    colframe=insightColor,
    arc=4pt,
    boxrule=1.2pt,
    left=8pt, right=8pt, top=8pt, bottom=8pt,
    fonttitle=\bfseries\sffamily,
    title={\faLightbulb[regular]~~Key Insight},
    attach boxed title to top left={xshift=12pt, yshift*=-8pt},
    boxed title style={
        colback=insightColor,
        frame hidden,
        arc=3pt,
        shadow={1mm}{-1mm}{0mm}{black!20}
    },
    #1
}

\newtcolorbox{biasbox}[1][]{
    enhanced,
    colback=red!5,
    colframe=alertColor,
    arc=4pt,
    boxrule=1.2pt,
    left=35pt, right=8pt, top=8pt, bottom=8pt, % Extra left padding for icon
    overlay={
        \node[anchor=west, color=alertColor] at ([xshift=10pt]frame.west) {\Huge\faSearchMinus};
    },
    fonttitle=\bfseries\sffamily,
    #1
}

\usepackage{hyperref}

\usepackage[accepted]{icml2026}

\usepackage[capitalize,noabbrev]{cleveref}

\theoremstyle{plain}
\newtheorem{theorem}{Theorem}[section]
\newtheorem{proposition}[theorem]{Proposition}

\theoremstyle{definition}

\theoremstyle{remark}

\usepackage[textsize=tiny]{todonotes}

\icmltitlerunning{Mitigating Translationese Bias in Multilingual LLM-as-a-Judge via Disentangled Information Bottleneck}

\begin{document}

\twocolumn[
  \icmltitle{Mitigating Translationese Bias in Multilingual \\ LLM-as-a-Judge via Disentangled Information Bottleneck}

  % It is OKAY to include author information, even for blind submissions: the
  % style file will automatically remove it for you unless you've provided
  % the [accepted] option to the icml2026 package.

  % List of affiliations: The first argument should be a (short) identifier you
  % will use later to specify author affiliations Academic affiliations
  % should list Department, University, City, Region, Country Industry
  % affiliations should list Company, City, Region, Country

  % You can specify symbols, otherwise they are numbered in order. Ideally, you
  % should not use this facility. Affiliations will be numbered in order of
  % appearance and this is the preferred way.
  \icmlsetsymbol{equal}{*}

  \begin{icmlauthorlist}
    \icmlauthor{Hongbin Zhang}{hit,pcl}
    \icmlauthor{Kehai Chen}{hit,pcl}
    \icmlauthor{Xuefeng Bai}{hit}
    \icmlauthor{Youcheng Pan}{pcl}
    \icmlauthor{Yang Xiang}{pcl}
    \icmlauthor{Jinpeng Wang}{keeta}
    \icmlauthor{Min Zhang}{hit,pcl}
  \end{icmlauthorlist}

  \icmlaffiliation{hit}{Institute of Computing and Intelligence, Harbin Institute of Technology, Shenzhen, China}
  \icmlaffiliation{pcl}{Pengcheng Laboratory, Shenzhen, China}
  \icmlaffiliation{keeta}{Keeta AI, Meituan, Beijing, China}

  \icmlcorrespondingauthor{Kehai Chen}{chenkehai@hit.edu.cn}
  % \icmlcorrespondingauthor{Firstname2 Lastname2}{first2.last2@www.uk}

  % You may provide any keywords that you find helpful for describing your
  % paper; these are used to populate the "keywords" metadata in the PDF but
  % will not be shown in the document
  \icmlkeywords{Machine Learning, ICML}

  \vskip 0.3in
]

% this must go after the closing bracket ] following \twocolumn[ ...

% This command actually creates the footnote in the first column listing the
% affiliations and the copyright notice. The command takes one argument, which
% is text to display at the start of the footnote. The \icmlEqualContribution
% command is standard text for equal contribution. Remove it (just {}) if you
% do not need this facility.

% Use ONE of the following lines. DO NOT remove the command.
% If you have no special notice, KEEP empty braces:
\printAffiliationsAndNotice{}  % no special notice (required even if empty)
% Or, if applicable, use the standard equal contribution text:
% \printAffiliationsAndNotice{\icmlEqualContribution}

\begin{abstract}

Large language models (LLMs) have become a standard for multilingual evaluation, yet they exhibit a severe systematic ``\emph{translationese bias}''. 
In this paper, ``translationese bias'' is characterized as LLMs systematically favoring machine-translated text over human-authored references, particularly in low-resource languages. 
We attribute this bias to spurious correlations with (i) latent manifold alignment with English and (ii) cross-lingual predictability.
To mitigate this bias, we propose \textsc{DIBJudge}, a robust fine-tuning framework that learns a minimally sufficient, judgment-critical representation via variational information compression, while explicitly isolating spurious factors into the dedicated bias branch. 
Furthermore, we incorporate a cross-covariance penalty that explicitly suppresses statistical dependence between robust and bias representations, thereby encouraging effective disentanglement.
% To mitigate these issues, we present \textsc{DIBJudge}, a robust fine-tuning framework designed to extract robust representation retains minimal and sufficient judgment-critical information via vector-quantized compression and to explicitly isolate spurious attributes into a dedicated bias branch; In particular, to decouple robust features and bias representations, we penalize mutual dependence to enforce disentanglement.
Extensive evaluations on multilingual reward modeling benchmarks and a dedicated translationese bias evaluation suite demonstrate that the proposed \textsc{DIBJudge} consistently outperforms strong baselines and substantially mitigates translationese bias.
% Code and data are available \href{www.github.com}{here}.
% Extensive evaluations on multilingual reward modeling benchmarks and a specially designed translationese bias evaluation set demonstrate that DIBJudge outperforms strong baselines and effectively reduces translationese bias severity.

\end{abstract}

\section{Introduction}

% The advent of Large Language Models (LLMs) has fundamentally transformed the evaluation paradigms by establishing ``LLM-as-a-Judge'' as a scalable alternative to costly human annotation~\citep{gu2024survey,li2024llmasajudge,li-etal-2025-generation}. In multilingual contexts, these automated judges have become the de facto standard for facilitating Reinforcement Learning from Human Feedback (RLHF) and benchmarking cross-lingual capabilities \citep{son2024mm,ghiasvand-mohammadkhani-beigy-2025-checklist,pombal2025mprometheus,anugraha2025mr3,hada-etal-2024-large,fu2025reliable,doddapaneni-etal-2025-cross}. Consequently, ensuring the accuracy and robustness of LLM judges across languages has become a critical necessity.

The emergence of Large Language Models (LLMs) has revolutionized evaluation paradigms~\citep{gu2024survey,li-etal-2025-generation}, establishing ``LLM-as-a-Judge'' as a standard framework for multilingual assessment~\citep{son2024mm,pombal2025mprometheus,anugraha2025mr3,hada-etal-2024-large,fu2025reliable,doddapaneni-etal-2025-cross}. Consequently, ensuring the accuracy and robustness of these automated judges across diverse languages has become a critical necessity~\citep{padarha2025evaluating,bogavelli2026evaluating}.

% The emergence of Large Language Models (LLMs) has revolutionized evaluation paradigms, establishing ``LLM-as-a-Judge'' as a scalable framework for automated assessment~\citep{gu2024survey,li2024llmasajudge,li-etal-2025-generation} and a cornerstone of multilingual benchmarking~\citep{son2024mm,ghiasvand-mohammadkhani-beigy-2025-checklist,pombal2025mprometheus,anugraha2025mr3,hada-etal-2024-large,fu2025reliable,doddapaneni-etal-2025-cross}. Consequently, ensuring the cross-lingual accuracy and robustness of these automated judges is imperative.

% The paradigm of automated evaluation has witnessed a fundamental transformation, including establishing ``LLM-as-a-Judge'' as a scalable surrogate for human annotation~\citep{zheng2024judging}, facilitating Reinforcement Learning from Human Feedback (RLHF) pipelines~\citep{dubois2024alpacafarm}, and standardizing cross-lingual benchmarking protocols~\citep{huang2023languages}. Consequently, these models have solidified their role as the de facto framework in multilingual contexts, successfully driving reward modeling and assessment across diverse language families.

\begin{figure}[t]
    \centering
    \includegraphics[width=0.95\linewidth]{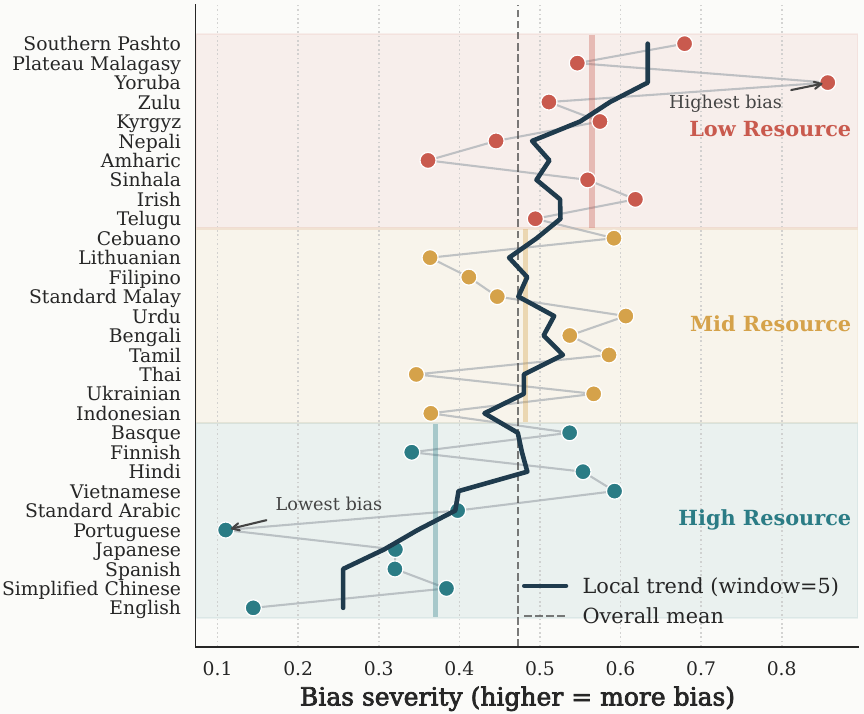}
    \caption{\textbf{Translationese Bias Severity of GPT-4o across languages.} Languages are sorted by resource availability from low (top) to high (bottom). The trend line illustrates the inverse relationship between resource availability and translationese bias.}
    \label{fig:bias_severity_by_resource_ordered}
    \vspace{-6mm}
\end{figure}

However, the reliability of LLM judges is frequently undermined by systematic biases~\citep{wang2025assessing,ye2025justice,gao2025evaluating}, such as position~\cite{shi-etal-2025-judging,wang-etal-2024-large-language-models-fair} and verbosity bias~\citep{saito2023verbosity}. While these limitations are well-studied in English contexts~\citep{chen-etal-2024-humans,zheng2023judging}, specific failure modes within multilingual settings remain underexplored. 

% In this work, we characterize a systematic phenomenon in multilingual contexts termed \emph{translationese bias}: the tendency of LLMs to exhibit a preference for machine-translated content over human-authored text, even when the former is semantically compromised. To investigate this, we conduct a comprehensive evaluation across a diverse spectrum of languages. Our findings demonstrate that this bias is pervasive and disproportionately acute in low-resource settings.

In this paper, we characterize a distinct bias of LLM-as-a-Judge in multilingual contexts, termed \emph{translationese bias}, in which LLMs favor machine-translated content over human-authored reference, even when the former is semantically flawed.
To investigate this bias, we first conduct a comprehensive evaluation across a diverse spectrum of languages. As shown in Figure~\ref{fig:bias_severity_by_resource_ordered}, this bias is not only pervasive but is significantly exacerbated in low-resource languages.
Crucially, our further attribution analysis suggests that LLM judges may conflate generation quality with two potential spurious factors: (a) latent manifold alignment with English, and (b) cross-lingual predictability.
While recent advancements in multilingual LLM judges have yielded promising results~\citep{pombal2025mprometheus,anugraha2025mr3,zhang2025evaluating}, most existing methods remain grounded in standard Supervised Fine-Tuning (SFT). 
However, SFT is susceptible to exploiting spurious correlations~\citep{shuieh2025assessing,gui2025mitigating,chen2025safety}, thereby limiting its efficacy in mitigating translationese bias.
To this end, we propose the Disentangled Information Bottleneck Judge (\textsc{DIBJudge}), a robust fine-tuning framework that explicitly decouples the latent representation into two components: a robust representation that preserves judgment-critical semantic information, and a bias representation that isolates the spurious factors identified above (i.e., latent manifold alignment with English and cross-lingual predictability). We leverage variational information compression to learn a robust, minimally sufficient representation that preserves only information essential for accurate judgment. To further encourage disentanglement between robust and bias representation, we penalize their mutual dependence during training.
% Motivated by the limitations and preliminary findings, we propose the \textbf{D}isentangle \textbf{I}nformation \textbf{B}ottleneck \textbf{J}udge (\textbf{DIBJudge}), a robust fine-tuning framework grounded in the Information Bottleneck (IB) principle~\citep{tishby2000information}. Unlike vanilla SFT, DIBJudge explicitly factorizes the latent representation into two distinct variables: a \textit{robust} variable containing task-critical semantic information, and a \textit{bias} variable capturing spurious attributes. By penalizing the mutual information between these channels while maximizing predictive accuracy, we force the model to discard non-robust shortcuts. 
Extensive experiments on multilingual reward modeling benchmarks, including M-RewardBench~\citep{gureja-etal-2025-rewardbench} and MM-Eval~\citep{son2024mm}, demonstrate that \textsc{DIBJudge} consistently outperforms strong baselines, yielding improved multilingual reward modeling performance. Moreover, evaluations on a dedicated translationese bias suite confirm that \textsc{DIBJudge} substantially mitigates the severity of translationese bias.
% Extensive empirical evaluations on the multilingual reward modeling benchmarks M-RewardBench and MM-Eval demonstrate that DIBJudge consistently outperforms strong baselines, exhibiting superior multilingual reward modeling capabilities. Furthermore, validation on a constructed bias evaluation suite confirms the model's effectiveness in mitigating the severity of translationese bias. 

In summary, we make three key claims: (i) we characterize \emph{translationese bias} in multilingual LLM judges and identify two related spurious factors—latent manifold alignment with English and cross-lingual predictability, (ii) we propose \textsc{DIBJudge}, a robust fine-tuning framework that disentangles judgment-critical semantics from spurious factors, and (iii) we show that \textsc{DIBJudge} consistently outperforms strong baselines on multilingual reward modeling benchmarks while effectively mitigating translationese bias.\footnote{\href{https://anonymous.4open.science/r/dibjudge-E341}{Our anonymous code repo is here.}}

% In summary, we make three key claims: (i) we first characterize \emph{translationese bias} in multilingual LLM judges and identify two potential factors that spurious correlates with judgment; (ii) we propose a fine-tuning framework, named\textsc{DIBJudge}, that disentangles judgment-critical from spurious translation artifacts; and (iii) we demonstrate that \textsc{DIBJudge} significantly outperforms baselines on several benchmarks and effective mitigate translationese bias.

% In summary, we make three key claims:
% \begin{enumerate}
%     \item We identify a distinct bias in multilingual LLM judges—translationese bias—where models favor machine-translated text over human-authored content due to ``English-Anchoring'' and ``Distributional Congruence'', particularly in low-resource settings.
%     \item We propose the DIBJudge, a robust fine-tuning framework grounded in the Information Bottleneck principle that disentangles task-critical semantics from spurious artifacts, preventing the model from overfitting to translationese shortcuts.
%     \item We demonstrate that DIBJudge significantly outperforms strong baselines on benchmarks like M-RewardBench and MM-Eval, while effectively mitigating bias and generalizing robustly to unseen biases and LLMs.
% \end{enumerate}

\section{Preliminary Analysis of Translationese Bias}
\label{sec:preliminary_analysis}

% To systematically investigate \emph{translationese bias}, we frame our preliminary analysis around two central research questions: (i) \textbf{RQ1:} To what extent does translationese bias persist in LLMs with varying language resource availability? (ii) \textbf{RQ2:} What latent linguistic artifacts or heuristic priors are spuriously correlated with translationese bias?
To systematically study translationese bias, we structure our preliminary analysis around two research questions: (i) \textbf{RQ1:} How does translationese bias vary across languages with different levels of resource availability? (ii) \textbf{RQ2:} What kinds of spurious factors are associated with this bias?

% \begin{itemize}[leftmargin=*, noitemsep, topsep=2pt]
%     \item \textbf{RQ1:} To what extent does translationese preference persist in LLMs with varying language resource availability?
%     \item \textbf{RQ2:} What latent linguistic artifacts or heuristic priors are spuriously correlated with this observed bias?
% \end{itemize}

\textbf{Bias Evaluation Protocol.} We construct a controlled translationese bias benchmark derived from BELEBELE~\citep{bandarkar-etal-2024-belebele}, a multilingual reading comprehension dataset spanning 122 languages. Following the language-resource taxonomy of~\citet{joshi-etal-2020-state}, we stratify languages into high-, medium-, and low-resource tiers and sample 10 representative languages per tier. For each language, we evaluate on 200 instances, yielding a balanced benchmark across resource levels. The full list of selected languages is provided in Appendix~\ref{subsec:language_selection}.
We formulate translationese bias evaluation as a pairwise preference task, where an LLM judge compares two candidate responses for the same query: (i) a \emph{human-authored reference} $x_H$, and (ii) a \emph{machine-generated variant} $x_M$ obtained via back-translation to induce translationese artifacts. To avoid position bias, each instance $i$ is evaluated under both forwarding and reverse ordering. Details are available in Appendix~\ref{apdx:bias_benchmark_construction}.

\textbf{Bias Metric Definition.}  Let $y_i \in \{0,1\}$ indicate whether the judge prefers $x_M$ in the forward order, and let $\bar{y}_i \in \{0,1\}$ denote the corresponding preference in the reverse order. We define bias severity $\mathcal{S}_{\mathrm{bias}}$ as the fraction of consistent judgments that favor the machine-generated output:
\vspace{-1mm}
\begin{equation}
\label{eq:bias_metric}
\mathcal{S}_{\mathrm{bias}} =
\frac{\sum_{i=1}^{N} \mathbb{I}\left[y_i = 1 \land \bar{y}_i = 1\right]}
{\sum_{i=1}^{N} \mathbb{I}\left[y_i = \bar{y}_i\right]}.
\end{equation}
\vspace{-3mm}

\begin{figure*}[htbp]
    \centering
    \begin{subfigure}[t]{0.32\textwidth}
        \centering
        \includegraphics[width=\textwidth]{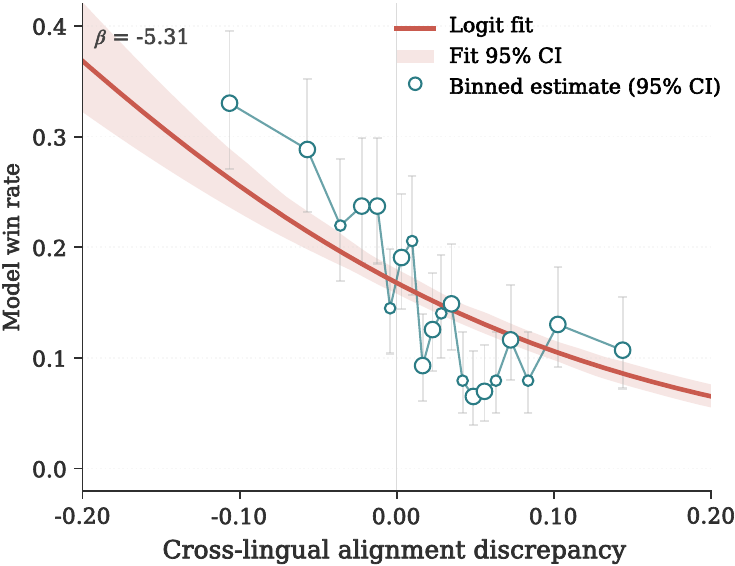}
        \caption{Machine Win Rate vs. CAD}
        \label{fig:cad_binned}
    \end{subfigure}
    \hfill
    \begin{subfigure}[t]{0.32\textwidth}
        \centering
        \includegraphics[width=\textwidth]{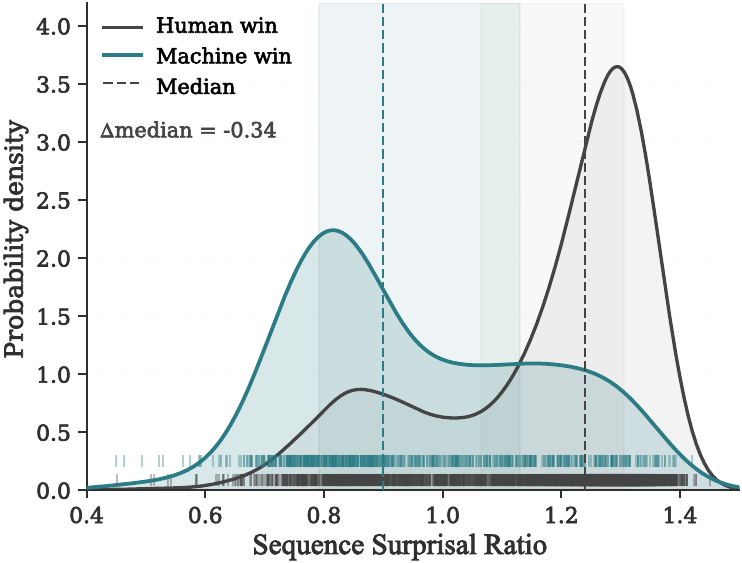}
        \caption{SSR Density (Human vs. Machine Win)}
        \label{fig:csr_density}
    \end{subfigure}
    \hfill
    \begin{subfigure}[t]{0.32\textwidth}
        \centering
        \includegraphics[width=\textwidth]{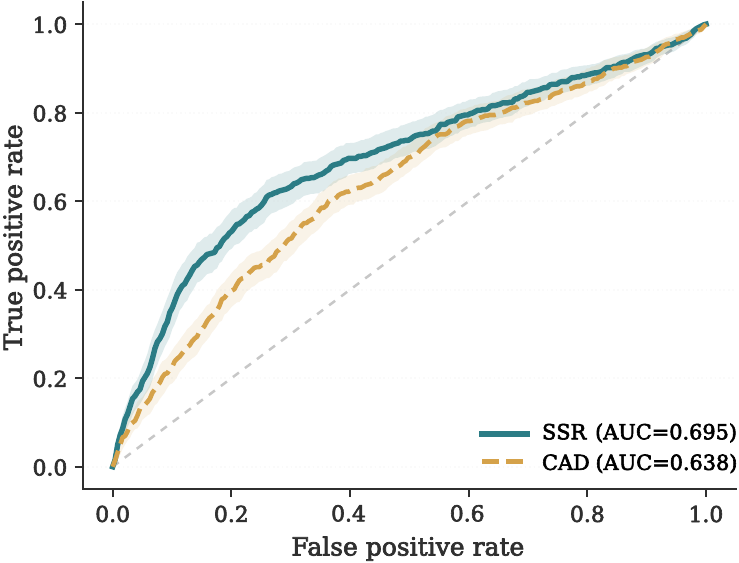}
        \caption{ROC Analysis of CAD and SSR}
        \label{fig:csr_cad_roc}
    \end{subfigure}
    \caption{\textbf{Correlation analysis of judge preference with confounding factors.} 
(a) Machine win rate decreases monotonically as CAD increases, indicating that judge preference spuriously tracks latent manifold isomorphism with English. 
(b) SSR distributions exhibit a clear drift between human-win and machine-win cases, showing that the judge systematically favors higher-likelihood outputs. 
(c) ROC curves confirm that both CAD and SSR reliably predict judge outcomes, reinforcing the attribution that translationese bias is mediated by latent manifold isomorphism with English and high predicative confidence.}
    \label{fig:csr_cad_subfig}
    \vspace{-5mm}
\end{figure*}

% \subsection{Quantifying Bias across Resource Spectrums (RQ1)}
% \label{subsec:rq1_bias_quantification}

% Figure~\ref{fig:bias_severity_by_resource_ordered} illustrates the bias severity of GPT-4o stratified by linguistic resource availability, revealing a pronounced inverse correlation. While high-resource languages (e.g., English, Portuguese) exhibit negligible bias ($\mathcal{S}_{\text{bias}} \approx 0.10$), low-resource languages (e.g., Yoruba) display extreme bias ($\mathcal{S}_{\text{bias}} \approx 0.85$). This trend indicates that diminishing pre-training exposure degrades the model's discriminative ability, causing it to systematically favor the structural regularity of machine outputs over natural human variation. Consequently, translationese preference bias is disproportionately concentrated in low-resource settings, implying that current LLM-as-a-Judge paradigms disadvantage under-represented languages by mistaking structural artifacts for quality.

\vspace{-2mm}
\subsection{Quantifying Bias across Resource Levels (RQ1)}
\label{subsec:rq1_bias_quantification}
% Figure~\ref{fig:bias_severity_by_resource_ordered} shows translationese bias severity of GPT-4o stratified by language resource level. We observe a clear inverse trend: high-resource languages (e.g., English, Portuguese) exhibit low bias ($\mathcal{S}_{\mathrm{bias}}\!\approx\!0.10$), whereas low-resource languages (e.g., Yoruba) show substantially higher bias ($\mathcal{S}_{\mathrm{bias}}\!\approx\!0.85$). This concentration of bias in low-resource settings indicates that LLM-as-a-Judge evaluation pipelines may be less reliable for underrepresented languages.
Figure~\ref{fig:bias_severity_by_resource_ordered} illustrates the translationese bias severity of GPT-4o across varying language resource levels. We observe two salient patterns: first, translationese bias is pervasive across the entire linguistic spectrum; second, there is a distinct inverse correlation between resource availability and the magnitude of bias. Specifically, while high-resource languages exhibit minimal bias, low-resource languages demonstrate significantly elevated severity. These findings expose translationese bias as a severe yet previously neglected failure mode in multilingual LLM judges, which critically undermines evaluation reliability and disproportionately compromises under-resourced languages.

\vspace{-2mm}
\subsection{Attribution Analysis of Translationese Bias (RQ2)}
\label{subsec:rq2_spurious_factors}
Due to the scarcity of high-quality native resources~\citep{QIN2025101118,huang2024survey}, multilingual LLMs are typically pre-trained on English-dominated corpora~\citep{kreutzer2022quality,weber2024redpajama} and subsequently adapted using translated or synthetic data~\citep{muennighoff-etal-2023-crosslingual,zhang-etal-2020-improving}. 
Accordingly, we hypothesize that translationese bias stems from spurious correlations induced across these two stages: (i) latent manifold alignment with English, where non-English representations are implicitly aligned to an English-centric latent space during pre-training; and (ii) cross-lingual predictability, where the judge over-relies on probability heuristics that favor the statistical patterns of machine-translated text, potentially amplified by exposure to translated or synthetic data during fine-tuning. However, causally attributing this bias to particular data mixtures remains non-trivial given the opacity and heterogeneity of LLM training pipelines~\citep{lai-etal-2025-survey}.
% Accordingly, we hypothesize that translationese bias arises from the combined effects of these two stages: (i) \emph{latent manifold isomorphism}, where non-English representations are shaped toward an English-centric embedding space during pre-training; and (ii) \emph{confidence-driven bias}, driven by the high predictability of synthetic regularities encountered during fine-tuning. However, causally attributing this bias to specific data mixtures remains non-trivial given the opacity and heterogeneity of contemporary training pipelines.

To address this, we introduce two measurable latent metrics that serve as quantitative proxies for these two factors: 
(i) \emph{Language Alignment Score (LAS)}, defined as the degree to which a representation is geometrically aligned with an English latent manifold: $\mathrm{LAS}(x) = \frac{1}{L} \sum_{l=1}^{L} \cos\big(\mathbf{h}_l(x), \mathbf{c}_{\mathrm{en},l}\big),$
where $x$ is input sequence, $\mathbf{h}_l(x)$ is the layer-$l$ hidden representation of $x$, $\mathbf{c}_{\mathrm{en},l}$ is the English centroid at layer $l$, and $L$ is the total number of layers. 
% Higher LAS indicates stronger anchoring to the English semantic pivot.
% (i) \textbf{English-Anchoring (EA)}, defined as the degree to which multilingual representations are geometrically anchored to the English manifold as a semantic pivot. Formally, given a sequence $x$, we define EA as the mean layer-wise cosine similarity between the hidden states $\mathbf{h}_l$ and a precomputed global English centroid $\mathbf{c}_{\text{en},l}$:
% \begin{equation}
%     \text{EA}(x) = \frac{1}{L} \sum_{l=1}^{L} \cos(\mathbf{h}_l, \mathbf{c}_{\text{en},l}),
% \end{equation}
% where $L$ denotes the total number of layers.
(ii) \emph{Cross-lingual Sequence Surprisal (CSS)}, defined as the length-normalized negative log-likelihood of a target sequence $x$ of $T$ tokens, conditioned on its English translated context $x_{\mathrm{en}}$: $\mathrm{CSS}(x) = -\frac{1}{T} \sum_{t=1}^{T} \log P(x_t \mid x_{\mathrm{en}}, x_{<t}).$
To answer RQ2, we then investigate the extent to which the distributional divergence of these metrics between $x_H$ and $x_M$ correlates with machine win rate: (i) \emph{Cross-lingual Alignment Discrepancy (CAD)}, $\text{CAD} = \text{LAS}(x_{H}) - \text{LAS}(x_M),$ where $\text{CAD} < 0$ implies that $x_M$ exhibits closer alignment to the English latent space than $x_H$ does. (ii) \emph{Sequence Surprisal Ratio (SSR)}: $\text{SSR} = \frac{\text{CSS}(x_{M})}{\text{CSS}(x_{H})},$ where $\text{SSR} < 1$ indicates that $x_M$ is more cross-lingual predictable by the model relative to $x_H$.

As shown in Figure~\ref{fig:csr_cad_subfig}, LLM judges exhibit strong correlations with the introduced latent metrics. 
Specifically, Figure~\ref{fig:cad_binned} reveals a negative association between CAD and the machine win rate, suggesting that the judge favors outputs that align closely with the English manifold. Meanwhile, Figure~\ref{fig:csr_density} demonstrates a pronounced distributional shift in SSR: machine-generated outputs preferred by the judge cluster significantly at lower SSR. This pattern indicates a bias toward sequences with high statistical cross-lingual predictability. The robustness of the observed correlations is further corroborated by the ROC analysis in Figure~\ref{fig:csr_cad_roc}, demonstrating meaningful discriminative power of these features. Collectively, these results ascribe the observed \emph{translationese bias} to two confounding factors: latent manifold alignment with English and cross-lingual predictability.

\begin{figure*}[t]
    \centering
    \includegraphics[width=.9\linewidth]{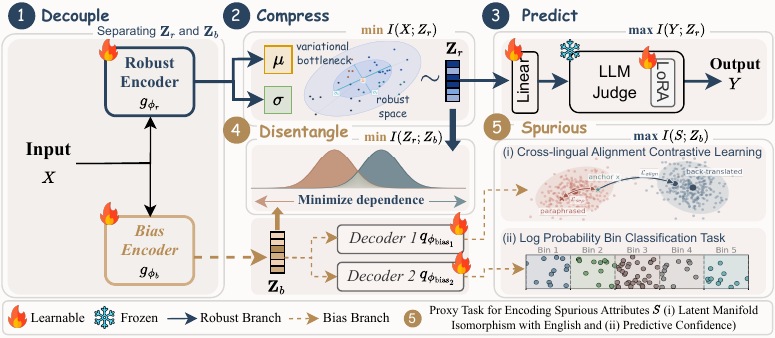}
    \caption{Overview of our \textsc{DIBJudge}, which grounded in Equation~\ref{eq:dib_objective}. (1) employs a robust encoder $g_{\phi_r}$ and a bias encoder $g_{\phi_b}$ to separate the input $X$ into robust representations $\mathbf{Z}_r$ and bias representations $\mathbf{Z}_b$. (2) introduces a variational bottleneck to minimize the mutual information $I(X; Z_r)$. (3) leverages the compressed $\mathbf{Z}_r$ through LLM Judge optimized using LoRA~\citep{hu2022lora} to generate the final output $Y$ by maximizing $I(Y; Z_r)$. (4) ensures feature independence by minimizing the dependence $I(Z_r; Z_b)$ between the robust and bias branches. (5) explicitly captures spurious attributes $S$ within $\mathbf{Z}_b$ by maximizing $I(S; Z_b)$ through two proxy tasks: (i) cross-lingual alignment contrastive learning and (ii) predictive confidence estimation via log-probability bin classification.}
    \label{fig:method_overview}
    \vspace{-3mm}
\end{figure*}

\vspace{-2mm}
\section{Disentangled Information Bottleneck Judge}
\label{sec:method}

% Motivated by the previous analysis, where we identified empirical correlations between LLM judges and spurious factors (e.g., \textit{English Anchoring} and \textit{Distributional Congruence}), we seek a trustworthy LLM judge whose decisions are driven by judge-relevant evidence rather than spurious cues in the noisy input. To this end, we propose the \textbf{D}isentangled \textbf{I}nformation \textbf{B}ottleneck \textbf{Judge} (DIB-Judge), a framework that combines the Information Bottleneck (IB) principle with disentangled representation learning. 

% Building on prior empirical analyses that revealed correlations between LLM-based judgments and spurious factors such as \emph{English Anchoring} and \emph{Distributional Congruence}, we aim to develop a robust LLM-as-a-Judge grounded in semantic evidence rather than superficial artifacts. To this end, we propose the \textbf{Disentangled Information Bottleneck Judge} (\textsc{dib-judge}), a framework that integrates the Information Bottleneck (IB) principle with disentangled representation learning to isolate judge-relevant signals from input noise.

To mitigate the spurious correlations identified in \S~\ref{subsec:rq2_spurious_factors}, we propose the Disentangled Information Bottleneck Judge (\methodname{}), as illustrated in Figure~\ref{fig:method_overview}. 
By explicitly disentangling these spurious factors, \methodname{} learns a compressed representation that retains sufficient, robust features essential for accurate quality assessment.

\vspace{-2mm}
\subsection{Preliminaries}

% \textbf{Mutual Information.}
% Mutual Information quantifies the statistical dependence between two random variables, $A$ and $B$. Given a joint distribution $p(a,b)$ and marginal distributions $p(a)$ and $p(b)$, MI is formally defined as:
% \begin{equation}
% \vspace{-1mm}
%     I(A;B) \triangleq \mathbb{E}_{p(a,b)}\left[\log \frac{p(a,b)}{p(a)p(b)}\right].
%     \label{eq:mutual_information}
% \vspace{-1mm}
% \end{equation}
% This metric measures non-linear correlation: $I(A;B) = 0$ iff $A$ and $B$ are independent ($A \perp B$), whereas larger values indicate a stronger association between the variables.

% \textbf{Information Bottleneck Principle.}
% The Information Bottleneck (IB) principle~\cite{tishby2000information,alemi2017deep} facilitates learning a compacted representation $Z$ that contains information \emph{sufficient} for a target $Y$ while being \emph{minimal} with respect to the input $X$, which is achieved by minimizing the following objective:
% \begin{equation}
%     \mathcal{L}_{\mathrm{IB}} = - I(Y; Z) + \beta I(X; Z),
%     \label{eq:ib_objective}
% \end{equation}
% where $\beta \ge 0$ is a multiplier balancing the trade-off between prediction and compression.
% However, as noted by~\citet{liu2022learning}, compression can inadvertently discard robust semantic evidence in favor of simpler, yet highly predictive, spurious signals. Consequently, relying solely on information compactness cannot guarantee the robustness of judgment.

\textbf{Mutual Information.}
Mutual Information (MI) quantifies the statistical dependence between two random variables, $A$ and $B$. Given the joint distribution $p(a,b)$ and marginals $p(a), p(b)$, MI is formally defined as :
\vspace{-1mm}
\begin{equation}
    I(A;B) \triangleq \mathbb{E}_{p(a,b)} \left[ \log \frac{p(a,b)}{p(a)p(b)} \right].
    \label{eq:mi_def}
\end{equation}
\vspace{-3mm}

\textbf{Information Bottleneck Principle.}
The Information Bottleneck (IB) principle~\citep{tishby2000information, alemi2017deep} seeks a compressed representation $Z$ that is \textit{sufficient} for a target task $Y$ while remaining \textit{minimal} with respect to the input $X$. This is formalised by minimizing the Lagrangian:
% \begin{equation}
%     \mathcal{L}_{\text{IB}} = - I(Y; Z) + \beta I(X; Z),
%     \label{eq:ib_objective}
% \end{equation}
$\mathcal{L}_{\text{IB}} = - I(Y; Z) + \beta I(X; Z),$
where $\beta \ge 0$ governs the trade-off between prediction and compression. However, compression may discard robust semantic features in favor of simpler spurious correlations~\citep{liu2022learning,zhang2024choose}. Thus, relying solely on information compactness cannot guarantee the robustness.

\vspace{-2mm}
\subsection{Disentangled Information Bottleneck Objective}
\label{subsec:dib_obj}
To prevent LLM judges from exploiting spurious shortcuts solely through information compression, we are inspired by the idea of disentangled representation learning~\citep{wang2024disentangled} to refine vanilla IB. Let $S$ be a spurious variable, $Z_r$ be a \emph{relevant} variable encoding features necessary for predicting the target $Y$, and $Z_b$ be a \emph{bias} variable serving as a dedicated ``sink'' to absorb $S$. We formalize the Disentangled Information Bottleneck Objective as:
\begin{equation}
\begin{aligned}
\mathcal{L}_{\mathrm{DIB}}
&=
\underbrace{- I(Y; Z_r)}_{\text{Prediction}}
+
\underbrace{\beta\, I(X; Z_r)}_{\text{Compression}} \\
&\quad
-\underbrace{\gamma\, I(S; Z_b)}_{\text{Bias Capture}}
+
\underbrace{\lambda\, I(Z_r; Z_b)}_{\text{Disentanglement}} .
\end{aligned}
\label{eq:dib_objective}
\end{equation}
The first two terms apply a vanilla IB constraint restricted to the robust channel $Z_r$; the third term makes $Z_b$ informative about spurious attributes $S$; and the final term penalizes dependence between $(Z_r, Z_b)$, encouraging $Z_r$ to exclude spurious correlations that are explicitly routed into $Z_b$.
Directly optimizing this objective is computationally intractable due to the difficulty of estimating mutual information in high-dimensional spaces~\citep{liu2024protecting}. To address this, we derive tractable variational surrogate objectives as follows:

\textbf{Decouple Robust and Bias Representation.} 
Using separate encoders $g_{\phi_r}$ and $g_{\phi_b}$, we decompose the input $X$ into a robust representation $\mathbf{Z}_r = g_{\phi_r}(X) \in \mathbb{R}^{T \times d}$ and a bias representation $\mathbf{Z}_b = g_{\phi_b}(X) \in \mathbb{R}^{T \times d},$ where $T$ denotes the sequence length and $d$ the feature dimension. 
We leverage $\mathbf{Z}_r$ for task prediction while using $\mathbf{Z}_b$ only during training.

\textbf{Compression via Variational Information Constraints.}
% To facilitate the compression objective $I(X;Z_r)$, we follow~\citet{alemi2017deep} to leverage the variational information bottleneck to strictly limit the mutual information between the robust channel and input via variational inference~\citep{alemi2017deep,kingma2019introduction}. 
% From an information-theoretic perspective, we penalize the divergence between the learned posterior distribution and a fixed prior. By enforcing the posterior to approximate an uninformative prior, we impose an upper bound on the mutual information.
% Specifically, we penalize the divergence between the variational posterior and a fixed prior, yielding an upper bound on $I(X;Z_r)$.
To facilitate compression, we leverage the Variational Information Bottleneck~\citep{alemi2017deep}, which imposes an upper bound on $I(X;Z_r)$ via variational inference. 
% We achieve this by penalizing the Kullback-Leibler divergence between the learned variational posterior and a fixed prior. 
% By forcing the posterior to approximate this uninformative prior, we effectively constrain the information flow from the input $X$ to the latent $Z_r$.

\begin{proposition}
\label{prop:vi_capacity}
Let $Z_r$ be a continuous random variable, with variational posterior $q_{\phi}(Z_r|X)$ and fixed prior $p(Z_r)$. Then $I(X; Z_r) \leq \mathbb{E}_{x \sim p(X)} \left[ D_{\mathrm{KL}}(q_{\phi}(Z_r|x) \| p(Z_r)) \right].$
\end{proposition}

Guided by Proposition~\ref{prop:vi_capacity} (proved in Appendix.~\ref{app:proof_1}), we can constrain $I(X;Z_r)$ via penalizing the KL divergence between the variational posterior and fixed prior. Accordlingly, we adopt a standard Gaussian prior $p(Z_r) = \mathcal{N}(\mathbf{0}, \mathbf{I})$ and parameterize the variational posterior $q_{\phi}(\mathbf{z}_{r, t}|x)$ at each time step $t$ as a multivariate Gaussian $\mathcal{N}(\boldsymbol{\mu}_t, \boldsymbol{\sigma}^2_t)$. 
The resulting compression objective, defined as the average KL divergence over the sequence length $T$ and feature dimension $d$, is derived as follows (details in Appendix.~\ref{app:compress_derivation}):

\vspace{-6mm}
{\small
\begin{equation}
\label{eq:vi_loss}
\begin{aligned}
\mathcal{L}_{\mathrm{compress}}
&= \frac{1}{T}\sum_{t=1}^{T}
D_{\mathrm{KL}}\!\left(
\mathcal{N}\!\left(\boldsymbol{\mu}_{t}, \boldsymbol{\sigma}^{2}_{t}\right)
\,\big\|\,
\mathcal{N}\!\left(\mathbf{0}, \mathbf{I}\right)
\right) \\
&= -\frac{1}{2T}\sum_{t=1}^{T}\sum_{j=1}^{d}
\Bigl(1 + \log \sigma^{2}_{t,j} - \mu^{2}_{t,j} - \sigma^{2}_{t,j}\Bigr).
\end{aligned}
\end{equation}}
\vspace{-4mm}

To allow for backpropagation, we sample the latent representation $\mathbf{z}_{r, t}$ using the reparameterization trick~\citep{NIPS2015_bc731692}: $\mathbf{z}_{r, t} = \boldsymbol{\mu}_t + \boldsymbol{\sigma}_t \odot \boldsymbol{\epsilon}$, where $\boldsymbol{\epsilon} \sim \mathcal{N}(\mathbf{0}, \mathbf{I}).$

\textbf{Variational Mutual Information Maximization.}
\label{sec:variational_maximization}
The objective (Eq.~\ref{eq:dib_objective}) necessitates maximizing mutual information along two disentangled pathways: the task-predictive term $I(Y; Z_r)$ and the bias-capturing term $I(S; Z_b)$. We then maximize a variational lower bound on the mutual information guided by Proposition~\ref{prop:variational_bound} (proved in Appendix~\ref{app:proof_2}). 
% Specifically, we prove the following Proposition in Appendix~\ref{app:proof_2}.
% \begin{proposition}
% \label{prop:variational_bound}
% Let $U$ be a target variable and $V$ be a latent variable. For any variational distribution $q_\theta(U|V)$ parameterized by a decoder $\theta$, the mutual information $I(U; V)$ satisfies $I(U; V) \ge \E_{U, V} \left[ \log q_\theta(U|V) \right] + H(U),$ where $H(U)$ is the constant marginal entropy of the target. 
% Thus, maximizing mutual information is equivalent to minimizing the negative log-likelihood (cross-entropy) of the decoder.
% \end{proposition}
\begin{proposition}
\label{prop:variational_bound}
Let $U$ and $V$ be random variables with joint distribution $p(U,V)$.  
For any variational conditional distribution $q_{\theta}(U \mid V)$, the mutual information satisfies $I(U; V) \ge \mathbb{E}_{(U,V) \sim p(U,V)}\!\left[ \log q_{\theta}(U \mid V) \right] + H(U),$
where $H(U)$ denotes the marginal entropy of $U$.
\end{proposition}

For the robust pathway, we treat the LLM judge $f_{\text{judge}}$ as the variational decoder $q_{\theta}$. We condition the generation on a sequence formed by concatenating the instruction embeddings $\mathbf{E}_{\text{inst}}$ with the sampled robust representation $\mathbf{Z}_r$. The task loss is defined as:

\vspace{-5mm}
{\small
\begin{equation}
\label{eq:loss_task}
\mathcal{L}_{\text{task}} = \E_{X, Y} \left[ - \sum_{t=1}^{|Y|} \log f_{\text{judge}}(Y_t \mid [\mathbf{E}_{\text{inst}}; \mathbf{Z}_r], Y_{<t}) \right].
\end{equation}}
\vspace{-5mm}

We employ lightweight MLP decoders $q_{\psi_{\text{bias}}}$ and minimize the negative log-likelihood of the spurious attribute ${S}$ given the bias representation $\mathbf{Z}_b$ to facilitate the encoding of spurious information into the bias pathway:
\begin{equation}
    \mathcal{L}_{\text{bias}} = \E_{X, S} \left[ - \log q_{\psi_{\text{bias}}}(S\mid \mathbf{Z}_b) \right].
    \label{eq:loss_bias}
\end{equation}
We operationalize the identified spurious factors through two designed proxy tasks. 
First, we address \textit{latent-manifold alignment with English} using a cross-lingual contrastive learning objective. 
Second, we estimate \textit{cross-lingual predictability} via a log-probability bin classification task. 
\textbf{Disentanglement via Cross-Covariance Penalty.}
% While directly minimizing the mutual information $I(Z_r; Z_b)$ is computationally prohibitive, in the high-dimensional feature spaces characteristic of LLMs, representation distributions can be reasonably approximated as Gaussian~\citep{lee2018deep,hron2020infinite}. 
% Although direct minimization of the mutual information $I(Z_r; Z_b)$ is generally intractable, high-dimensional representations in LLMs admit accurate Gaussian approximations~\citep{lee2018deep, hron2020infinite}.
Directly minimizing the mutual information $I(Z_r; Z_b)$ is generally intractable.  
However, in high-dimensional regimes typical of LLMs, representation distributions are often well-approximated by Gaussian statistics~\citep{lee2018deep, hron2020infinite}.
Under this Gaussian assumption, minimizing mutual information reduces to minimizing the cross-covariance between latent variables~\citep{cover1999elements, HYVARINEN2000411}. We formalize this relationship in Proposition~\ref{prop:xcov_bound} (proof provided in Appendix~\ref{app:proof_3}).
% Under this Gaussian assumption, minimizing mutual information reduces to minimizing the cross-covariance between latent variables~\citep{cover1999elements, HYVARINEN2000411}. 
% Specifically, we prove the following proposition in Appendix~\ref{app:proof_3}.

\begin{proposition}
\label{prop:xcov_bound}
Let $Z_r$ and $Z_b$ be jointly Gaussian random vectors with marginal covariance matrices
$\Sigma_r$ and $\Sigma_b$, and cross-covariance $\Sigma_{rb}$.
Define the normalized cross-covariance matrix as
$C = \Sigma_r^{-1/2}\Sigma_{rb}\Sigma_b^{-1/2}$.
Provided the spectral norm $\|C\|_2$ is sufficiently small, the mutual information admits the following second-order expansion:
\[
I(Z_r; Z_b) = \frac{1}{2}\|C\|_F^2 + o\left(\|C\|_F^2\right), \quad \text{as } \|C\|_2 \to 0.
\]
% Let $Z_r,Z_b$ be random vectors with a joint multivariate Gaussian distribution, $\Sigma_r$ ) and ( $\Sigma_b$ ) are the marginal covariance matrices of ($Z_r$) and ($Z_b$), and ( $\Sigma*{rb}$ ) is their cross-covariance matrix. Then $I(Z_r; Z_b) ;\le; \frac{1}{2},\big| \Sigma_r^{-1/2},\Sigma_{rb},\Sigma_b^{-1/2} \big|*F^2$
% The mutual information $I(Z_r; Z_b)$ is upper-bounded by the mutual information of a joint Gaussian distribution with the same covariance structure. Therefore, minimizing the squared Frobenius norm of the cross-covariance matrix, $\|\Sigma_{rb}\|_F^2$, minimizes this Gaussian upper bound.
\end{proposition}
Accordingly, we adopt the cross-covariance penalty as a computationally efficient surrogate for disengtanglement term in Eq.~\ref{eq:dib_objective}.
Given centered mini-batch representations $\bar{\mathbf{Z}}_r, \bar{\mathbf{Z}}_b \in \mathbb{R}^{N\times d}$, we compute the empirical cross-covariance matrix $\hat{\Sigma}_{rb} = \frac{1}{N-1} \bar{\mathbf{Z}}_r^\top \bar{\mathbf{Z}}_b$.
We then minimize the squared Frobenius norm of $\hat{\Sigma}_{rb}$ to suppress correlations:
% \vspace{-1mm}
% \begin{equation}
% \label{eq:loss_xcov}
% \mathcal{L}_{\text{disc}} = ||\Sigma_{rb}||_F^2 = \sum_{i=1}^d \sum_{j=1}^d (\Sigma_{rb})_{ij}^2.
% \end{equation}
% \vspace{-1mm}
% Formally, we compute the cross-covariance matrix over a mini-batch of size $N$. We first center the representations by subtracting their respective batch means, yielding $\bar{\mathbf{Z}}_r\in\mathbb{R}^{N\times d}$ and $\bar{\mathbf{Z}}_b\in\mathbb{R}^{N\times d}$. The cross-covariance matrix $\Sigma_{rb} \in \mathbb{R}^{d \times d}$ is then computed as: $\Sigma_{rb} = \frac{1}{N-1} \bar{\mathbf{Z}}_r^\top \bar{\mathbf{Z}}_b.$ To enforce disentanglement, we minimize the squared Frobenius norm of $\Sigma_{rb}$, formulated as:

\vspace{-1mm}
\begin{equation}
    \label{eq:loss_xcov}
    \mathcal{L}_{\text{disc}} = \|\hat{\Sigma}_{rb}\|_F^2 = \sum_{i=1}^d \sum_{j=1}^d (\hat{\Sigma}_{rb})_{ij}^2.
\end{equation}
\vspace{-1mm}

In practice, feature-wise normalization ensures that $\mathbf{Z}_r$ and $\mathbf{Z}_b$ have approximately unit variance along each dimension~\citep{ba2016layer}.
This objective penalizes second-order dependencies, thereby encouraging statistical independence in the learned representations~\citep{zbontar2021barlow}.

\textbf{Overall Learning Objective.}
We optimize \methodname{} end-to-end by minimizing a weighted sum of tractable objectives derived above. Concretely, the final training objective is
\begin{equation}
\mathcal{L}
=
\mathcal{L}_{\mathrm{task}}
+
\beta\,\mathcal{L}_{\mathrm{compress}}
+
\gamma\,\mathcal{L}_{\mathrm{bias}}
+
\lambda\,\mathcal{L}_{\mathrm{disc}},
\label{eq:final_objective}
\end{equation}
where the weights $\beta,\gamma,\lambda$ control the accuracy--compression--bias-capture--independence trade-off.

\begin{table*}[t]
    \centering
    \caption{
        Performance evaluation on multilingual reward benchmarks. 
        \textbf{Bold} indicates the best performance, and \underline{underlined} indicates the second-best. 
        Statistical significance compared to the best baseline is denoted by $^{\dagger}$ ($p < 0.05$) and $^{\ddagger}$ ($p < 0.01$).
    }
    \label{tab:combined-results}

    % Professional spacing
    \renewcommand{\arraystretch}{1.2} 
    \setlength{\tabcolsep}{6pt}
    \small
    
    % Configure siunitx
    % separate-uncertainty: converts "80 +- 0.2" to "80 \pm 0.2"
    % table-format: reserves space for "88.88" and the uncertainty
    \sisetup{
        separate-uncertainty = true,
        detect-weight = true, 
        mode = text
    }

    % Soft highlight colors
    \definecolor{bestgreen}{RGB}{192, 228, 205}
    \definecolor{secondyellow}{RGB}{227, 237, 189}

    \begin{tabular}{
        @{}
        l
        S[table-format=2.2(2)]
        S[table-format=2.2(2)]
        S[table-format=2.2(2)]
        @{}
    }
        \toprule
        & {\textbf{m-RewardBench}} & {\textbf{RewardBench}} & {\textbf{MM-Eval}} \\
        \textbf{Model} & {(Avg. 23 langs)} & {(English)} & {(Avg. 18 lang)} \\
        \midrule
        
        \multicolumn{4}{l}{\textit{Proprietary Models}} \\
        \hspace{3mm}GPT-4o~\citep{hurst2024gpt} & 85.75 +- 0.42 & 85.96 +- 0.35 & 71.85 +- 0.81 \\
        \hspace{3mm}Gemini-2.5-Flash~\citep{comanici2025gemini} & 88.06 +- 0.49 & 88.83 +- 0.47 & 77.47 +- 0.76 \\
        \midrule
        
        \multicolumn{4}{l}{\textit{General Open Models}} \\
        \hspace{3mm}Qwen2.5-3B-Instruct~\citep{qwen2025qwen25technicalreport} & 66.97 +- 1.12 & 68.99 +- 1.05 & 57.99 +- 1.20 \\
        \hspace{3mm}Qwen2.5-7B-Instruct~\citep{qwen2025qwen25technicalreport} & 77.89 +- 0.89 & 78.59 +- 0.91 & 65.64 +- 0.95 \\
        \hspace{3mm}Qwen3-4B~\citep{yang2025qwen3} & 85.06 +- 0.65 & 87.54 +- 0.55 & 80.85 +- 0.68 \\
        \hspace{3mm}Qwen3-8B~\citep{yang2025qwen3} & 86.12 +- 0.52 & 88.81 +- 0.48 & 82.20 +- 0.60 \\
        \midrule

        \multicolumn{4}{l}{\textit{Multilingual Open Reward Models}} \\
        \hspace{3mm}Nemotron-Multi-49B~\citep{wang-etal-2025-helpsteer3} & {\cellcolor{secondyellow}{\underline{88.83 $\pm$ 0.35}}} & 89.71 +- 0.31 & 76.31 +- 0.55 \\
        \hspace{3mm}M-\textsc{Prometheus} 3B~\citep{pombal2025mprometheus} & 68.45 +- 0.98 & 69.79 +- 0.92 & 64.17 +- 1.10 \\
        \hspace{3mm}M-\textsc{Prometheus} 7B~\citep{pombal2025mprometheus} & 78.03 +- 0.85 & 76.69 +- 0.78 & 69.38 +- 0.88 \\
        \hspace{3mm}mR3-Qwen3-4B~\citep{anugraha2025mr3} & 87.21 +- 0.45 & 89.75 +- 0.38 & 82.55 +- 0.52 \\
        \hspace{3mm}mR3-Qwen3-8B~\citep{anugraha2025mr3} & 88.58 +- 0.41 & {\cellcolor{secondyellow}{\underline{90.10 $\pm$ 0.40}}} & {\cellcolor{secondyellow}{\underline{85.29 $\pm$ 0.45}}} \\
        \hspace{3mm}Think-as-Locals 7B~\citep{zhang2025evaluating} & 84.51 +- 0.60 & 88.79 +- 0.52 & 72.95 +- 0.70 \\
        \midrule
        
        \multicolumn{4}{l}{\textit{\textbf{Ours}}} \\
        % Note: We use { } brackets around the superscript logic {$...$} so siunitx ignores it
        \hspace{3mm}DIBJudge-Qwen3-4B & 89.84 +- 0.28{$^{\dagger}$} & {\cellcolor{bestgreen}{\textbf{90.32 $\pm$ 0.25}}} & 85.16 +- 0.33 \\
        \hspace{3mm}DIBJudge-Qwen3-8B & {\cellcolor{bestgreen}{\textbf{91.37 $\pm$ 0.22}}$^{\ddagger}$} & 91.01 +- 0.20{$^{\dagger}$} & {\cellcolor{bestgreen}{\textbf{87.53 $\pm$ 0.28}}$^{\ddagger}$} \\
        \bottomrule
    \end{tabular}
    \vspace{-3mm}
\end{table*}

\vspace{-2mm}
\section{Experiments}

% We first introduce the multilingual reward modeling benchmark, followed by the specific configurations for downstream tasks and baseline models. The section concludes by presenting a comparison of DIBJudge with these baselines.

% \subsection{Task and Setups}

% \textbf{Datasets.} To evaluate the effectiveness of LLM judges and assess their retention of English capabilities, we utilize RewardBench~\citep{lambert2024rewardbench}. This standard benchmark consists of 3,000 pairwise comparisons across four categories: Chat, Chat Hard, Reasoning, and Safety. To assess multilingual performance, we employ M-RewardBench~\citep{gureja-etal-2025-rewardbench}—a translated adaptation of RewardBench covering 23 languages—and MM-Eval~\citep{son2024mm}. MM-Eval encompasses 18 languages and distinguishes itself by prioritizing native-speaker data over translations. Beyond standard categories, MM-Eval includes Linguistics (e.g., identifying homophones) and Language Hallucination (detecting code-switching). The primary evaluation metric is accuracy. For RewardBench, we report the average of category scores. For the multilingual benchmarks, we calculate the micro-average for each language and report the mean across all languages. Detailed results by category and language appear in Appendix [MASK].
\textbf{Evaluation Benchmarks.} To evaluate the effectiveness of LLM judges across multilingual contexts, we utilize three primary reward modeling benchmarks~\citep{lambert2024rewardbench,son2024mm,gureja-etal-2025-rewardbench} selected to ensure a balanced consideration of the following aspects: a) reasoning and safety alignment across diverse conversational contexts, b) performance across 23 distinct languages, c) the distinction between translated content and native-speaker data. Our primary evaluation metric is accuracy, reported as the category average and the mean of language-specific micro-averages. More details are provided in Appendix~\ref{app:eval_bench_details}.

\textbf{Training settings.}
We adopt the same training corpus as mR3~\citep{anugraha2025mr3} and fine-tune using LoRA~\citep{hu2022lora}. All experiments are optimized with the Adam optimizer~\citep{kingma2014adam} using a learning rate of $1\times10^{-4}$ and a maximum sequence length of $16384$. Further implementation details are provided in Appendix~\ref{app:training_settings}.

% \subsection{Baselines}
% \textbf{Baselines.} We evaluate DIBJudge against two primary baseline categories, selected to represent both general-purpose capabilities and state-of-the-art reward modeling.
% We compare against proprietary models, including GPT-4o and Gemini-2.5-Flash, to benchmark against closed-source upper bounds. In the open-source domain, we evaluate the Qwen2.5 and Qwen3 families (Instruct variants). As Qwen3 serves as the backbone for our method, these baselines allow us to quantify the specific gains attributable to our training recipe rather than the base model capabilities alone.
% We also compare against leading open-source multilingual reward models and judges. This includes the heavy-weight Nemotron-Multilingual-49B, the M-Prometheus family (3B and 7B), and mR3 (based on Qwen3). We also include Think-as-Locals, reporting the maximum score across language-specific variants to ensure a rigorous upper-bound comparison. 
\textbf{Baselines.}
We evaluate \textsc{DIBJudge} against proprietary general-purpose models (GPT-4o~\citep{hurst2024gpt}, Gemini-2.5-Flash~\citep{comanici2025gemini}) and open-source general-purpose LLMs (Qwen2.5/3~\citep{qwen2025qwen25technicalreport,yang2025qwen3}). Since Qwen3 is our backbone, these comparisons isolate gains from our training recipe beyond base model capacity. We additionally benchmark multilingual reward models/judges, including Nemotron-Multilingual-49B~\citep{wang-etal-2025-helpsteer3}, M-Prometheus (3B/7B)~\citep{pombal2025mprometheus}, mR3~\citep{anugraha2025mr3}, and Think-as-Locals (7B)~\citep{zhang2025evaluating}.

% \subsection{Main Results}
% \textbf{Main Results.}
% Table~\ref{tab:combined-results} summarizes the mean accuracy and standard deviation across selected benchmarks over three independent runs, with statistical significance determined via pairwise $t$-tests. On m-RewardBench, \textsc{DIBJudge}-Qwen3-8B establishes a new state-of-the-art for open-weights models (\SI{91.37}{\percent}). It significantly outperforms the substantially larger Nemotron-Multilingual-49B (\SI{88.83}{\percent}, $p < 0.01$) and exceeds its backbone-matched baseline, mR3-Qwen3-8B (\SI{88.58}{\percent}), validating the efficacy of the proposed method. Regarding generalizability, \textsc{DIBJudge}-8B achieves \SI{91.01}{\percent} on the English-centric RewardBench, statistically surpassing the previous best result (mR3, \SI{90.10}{\percent}, $p < 0.05$). This suggests that our approach enhances multilingual performance without compromising general reward modeling capabilities. Furthermore, \textsc{DIBJudge} consistently outperforms proprietary systems. These results underscore the efficiency of specialized, smaller-scale models relative to large, general-purpose systems for targeted evaluation tasks.
\paragraph{Main Results.}
Table~\ref{tab:combined-results} reports the mean accuracy and standard deviation across benchmarks over three independent runs, with statistical significance assessed using pairwise $t$-tests. 
On m-RewardBench, \textsc{DIBJudge}-Qwen3-8B establishes a new SOTA among open-weight models, significantly outperforming both its backbone-matched counterpart and a substantially larger multilingual baseline. These results confirm the effectiveness of the proposed approach. 
In terms of generalization, \textsc{DIBJudge}-Qwen3-8B achieves superior performance on the English-centric RewardBench, statistically surpassing prior leading methods. This indicates that the proposed method improves multilingual reward modeling without degrading performance on monolingual benchmarks. Detailed results are provided in Appendix~\ref{app:comprehensive_results}.
% Moreover, \textsc{DIBJudge} consistently outperforms proprietary systems, highlighting the efficiency of specialized, smaller-scale models compared to large, general-purpose alternatives for targeted evaluation tasks.

\vspace{-2mm}
\section{Analysis}
\label{sec:analysis}

To substantiate the theoretical claims proposed in this work, we conduct a series of targeted experiments designed to rigorously validate the efficacy and internal mechanics of \textsc{DIBJudge}. Our analysis focuses on verifying that the proposed information disentanglement objective translates into tangible performance gains and interpretable latent structures. We organize this empirical investigation around five core research questions:
(i) \textbf{RQ1 (Bias Mitigation):} To what extent does \textsc{DIBJudge} effectively mitigate translationese bias across languages with varying resource availability?
(ii) \textbf{RQ2 (Utility Trade-off):} How does the information bottleneck constraint shape the Pareto Frontier between bias mitigation and downstream task utility?
(iii) \textbf{RQ3 (Disentanglement):} Do the learned latent representations geometrically disentangle semantic content from translationese artifacts, as theoretically hypothesized?
(iv) \textbf{RQ4 (Generalization):} Does the model exhibit robustness against unseen bias types (e.g., length bias) that were not explicitly included in the spurious proxy task?
(v) \textbf{RQ5 (Ablation Study):} How do the distinct components of the DIB objective (Eq.~\ref{eq:dib_objective}) and spurious proxy task contribute to bias mitigation and reward modeling utility?
Additional analyses are deferred to the appendix, including studies of proxy-task design (Appx.~\ref{sec:ablation_analysis_proxy}), sensitivity to CAD and SSR (Appx.~\ref{sec:cad_sensitivity},~\ref{sec:bias_spectrum}), comparisons with alternative compression and disentanglement mechanisms (Appx.~\ref{sec:ablation_bottleneck},~\ref{sec:ablation_disentangle_mechanisms}), and linear probing to assess information leakage (Appx.~\ref{sec:linear_probing}).

% \begin{itemize}[leftmargin=*, nosep]
%     \item \textbf{RQ1 (Bias Mitigation):} To what extent does \textsc{DIBJudge} effectively mitigate translationese bias across languages with varying resource availability?
%     \item \textbf{RQ2 (Utility Trade-off):} How does the information bottleneck constraint shape the Pareto frontier between bias mitigation and downstream task utility?
%     \item \textbf{RQ3 (Disentanglement):} Do the learned latent representations geometrically disentangle semantic content from translationese artifacts, as theoretically hypothesized?
%     \item \textbf{RQ4 (Generalization):} Does the model exhibit robustness against unseen bias types (e.g., length bias) that were not explicitly included in the spurious proxy task?
%     \item \textbf{RQ5 (Ablation Study):} How do the distinct components of the Disentangled Information Bottleneck objective contribute to bias mitigation and reward modeling utility?
% \end{itemize}

% \subsection{RQ1: Efficacy in Multilingual Bias Mitigation}
% \label{sec:bias_validation}
\textbf{RQ1: Efficacy in Translationese Bias Mitigation.}
% To address \textbf{RQ1}, we expand the preliminary evaluation from \S\ref{sec:preliminary_analysis} to a comprehensive suite of multilingual domains and datasets.
% \textbf{Benchmarks and Setup.} 
We extend the preliminary bias evaluation (\S~\ref{sec:preliminary_analysis}) to a broader suite of domains and datasets.
We evaluate performance on three diverse benchmarks: \textsc{Belebele} (machine reading comprehension)~\citep{bandarkar-etal-2024-belebele}, \textsc{Aya}~\citep{singh-etal-2024-aya} (open-ended instruction following), and \textsc{XL-Sum}~\citep{hasan-etal-2021-xl} (summarization). This selection allows us to assess translationese bias across constrained formats and realistic, open-ended interactions. Human-authored references are used as ground-truth targets, while negative samples (rejected responses) are generated via back-translation as described in \S~\ref{sec:preliminary_analysis}. 
To investigate the impact of data scarcity, we stratify languages into High-, Mid-, and Low-Resource tiers ($n=10$ languages per tier). We benchmark \textsc{DIBJudge} against three baselines: the Base model, Vanilla SFT, and a Vanilla IB variant. We quantify efficacy using the \textit{Bias Severity} metric ($\mathcal{S}_{\text{bias}}$) defined in Equation~\ref{eq:bias_metric}. More details in Appendix.~\ref{apdx:bias_benchmark_construction}

\begin{figure}[t]
    \centering
    \includegraphics[width=1\linewidth]{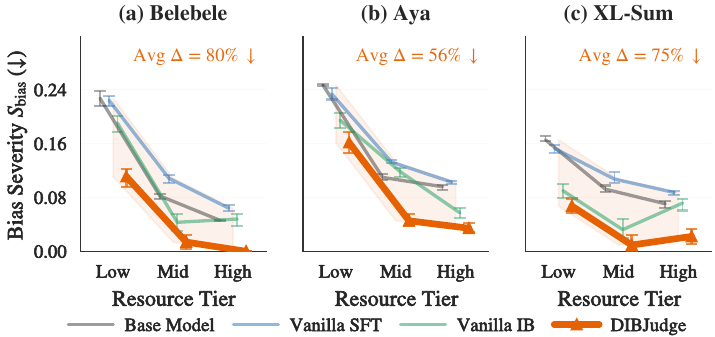}
    \caption{\textbf{Bias severity across resource tiers.}
    $\mathcal{S}_{\text{bias}}$ (lower is better) on
    \textsc{Belebele}, \textsc{Aya}, and \textsc{XL-Sum}.
    \textsc{DIBJudge} reduces bias across all tiers, with average reductions of
    $80\%$, $56\%$, and $75\%$, and the strongest improvements in Low-Resource settings.
    Error bars show std over 3 runs; Avg $\Delta$ is relative to \textit{Vanilla SFT}.
    }
    \label{fig:dibjudge-bias-severity-tiers}
    \vspace{-7mm}
\end{figure}

% \textbf{Results.} 
% Figure~\ref{fig:dibjudge-bias-severity-tiers} demonstrates that \textsc{DIBJudge} effectively mitigates translationese bias across diverse resource settings. On \textsc{Belebele}, the method achieves an $80\%$ average reduction in bias severity, bringing $\mathcal{S}_{\text{bias}}$ to near-zero for the Mid- and High-Resource tiers. Similarly, for the generative tasks \textsc{Aya} and \textsc{XL-Sum}, we observe reductions of $56\%$ and $75\%$, respectively. Notably, \textsc{DIBJudge} significantly diminishes the disparity between resource tiers. While Vanilla SFT exhibits marked bias in the Low-Resource tier (e.g., $\mathcal{S}_{\text{bias}} \approx 0.24$ on \textsc{Aya}), \textsc{DIBJudge} substantially suppresses these values. These findings confirm that our method specifically targets the bias amplification disproportionately affecting under-represented languages, rather than merely enhancing general instruction-following capabilities.

Figure~\ref{fig:dibjudge-bias-severity-tiers} demonstrates the efficacy of \textsc{DIBJudge} in mitigating translationese bias across diverse language resource levels. 
On the benchmark \textsc{Belebele}, DIBJudge achieves a drastic reduction in bias severity, approaching near-zero levels across the mid- and high-resource tiers. 
This trend extends to generative tasks such as \textsc{Aya} and \textsc{XL-Sum}, where we observe consistent bias suppression. 
Crucially, \textsc{DIBJudge} significantly reduces disparity across resource tiers; whereas vanilla SFT retains marked bias in low-resource settings, our approach effectively dampens these spurious correlations. 
These findings confirm that \textsc{DIBJudge} targets the bias amplification that disproportionately affects underrepresented languages, rather than merely enhancing general instruction-following capabilities.

\begin{figure}[htbp]
    \vspace{-1mm}
    \centering
    \includegraphics[width=3.0in, height=1.65in]{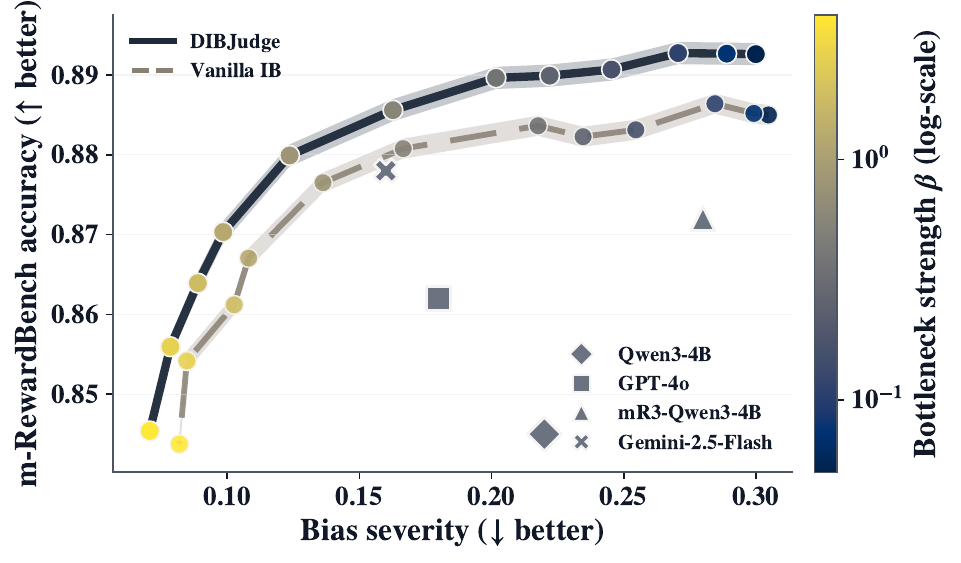}
    % \caption{\textbf{Pareto Frontier of Robustness vs. Utility.} The trade-off between Bias Severity ($\downarrow$ better) and m-RewardBench Accuracy ($\uparrow$ better). The color gradient indicates the bottleneck strength $\beta$ (log-scale). The solid curve represents \textsc{DIBJudge}, which consistently outperforms the \textsc{Vanilla IB} baseline (dashed line). Moreover, \textsc{DIBJudge} dominates state-of-the-art models (e.g., GPT-4o) by achieving a superior balance of low bias and high accuracy.}
    \caption{\textbf{Bias--utility Pareto Frontier.}
Trade-off between Bias Severity ($\downarrow$; x-axis) and m-RewardBench accuracy ($\uparrow$; y-axis).
Each point corresponds to a bottleneck strength $\beta$ (log-scaled, color-coded).
The resulting Pareto frontiers are traced by \textsc{DIBJudge} (solid) and the \textsc{Vanilla IB} baseline (dashed).
\textsc{DIBJudge} consistently achieves higher accuracy at comparable bias levels across $\beta$, yielding a uniformly superior bias--utility trade-off.
Markers indicate representative SOTA models, which \textsc{DIBJudge} outperforms in terms of lower bias and higher accuracy.
}
    \label{fig:pareto_frontier}
    \vspace{-1mm}
\end{figure}

% \subsection{RQ2: The Robustness-Utility Trade-off}
% \label{sec:pareto_frontier}

% To address \textbf{RQ2}, we investigate the tension between robustness and utility by modulating the information bottleneck regularization coefficient, $\beta$. Specifically, we aim to characterize the Pareto frontier of this trade-off. We posit that increasing $\beta$ enforces tighter constraints that filter spurious correlations; while this enhances robustness, it may compromise the semantic granularity required for high-utility reward modeling.

% Figure~\ref{fig:pareto_frontier} visualizes this trade-off, plotting Bias Severity (robustness) against m-RewardBench Accuracy (utility). The results demonstrate that \textsc{DIBJudge} establishes a superior Pareto frontier compared to the \textsc{Vanilla IB} baseline, indicating that our specific discretization mechanism more effectively disentangles robust features from noise. Furthermore, the \textsc{DIBJudge} frontier strictly dominates strong static baselines, including GPT-4o, Gemini-2.5-Flash, and mR3-Qwen3-4B. Notably, for any fixed level of bias severity exhibited by these baselines, our method consistently achieves higher reward modeling accuracy. This confirms that discrete information bottlenecks can successfully mitigate bias without the severe utility degradation observed in standard baselines.
% \subsection{RQ2: The Robustness-Utility Trade-off}
% \label{sec:pareto_frontier}
\textbf{RQ2: The Robustness-Utility Trade-off.}
We investigate the tension between robustness and utility by modulating the coefficient $\beta$ of the compression term in Equation~\ref{eq:dib_objective}. Specifically, we aim to characterize the Pareto Frontier of this trade-off.
Figure~\ref{fig:pareto_frontier} illustrates this dynamic by plotting bias severity ($\mathcal{S}_{\text{bias}}$) against m-RewardBench accuracy. \textsc{DIBJudge} achieves a consistently better Pareto frontier than the \textsc{Vanilla IB} baseline, indicating that it learns a more compact and robust representation without discarding key semantic features.
Furthermore, \textsc{DIBJudge} strictly dominates its base model and strong proprietary baselines (e.g., GPT-4o, Gemini-2.5-Flash), consistently achieving higher accuracy across all fixed levels of bias severity. These findings confirm that the proposed method mitigates translationese bias without substantial utility degradation.

\textbf{RQ3: Disentanglement of Latent Representations.}
% \subsection{RQ3: Disentanglement of Latent Representations}
% \label{sec:visualization}
% We visualize the geometric structure of the learned factorized representations using t-SNE~\citep{tsne}. We extract bias representations $\mathbf{Z}_b$ and robust representations $\mathbf{Z}_r$ from the held-out bias evaluation set containing both Human-written and Machine-translated texts.
% As shown in Figure~\ref{fig:tsne_separation}, the latent spaces exhibit divergent structures. The bias space (Fig.~\ref{fig:tsne_bias}) displays distinct clustering based on text origin, confirming that $\mathbf{Z}_b$ captures translationese artifacts. Conversely, the robust space (Fig.~\ref{fig:tsne_robust}) shows a high degree of overlap between domains. This indistinguishability provides visual evidence that $\mathbf{Z}_r$ has achieved invariance to translationese features while retaining the underlying semantic content necessary for the task.
We visualize the geometry of the learned representations using t-SNE~\citep{tsne}, extracting bias ($\mathbf{Z}_b$) and robust ($\mathbf{Z}_r$) features from a held-out evaluation set comprising human and machine-translated texts. As shown in Figure~\ref{fig:tsne_separation}, the latent spaces exhibit divergent topologies. The bias space (Fig.~\ref{fig:tsne_bias}) forms distinct clusters based on text origin, confirming that $\mathbf{Z}_b$ encodes translationese artifacts. Conversely, the robust space (Fig.~\ref{fig:tsne_robust}) demonstrates substantial domain overlap. This phenomenon demonstrates that $\mathbf{Z}_r$ achieves invariance to translationese artifacts, effectively disentangling them from the underlying semantic content.

\begin{figure}[t]
    \centering
    \begin{subfigure}[b]{0.48\linewidth}
        \centering
        \includegraphics[width=\linewidth]{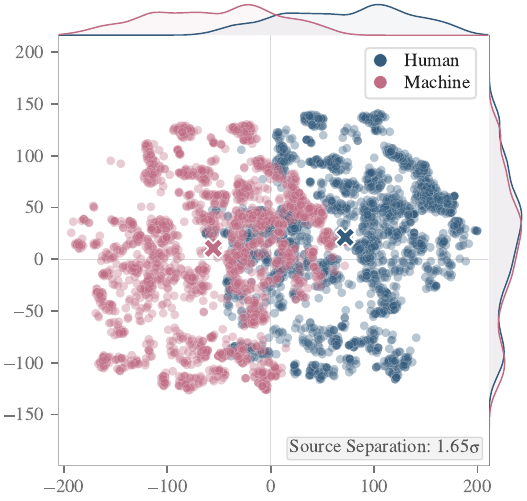}
        % \framebox[\linewidth]{\rule{0pt}{100pt} \textbf{Clustered by Origin}}
        \caption{Bias Representations ($\mathbf{Z}_b$)}
        \label{fig:tsne_bias}
    \end{subfigure}
    \hfill
    \begin{subfigure}[b]{0.48\linewidth}
        \centering
        \includegraphics[width=\linewidth]{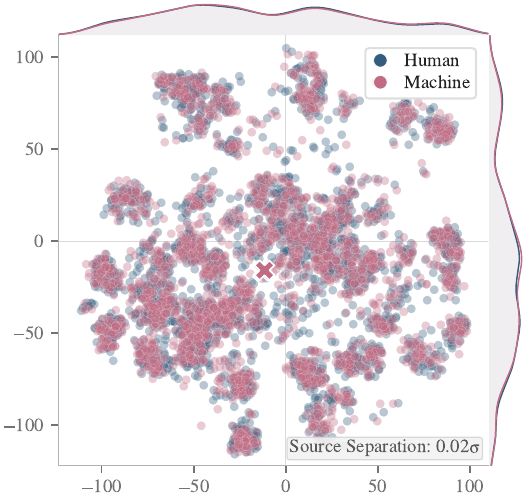}
        % \framebox[\linewidth]{\rule{0pt}{100pt} \textbf{Homogeneous Mixing}}
        \caption{Robust Representations ($\mathbf{Z}_r$)}
        \label{fig:tsne_robust}
    \end{subfigure}
    \caption{\textbf{Visualization of Latent Representation Disentanglement.} t-SNE projections of embeddings for Human (Blue) vs. Machine (Red) text. \textbf{(a)} $\mathbf{Z}_b$ clearly separates domains. \textbf{(b)} $\mathbf{Z}_r$ shows a mixed distribution, corroborating domain invariance.}
    \label{fig:tsne_separation}
    \vspace{-1mm}
\end{figure}

\textbf{RQ4: Zero-Shot Generalization to Unseen Biases.}
We evaluate the generalization capability of \textsc{DIBJudge} against biases not encountered during training—specifically length~\citep{saito2023verbosity} and self-preference~\citep{wataoka2024selfpreference}—using a held-out subset of \textit{Skywork-Reward-Preference-80K}~\citep{liu2024skyworkrewardbagtricksreward}. Length bias is quantified via the Spearman rank correlation ($\rho$) between response length and predicted ratings, while self-preference is evaluated using bias severity ($\mathcal{S_{\text{bias}}}$) as defined in Eq.~\ref{eq:bias_metric}, specifically favoring the model's own generations. As detailed in Table~\ref{tab:generalization_biases}, \textsc{DIBJudge} demonstrates superior robustness compared to vanilla SFT and IB baselines. Our method significantly reduces the correlation between response length and response bias and minimizes self-preference bias. These results suggest that \textsc{DIBJudge} successfully learns to filter superficial heuristics (e.g., verbosity) rather than merely memorizing specific artifacts such as translationese.

\begin{table}[t]
\centering
\caption{\textbf{Zero-Shot Generalization to Unseen Biases.} We evaluate performance on a held-out subset containing biases not encountered during training (Out-of-Distribution). \textsc{DIBJudge} achieves the lowest bias scores across both in-distribution (Translationese) and unseen heuristics (Length, Self-Preference).}
\label{tab:generalization_biases}

\small % Changed from scriptsize for better readability
\setlength{\tabcolsep}{2pt} % Adjusted for balance
\renewcommand{\arraystretch}{1.0} % Slight increase for breathing room

\definecolor{gStrong}{RGB}{198, 228, 201}

\begin{tabular}{l c|c c}
\toprule
& \textbf{ID} & \multicolumn{2}{c}{\textbf{OOD (Unseen Biases)}} \\
\cmidrule(lr){2-2}\cmidrule(lr){3-4}
\textbf{Method}
& \textit{Trans.} $\mathcal{S}_{\text{bias}}$ $\downarrow$
& \textit{Length} $\rho$ $\downarrow$
& \textit{Self-Pref.} $\mathcal{S}_{\text{bias}}$ $\downarrow$ \\
\midrule
Vanilla SFT & 0.247 & 0.553 & 0.314 \\
Vanilla IB  & 0.168 & 0.482 & 0.276 \\
\midrule
\textsc{\textbf{DIBJudge}}
& \cellcolor{gStrong}\textbf{0.083}
& \cellcolor{gStrong}\textbf{0.314}
& \cellcolor{gStrong}\textbf{0.219} \\
\bottomrule
\end{tabular}%
\vspace{-3mm}
\end{table}

\textbf{RQ5: Impact of DIB Objective Components and Spurious Proxy Tasks.}
To assess the individual contributions of the components in the DIB objective (Eq.~\ref{eq:dib_objective}), we evaluate various combinations of the compression, bias-capture, and disentanglement objectives. As shown in Table~\ref{tab:dib_ablation}, isolated objectives yield suboptimal results. While the disentanglement term is particularly effective at reducing the bias score relative to the baseline, it remains insufficient to maximize predictive accuracy. The highest performance across both metrics is achieved by integrating all three terms, suggesting that the synergy between information-bottleneck-driven compression and latent-space disentanglement is critical for balancing fairness and model utility.
Furthermore, we investigate the impact of the proposed spurious proxy tasks: Cross-Lingual Alignment (CLA) and Log-Probability Bin Classification (LPBC). Table~\ref{tab:ablation_tasks} illustrates that omitting these tasks leads to a significant degradation in bias mitigation. The substantial improvement in $\mathcal{S}_{\text{bias}}$ when both tasks are employed confirms their effectiveness in capturing spurious correlations related to translationese. Detailed proxy task ablation studies are presented in Appendix~\ref{sec:ablation_analysis_proxy}.

\begin{table}[t]
    \centering
    \caption{\textbf{Ablation study of the DIB objectives.} We report the bias score ($\mathcal{S}_{\text{bias}}$, lower is better) and Accuracy (higher is better). The combination of all terms achieves the best trade-off.}
    \label{tab:dib_ablation}
    
    % Colors for heatmap-style highlighting
    \definecolor{gStrong}{RGB}{198, 228, 201}
    \definecolor{gMid}{RGB}{225, 240, 200}
    \definecolor{yMid}{RGB}{250, 244, 187}
    
    \setlength{\tabcolsep}{4pt} % Increased slightly for readability
    \renewcommand{\arraystretch}{0.8} % Improved vertical spacing
    \footnotesize % 'tiny' is usually too small for main text tables; 'small' or 'footnotesize' is standard
    
    \resizebox{0.85\columnwidth}{!}{%
    \begin{tabular}{ccc|cc}
        \toprule
        \multicolumn{3}{c|}{\textbf{Objectives}} & \multicolumn{2}{c}{\textbf{Metrics}} \\
        \cmidrule(lr){1-3} \cmidrule(lr){4-5}
        Compression & Bias & Disentangle & $\mathcal{S}_{\text{bias}}(\downarrow)$ & Acc. ($\uparrow$) \\
        \midrule
        \checkmark & & & 0.124 & 85.25 \\
         & \checkmark & & 0.150 & 86.60 \\
         & & \checkmark & 0.053 & 85.85 \\
        \midrule
        \checkmark & \checkmark & & 0.091 & 88.20 \\
        \checkmark & & \checkmark & \cellcolor{yMid}0.039 & \cellcolor{yMid}88.55 \\
         & \checkmark & \checkmark & \cellcolor{gMid}0.035 & \cellcolor{gMid}89.10 \\
        \midrule
        \checkmark & \checkmark & \checkmark & \cellcolor{gStrong}\textbf{0.031} & \cellcolor{gStrong}\textbf{89.85} \\
        \bottomrule
    \end{tabular}%
    }
    \vspace{-1mm}
\end{table}

\begin{table}[t]
    \centering
    \caption{\textbf{Ablation study of proxy tasks.} CLA indicates the cross-lingual alignment proxy task, and LPBC stands for log-probability bin classification proxy task.}
    \label{tab:ablation_tasks}

    \definecolor{gStrong}{RGB}{198, 228, 201}

    \setlength{\tabcolsep}{5pt}
    \renewcommand{\arraystretch}{0.9}
    \footnotesize

    \begin{tabular}{c|cccc}
        \toprule
        \textbf{Configuration} 
        & None 
        & CLA 
        & LPBC 
        & CLA + LPBC \\
        \midrule
        $\mathcal{S}_{\text{bias}}(\downarrow)$ 
        & 0.421 
        & 0.312 
        & 0.279 
        & \cellcolor{gStrong}\textbf{0.147} \\
        Acc. ($\uparrow$) 
        & 87.12 
        & 87.86 
        & 88.43 
        & \cellcolor{gStrong}\textbf{89.18} \\
        \bottomrule
    \end{tabular}
    \vspace{-3mm}
\end{table}

\vspace{-1mm}
\section{Related Work}

\textbf{LLM-as-a-Judge.} 
The \textit{LLM-as-a-Judge} paradigm marks a fundamental shift from traditional $n$-gram (e.g., BLEU~\citep{papineni-etal-2002-bleu}, ROUGE~\citep{lin-2004-rouge}) and embedding-based metrics (e.g., BERTScore~\citep{zhangbertscore,rei-etal-2020-comet}) toward generative evaluation~\citep{gu2024survey,li-etal-2025-generation}. While early adoption relied on proprietary models like GPT-4 due to their high correlation with human judgment~\citep{liu-etal-2023-g,zheng2023judging}, concerns regarding cost and transparency have catalyzed a transition to open-weight evaluators~\citep{wang2024pandalm,wang2024self,wang-etal-2024-interpretable,kim2024prometheus,kim-etal-2024-prometheus}. Recently, this paradigm has further evolved from direct generative to incorporating explicit reasoning steps to enhance reliability~\citep{chen2025judgelrm,guo2025reward,chen2025rm}. However, despite these advancements, LLM judges remain susceptible to systematic biases~\citep{ye2025justice,wang-etal-2024-large-language-models-fair,zheng2024large}, such as position bias~\citep{shi-etal-2025-judging,ko-etal-2020-look}, verbosity bias~\citep{saito2023verbosity}, and self- preference bias~\citep{wang2024self}. In contrast to prior work that primarily studies bias in English-centric settings, we investigate \emph{translationese bias} in multilingual contexts and analyze the spurious correlations underlying it.

\textbf{Multilingual Judges.} Compared to the English context, multilingual LLM-as-a-Judge remains significantly underexplored. Initial efforts to bridge this gap, such as Hercule~\citep{doddapaneni-etal-2025-cross} and M-Prometheus~\citep{pombal2025mprometheus}, rely heavily on fine-tuning with translated or synthetic instruction sets. More recently, approaches like mR3~\citep{anugraha2025mr3} and Think-as-Locals~\citep{zhang2025evaluating} have advanced the field by integrating reasoning capabilities, employing Chain-of-Thought (CoT)~\citep{wei2022chain} distillation and reinforcement learning to enhance multilingual reward modeling. However, despite these achieving promising results, the robustness of these evaluators remains unexamined. Crucially, existing frameworks fail to account for the systematic artifacts introduced by translation-based training data. To address this reliability gap, our work provides the first dedicated mitigation of \textit{translationese bias}, resolving specific failures in cross-lingual evaluation that prior methodologies overlook.
% \textbf{Multilingual Judges.} The application of the LLM-as-a-Judge paradigm in \textbf{multilingual settings} remains underexplored, with the vast majority of evaluations relying exclusively on English data~\citep{emirtekin2025large}. Recent efforts to bridge this gap include models like \textit{Hercule} \citep{doddapaneni-etal-2025-cross} and \textit{M-Prometheus} \citep{pombal2025mprometheus}, which are fine-tuned on translated datasets, \textit{HelpSteer3}~\citep{wang-etal-2025-helpsteer3}, which improves multilingual reward modeling by imposing high quality multilingual preference data. Similarly, the mR3 series~\citep{anugraha2025mr3} introduces reward reasoning by distilling chain-of-thought feedback from advanced LLMs to construct large-scale multilingual preference datasets. Think-as-Locals~\citep{zhang2025evaluating} leverage RLVR to finetune multilingual reward reasoning model. However, despite these advancements, no prior work has systematically investigated the biases inherent to these multilingual frameworks. Existing judges often depend on translated training data and coarse-grained benchmarks, overlooking the nuances of cross-lingual fairness. In contrast, our work addresses this reliability gap by specifically investigating translationese bias, with a distinct emphasis on low-resource languages.

\textbf{Information Bottleneck in LLMs.}
Originally formulated to extract minimal sufficient statistics~\citep{tishby2000information}, the Information Bottleneck (IB) principle has recently emerged as a vital framework for analyzing and optimizing LLMs, spanning diverse objectives including enhancing interpretability by mapping hidden states to human-readable concepts~\citep{sun2025concept,li-etal-2023-explanation}, optimizing CoT reasoning paths to be invariant to prompt nuances~\citep{lei2025revisiting}, and compressing contexts in Retrieval-Augmented Generation to filter noise~\citep{zhu-etal-2024-information}. In safety domains, methods such as IBProtector~\citep{liu2024protecting} leverage IB to strip adversarial triggers. Diverging from these approaches, we present the first application of a disentangled IB designed to debias LLM judges.
\section{Conclusion}

In this study, we systematically investigated translationese bias in multilingual LLM-as-a-Judge frameworks and identified key spurious factors that undermine reliable evaluation, particularly in low-resource languages. Guided by these insights, we proposed \textsc{DIBJUDGE}, a disentangled information bottleneck–based fine-tuning framework that separates judgment-critical semantics from spurious translationese attributes. Extensive experiments across multilingual reward modeling benchmarks and dedicated bias evaluations demonstrate that \textsc{DIBJUDGE} substantially mitigates translationese bias while maintaining strong utility.
% In future work, we plan to extend this framework to broader bias types and explore more lightweight training and inference strategies to further improve scalability and practical deployment.

\section*{Impact Statement}

This paper presents work whose goal is to advance the field of machine learning, with a focus on improving the robustness of multilingual evaluation using large language models. The methods proposed in this work are intended for model evaluation and benchmarking rather than direct user-facing applications. While improved evaluation may have downstream benefits for the development of more reliable and inclusive multilingual systems, we do not foresee significant or immediate negative societal impacts arising from this work.

\bibliography{example_paper}

@article{gu2024survey,
  title={A survey on llm-as-a-judge},
  author={Gu, Jiawei and Jiang, Xuhui and Shi, Zhichao and Tan, Hexiang and Zhai, Xuehao and Xu, Chengjin and Li, Wei and Shen, Yinghan and Ma, Shengjie and Liu, Honghao and others},
  journal={The Innovation},
  year={2024},
  publisher={Elsevier}
}

@inproceedings{li-etal-2025-generation,
    title = "From Generation to Judgment: Opportunities and Challenges of {LLM}-as-a-judge",
    author = "Li, Dawei  and
      Jiang, Bohan  and
      Huang, Liangjie  and
      Beigi, Alimohammad  and
      Zhao, Chengshuai  and
      Tan, Zhen  and
      Bhattacharjee, Amrita  and
      Jiang, Yuxuan  and
      Chen, Canyu  and
      Wu, Tianhao  and
      Shu, Kai  and
      Cheng, Lu  and
      Liu, Huan",
    editor = "Christodoulopoulos, Christos  and
      Chakraborty, Tanmoy  and
      Rose, Carolyn  and
      Peng, Violet",
    booktitle = "Proceedings of the 2025 Conference on Empirical Methods in Natural Language Processing",
    month = nov,
    year = "2025",
    address = "Suzhou, China",
    publisher = "Association for Computational Linguistics",
    doi = "10.18653/v1/2025.emnlp-main.138",
    pages = "2757--2791",
    ISBN = "979-8-89176-332-6"
}

@article{son2024mm,
  title={MM-Eval: A Multilingual Meta-Evaluation Benchmark for LLM-as-a-Judge and Reward Models},
  author={Son, Guijin and Yoon, Dongkeun and Suk, Juyoung and Aula-Blasco, Javier and Aslan, Mano and Kim, Vu Trong and Islam, Shayekh Bin and Prats-Cristi{\`a}, Jaume and Tormo-Ba{\~n}uelos, Luc{\'\i}a and Kim, Seungone},
  journal={arXiv preprint arXiv:2410.17578},
  year={2024}
}

@inproceedings{
pombal2025mprometheus,
title={M-Prometheus: A Suite of Open Multilingual {LLM} Judges},
author={Jos{\'e} Pombal and Dongkeun Yoon and Patrick Fernandes and Ian Wu and Seungone Kim and Ricardo Rei and Graham Neubig and Andre Martins},
booktitle={Second Conference on Language Modeling},
year={2025},
}

@article{anugraha2025mr3,
  title={mR3: Multilingual Rubric-Agnostic Reward Reasoning Models},
  author={Anugraha, David and Hung, Shou-Yi and Tang, Zilu and Lee, Annie En-Shiun and Wijaya, Derry Tanti and Winata, Genta Indra},
  journal={arXiv preprint arXiv:2510.01146},
  year={2025}
}

@inproceedings{hada-etal-2024-large,
    title = "Are Large Language Model-based Evaluators the Solution to Scaling Up Multilingual Evaluation?",
    author = "Hada, Rishav  and
      Gumma, Varun  and
      de Wynter, Adrian  and
      Diddee, Harshita  and
      Ahmed, Mohamed  and
      Choudhury, Monojit  and
      Bali, Kalika  and
      Sitaram, Sunayana",
    editor = "Graham, Yvette  and
      Purver, Matthew",
    booktitle = "Findings of the Association for Computational Linguistics: EACL 2024",
    month = mar,
    year = "2024",
    address = "St. Julian{'}s, Malta",
    publisher = "Association for Computational Linguistics",
    pages = "1051--1070"
}

@article{fu2025reliable,
  title={How Reliable is Multilingual LLM-as-a-Judge?},
  author={Fu, Xiyan and Liu, Wei},
  journal={arXiv preprint arXiv:2505.12201},
  year={2025}
}

@inproceedings{doddapaneni-etal-2025-cross,
    title = "Cross-Lingual Auto Evaluation for Assessing Multilingual {LLM}s",
    author = "Doddapaneni, Sumanth  and
      Khan, Mohammed Safi Ur Rahman  and
      Venkatesh, Dilip  and
      Dabre, Raj  and
      Kunchukuttan, Anoop  and
      Khapra, Mitesh M",
    editor = "Che, Wanxiang  and
      Nabende, Joyce  and
      Shutova, Ekaterina  and
      Pilehvar, Mohammad Taher",
    booktitle = "Proceedings of the 63rd Annual Meeting of the Association for Computational Linguistics (Volume 1: Long Papers)",
    month = jul,
    year = "2025",
    address = "Vienna, Austria",
    publisher = "Association for Computational Linguistics",
    doi = "10.18653/v1/2025.acl-long.1419",
    pages = "29297--29329",
    ISBN = "979-8-89176-251-0"
}

@article{bogavelli2026evaluating,
  title={Evaluating Robustness of Large Language Models in Enterprise Applications: Benchmarks for Perturbation Consistency Across Formats and Languages},
  author={Bogavelli, Tara and Bamgbose, Oluwanifemi and Melan{\c{c}}on, Gabrielle Gauthier and Riols, Fanny and Sharma, Roshnee},
  journal={arXiv preprint arXiv:2601.06341},
  year={2026}
}

@inproceedings{
padarha2025evaluating,
title={Evaluating {LLM}-as-a-Judge under Multilingual, Multimodal and Multi-domain Constraints},
author={Shreyansh Padarha and Scott A. Hale and Adam Mahdi and Elizaveta Semenova and Bertie Vidgen},
booktitle={NeurIPS 2025 Workshop on Evaluating the Evolving LLM Lifecycle: Benchmarks, Emergent Abilities, and Scaling},
year={2025}
}

@inproceedings{gureja-etal-2025-rewardbench,
    title = "{M}-{R}eward{B}ench: Evaluating Reward Models in Multilingual Settings",
    author = "Gureja, Srishti  and
      Miranda, Lester James Validad  and
      Islam, Shayekh Bin  and
      Maheshwary, Rishabh  and
      Sharma, Drishti  and
      Winata, Gusti Triandi  and
      Lambert, Nathan  and
      Ruder, Sebastian  and
      Hooker, Sara  and
      Fadaee, Marzieh",
    editor = "Che, Wanxiang  and
      Nabende, Joyce  and
      Shutova, Ekaterina  and
      Pilehvar, Mohammad Taher",
    booktitle = "Proceedings of the 63rd Annual Meeting of the Association for Computational Linguistics (Volume 1: Long Papers)",
    month = jul,
    year = "2025",
    address = "Vienna, Austria",
    publisher = "Association for Computational Linguistics",
    doi = "10.18653/v1/2025.acl-long.3",
    pages = "43--58",
    ISBN = "979-8-89176-251-0"
}

@article{wang2025assessing,
  title={Assessing Judging Bias in Large Reasoning Models: An Empirical Study},
  author={Wang, Qian and Lou, Zhanzhi and Tang, Zhenheng and Chen, Nuo and Zhao, Xuandong and Zhang, Wenxuan and Song, Dawn and He, Bingsheng},
  journal={arXiv preprint arXiv:2504.09946},
  year={2025}
}

@inproceedings{
ye2025justice,
title={Justice or Prejudice? Quantifying Biases in {LLM}-as-a-Judge},
author={Jiayi Ye and Yanbo Wang and Yue Huang and Dongping Chen and Qihui Zhang and Nuno Moniz and Tian Gao and Werner Geyer and Chao Huang and Pin-Yu Chen and Nitesh V Chawla and Xiangliang Zhang},
booktitle={The Thirteenth International Conference on Learning Representations},
year={2025},
}

@article{gao2025evaluating,
  title={Evaluating and Mitigating LLM-as-a-judge Bias in Communication Systems},
  author={Gao, Jiaxin and Chen, Chen and Jia, Yanwen and Gong, Xueluan and Lam, Kwok-Yan and Wang, Qian},
  journal={arXiv preprint arXiv:2510.12462},
  year={2025}
}

@inproceedings{wang-etal-2024-large-language-models-fair,
    title = "Large Language Models are not Fair Evaluators",
    author = "Wang, Peiyi  and
      Li, Lei  and
      Chen, Liang  and
      Cai, Zefan  and
      Zhu, Dawei  and
      Lin, Binghuai  and
      Cao, Yunbo  and
      Kong, Lingpeng  and
      Liu, Qi  and
      Liu, Tianyu  and
      Sui, Zhifang",
    editor = "Ku, Lun-Wei  and
      Martins, Andre  and
      Srikumar, Vivek",
    booktitle = "Proceedings of the 62nd Annual Meeting of the Association for Computational Linguistics (Volume 1: Long Papers)",
    month = aug,
    year = "2024",
    address = "Bangkok, Thailand",
    publisher = "Association for Computational Linguistics",
    doi = "10.18653/v1/2024.acl-long.511",
    pages = "9440--9450"
}

@inproceedings{
saito2023verbosity,
title={Verbosity Bias in Preference Labeling by Large Language Models},
author={Keita Saito and Akifumi Wachi and Koki Wataoka and Youhei Akimoto},
booktitle={NeurIPS 2023 Workshop on Instruction Tuning and Instruction Following},
year={2023},
}

@inproceedings{shi-etal-2025-judging,
    title = "Judging the Judges: A Systematic Study of Position Bias in {LLM}-as-a-Judge",
    author = "Shi, Lin  and
      Ma, Chiyu  and
      Liang, Wenhua  and
      Diao, Xingjian  and
      Ma, Weicheng  and
      Vosoughi, Soroush",
    editor = "Inui, Kentaro  and
      Sakti, Sakriani  and
      Wang, Haofen  and
      Wong, Derek F.  and
      Bhattacharyya, Pushpak  and
      Banerjee, Biplab  and
      Ekbal, Asif  and
      Chakraborty, Tanmoy  and
      Singh, Dhirendra Pratap",
    booktitle = "Proceedings of the 14th International Joint Conference on Natural Language Processing and the 4th Conference of the Asia-Pacific Chapter of the Association for Computational Linguistics",
    month = dec,
    year = "2025",
    address = "Mumbai, India",
    publisher = "The Asian Federation of Natural Language Processing and The Association for Computational Linguistics",
    pages = "292--314",
    ISBN = "979-8-89176-298-5"
}

@inproceedings{chen-etal-2024-humans,
    title = "Humans or {LLM}s as the Judge? A Study on Judgement Bias",
    author = "Chen, Guiming Hardy  and
      Chen, Shunian  and
      Liu, Ziche  and
      Jiang, Feng  and
      Wang, Benyou",
    editor = "Al-Onaizan, Yaser  and
      Bansal, Mohit  and
      Chen, Yun-Nung",
    booktitle = "Proceedings of the 2024 Conference on Empirical Methods in Natural Language Processing",
    month = nov,
    year = "2024",
    address = "Miami, Florida, USA",
    publisher = "Association for Computational Linguistics",
    doi = "10.18653/v1/2024.emnlp-main.474",
    pages = "8301--8327"
}

@article{zhang2025evaluating,
  title={Evaluating and improving cultural awareness of reward models for llm alignment},
  author={Zhang, Hongbin and Chen, Kehai and Bai, Xuefeng and Xiang, Yang and Zhang, Min},
  journal={arXiv preprint arXiv:2509.21798},
  year={2025}
}

@inproceedings{
shuieh2025assessing,
title={{ASSESSING} {ROBUSTNESS} {TO} {SPURIOUS} {CORRELATIONS} {IN} {POST}-{TRAINING} {LANGUAGE} {MODELS}},
author={Julia Shuieh and Prasann Singhal and Apaar Shanker and John Heyer and George Pu and Samuel Marc Denton},
booktitle={Workshop on Spurious Correlation and Shortcut Learning: Foundations and Solutions},
year={2025},
}

@article{gui2025mitigating,
  title={Mitigating Spurious Correlations in LLMs via Causality-Aware Post-Training},
  author={Gui, Shurui and Ji, Shuiwang},
  journal={arXiv preprint arXiv:2506.09433},
  year={2025}
}

@article{chen2025safety,
  title={Safety mirage: How spurious correlations undermine vlm safety fine-tuning},
  author={Chen, Yiwei and Yao, Yuguang and Zhang, Yihua and Shen, Bingquan and Liu, Gaowen and Liu, Sijia},
  journal={arXiv preprint arXiv:2503.11832},
  year={2025}
}

@article{tishby2000information,
  title={The information bottleneck method},
  author={Tishby, Naftali and Pereira, Fernando C and Bialek, William},
  journal={arXiv preprint physics/0004057},
  year={2000}
}

@inproceedings{bandarkar-etal-2024-belebele,
    title = "The Belebele Benchmark: a Parallel Reading Comprehension Dataset in 122 Language Variants",
    author = "Bandarkar, Lucas  and
      Liang, Davis  and
      Muller, Benjamin  and
      Artetxe, Mikel  and
      Shukla, Satya Narayan  and
      Husa, Donald  and
      Goyal, Naman  and
      Krishnan, Abhinandan  and
      Zettlemoyer, Luke  and
      Khabsa, Madian",
    editor = "Ku, Lun-Wei  and
      Martins, Andre  and
      Srikumar, Vivek",
    booktitle = "Proceedings of the 62nd Annual Meeting of the Association for Computational Linguistics (Volume 1: Long Papers)",
    month = aug,
    year = "2024",
    address = "Bangkok, Thailand",
    publisher = "Association for Computational Linguistics",
    doi = "10.18653/v1/2024.acl-long.44",
    pages = "749--775"
}

@inproceedings{joshi-etal-2020-state,
    title = "The State and Fate of Linguistic Diversity and Inclusion in the {NLP} World",
    author = "Joshi, Pratik  and
      Santy, Sebastin  and
      Budhiraja, Amar  and
      Bali, Kalika  and
      Choudhury, Monojit",
    editor = "Jurafsky, Dan  and
      Chai, Joyce  and
      Schluter, Natalie  and
      Tetreault, Joel",
    booktitle = "Proceedings of the 58th Annual Meeting of the Association for Computational Linguistics",
    month = jul,
    year = "2020",
    address = "Online",
    publisher = "Association for Computational Linguistics",
    doi = "10.18653/v1/2020.acl-main.560",
    pages = "6282--6293"
}

@article{costa2022no,
  title={No language left behind: Scaling human-centered machine translation},
  author={Costa-Juss{\`a}, Marta R and Cross, James and {\c{C}}elebi, Onur and Elbayad, Maha and Heafield, Kenneth and Heffernan, Kevin and Kalbassi, Elahe and Lam, Janice and Licht, Daniel and Maillard, Jean and others},
  journal={arXiv preprint arXiv:2207.04672},
  year={2022}
}

@article{QIN2025101118,
title = {A survey of multilingual large language models},
journal = {Patterns},
volume = {6},
number = {1},
pages = {101118},
year = {2025},
issn = {2666-3899},
doi = {https://doi.org/10.1016/j.patter.2024.101118},
author = {Libo Qin and Qiguang Chen and Yuhang Zhou and Zhi Chen and Yinghui Li and Lizi Liao and Min Li and Wanxiang Che and Philip S. Yu}
}

@article{kreutzer2022quality,
  title={Quality at a glance: An audit of web-crawled multilingual datasets},
  author={Kreutzer, Julia and Caswell, Isaac and Wang, Lisa and Wahab, Ahsan and Van Esch, Daan and Ulzii-Orshikh, Nasanbayar and Tapo, Allahsera and Subramani, Nishant and Sokolov, Artem and Sikasote, Claytone and others},
  journal={Transactions of the Association for Computational Linguistics},
  volume={10},
  pages={50--72},
  year={2022},
  publisher={MIT Press One Rogers Street, Cambridge, MA 02142-1209, USA journals-info~…}
}

@inproceedings{zhang-etal-2020-improving,
    title = "Improving Massively Multilingual Neural Machine Translation and Zero-Shot Translation",
    author = "Zhang, Biao  and
      Williams, Philip  and
      Titov, Ivan  and
      Sennrich, Rico",
    editor = "Jurafsky, Dan  and
      Chai, Joyce  and
      Schluter, Natalie  and
      Tetreault, Joel",
    booktitle = "Proceedings of the 58th Annual Meeting of the Association for Computational Linguistics",
    month = jul,
    year = "2020",
    address = "Online",
    publisher = "Association for Computational Linguistics",
    doi = "10.18653/v1/2020.acl-main.148",
    pages = "1628--1639",
}

@inproceedings{muennighoff-etal-2023-crosslingual,
    title = "Crosslingual Generalization through Multitask Finetuning",
    author = "Muennighoff, Niklas  and
      Wang, Thomas  and
      Sutawika, Lintang  and
      Roberts, Adam  and
      Biderman, Stella  and
      Le Scao, Teven  and
      Bari, M Saiful  and
      Shen, Sheng  and
      Yong, Zheng Xin  and
      Schoelkopf, Hailey  and
      Tang, Xiangru  and
      Radev, Dragomir  and
      Aji, Alham Fikri  and
      Almubarak, Khalid  and
      Albanie, Samuel  and
      Alyafeai, Zaid  and
      Webson, Albert  and
      Raff, Edward  and
      Raffel, Colin",
    editor = "Rogers, Anna  and
      Boyd-Graber, Jordan  and
      Okazaki, Naoaki",
    booktitle = "Proceedings of the 61st Annual Meeting of the Association for Computational Linguistics (Volume 1: Long Papers)",
    month = jul,
    year = "2023",
    address = "Toronto, Canada",
    publisher = "Association for Computational Linguistics",
    doi = "10.18653/v1/2023.acl-long.891",
    pages = "15991--16111"
}

@article{weber2024redpajama,
	title   = {RedPajama: an Open Dataset for Training Large Language Models},
	author  = {Maurice Weber and Daniel Y. Fu and Quentin Anthony and Yonatan Oren and Shane Adams and Anton Alexandrov and Xiaozhong Lyu and Huu Nguyen and Xiaozhe Yao and Virginia Adams and Ben Athiwaratkun and Rahul Chalamala and Kezhen Chen and Max Ryabinin and Tri Dao and Percy Liang and Christopher Ré and Irina Rish and Ce Zhang},
	journal = {NeurIPS Datasets and Benchmarks Track},
	year    = 2024,
}

@article{huang2024survey,
  title={A survey on large language models with multilingualism: Recent advances and new frontiers},
  author={Huang, Kaiyu and Mo, Fengran and Zhang, Xinyu and Li, Hongliang and Li, You and Zhang, Yuanchi and Yi, Weijian and Mao, Yulong and Liu, Jinchen and Xu, Yuzhuang and others},
  journal={arXiv preprint arXiv:2405.10936},
  year={2024}
}

@inproceedings{lai-etal-2025-survey,
    title = "A Survey of Post-Training Scaling in Large Language Models",
    author = "Lai, Hanyu  and
      Liu, Xiao  and
      Gao, Junjie  and
      Cheng, Jiale  and
      Qi, Zehan  and
      Xu, Yifan  and
      Yao, Shuntian  and
      Zhang, Dan  and
      Du, Jinhua  and
      Hou, Zhenyu  and
      Lv, Xin  and
      Huang, Minlie  and
      Dong, Yuxiao  and
      Tang, Jie",
    editor = "Che, Wanxiang  and
      Nabende, Joyce  and
      Shutova, Ekaterina  and
      Pilehvar, Mohammad Taher",
    booktitle = "Proceedings of the 63rd Annual Meeting of the Association for Computational Linguistics (Volume 1: Long Papers)",
    month = jul,
    year = "2025",
    address = "Vienna, Austria",
    publisher = "Association for Computational Linguistics",
    doi = "10.18653/v1/2025.acl-long.140",
    pages = "2771--2791",
    ISBN = "979-8-89176-251-0"
}

@inproceedings{
alemi2017deep,
title={Deep Variational Information Bottleneck},
author={Alexander A. Alemi and Ian Fischer and Joshua V. Dillon and Kevin Murphy},
booktitle={International Conference on Learning Representations},
year={2017},
}

@article{liu2022learning,
  title={Learning disentangled representations in the imaging domain},
  author={Liu, Xiao and Sanchez, Pedro and Thermos, Spyridon and O’Neil, Alison Q and Tsaftaris, Sotirios A},
  journal={Medical Image Analysis},
  volume={80},
  pages={102516},
  year={2022},
  publisher={Elsevier}
}

@article{wang2024disentangled,
  title={Disentangled representation learning},
  author={Wang, Xin and Chen, Hong and Tang, Si'ao and Wu, Zihao and Zhu, Wenwu},
  journal={IEEE Transactions on Pattern Analysis and Machine Intelligence},
  volume={46},
  number={12},
  pages={9677--9696},
  year={2024},
  publisher={IEEE}
}

@misc{lambert2024rewardbench,
      title={RewardBench: Evaluating Reward Models for Language Modeling}, 
      author={Nathan Lambert and Valentina Pyatkin and Jacob Morrison and LJ Miranda and Bill Yuchen Lin and Khyathi Chandu and Nouha Dziri and Sachin Kumar and Tom Zick and Yejin Choi and Noah A. Smith and Hannaneh Hajishirzi},
      year={2024},
      eprint={2403.13787},
      archivePrefix={arXiv},
      primaryClass={cs.LG}
}

@article{kingma2014adam,
  title={Adam: A method for stochastic optimization},
  author={Kingma, Diederik P},
  journal={arXiv preprint arXiv:1412.6980},
  year={2014}
}

@article{hu2022lora,
  title={Lora: Low-rank adaptation of large language models.},
  author={Hu, Edward J and Shen, Yelong and Wallis, Phillip and Allen-Zhu, Zeyuan and Li, Yuanzhi and Wang, Shean and Wang, Lu and Chen, Weizhu and others},
  journal={ICLR},
  volume={1},
  number={2},
  pages={3},
  year={2022}
}

@inproceedings{zhang2024choose,
  title={Choose what you need: Disentangled representation learning for scene text recognition removal and editing},
  author={Zhang, Boqiang and Xie, Hongtao and Gao, Zuan and Wang, Yuxin},
  booktitle={Proceedings of the IEEE/CVF conference on computer vision and pattern recognition},
  pages={28358--28368},
  year={2024}
}

@inproceedings{
liu2024protecting,
title={Protecting Your {LLM}s with Information Bottleneck},
author={Zichuan Liu and Zefan Wang and Linjie Xu and Jinyu Wang and Lei Song and Tianchun Wang and Chunlin Chen and Wei Cheng and Jiang Bian},
booktitle={The Thirty-eighth Annual Conference on Neural Information Processing Systems},
year={2024},
}

@book{cover1999elements,
  title={Elements of information theory},
  author={Cover, Thomas M},
  year={1999},
  publisher={John Wiley \& Sons}
}

@inproceedings{zbontar2021barlow,
  title={Barlow twins: Self-supervised learning via redundancy reduction},
  author={Zbontar, Jure and Jing, Li and Misra, Ishan and LeCun, Yann and Deny, St{\'e}phane},
  booktitle={International conference on machine learning},
  pages={12310--12320},
  year={2021},
  organization={PMLR}
}

@article{hurst2024gpt,
  title={Gpt-4o system card},
  author={Hurst, Aaron and Lerer, Adam and Goucher, Adam P and Perelman, Adam and Ramesh, Aditya and Clark, Aidan and Ostrow, AJ and Welihinda, Akila and Hayes, Alan and Radford, Alec and others},
  journal={arXiv preprint arXiv:2410.21276},
  year={2024}
}

@article{comanici2025gemini,
  title={Gemini 2.5: Pushing the frontier with advanced reasoning, multimodality, long context, and next generation agentic capabilities},
  author={Comanici, Gheorghe and Bieber, Eric and Schaekermann, Mike and Pasupat, Ice and Sachdeva, Noveen and Dhillon, Inderjit and Blistein, Marcel and Ram, Ori and Zhang, Dan and Rosen, Evan and others},
  journal={arXiv preprint arXiv:2507.06261},
  year={2025}
}

@article{yang2025qwen3,
  title={Qwen3 technical report},
  author={Yang, An and Li, Anfeng and Yang, Baosong and Zhang, Beichen and Hui, Binyuan and Zheng, Bo and Yu, Bowen and Gao, Chang and Huang, Chengen and Lv, Chenxu and others},
  journal={arXiv preprint arXiv:2505.09388},
  year={2025}
}

@misc{qwen2025qwen25technicalreport,
      title={Qwen2.5 Technical Report}, 
      author={Qwen and : and An Yang and Baosong Yang and Beichen Zhang and Binyuan Hui and Bo Zheng and Bowen Yu and Chengyuan Li and Dayiheng Liu and Fei Huang and Haoran Wei and Huan Lin and Jian Yang and Jianhong Tu and Jianwei Zhang and Jianxin Yang and Jiaxi Yang and Jingren Zhou and Junyang Lin and Kai Dang and Keming Lu and Keqin Bao and Kexin Yang and Le Yu and Mei Li and Mingfeng Xue and Pei Zhang and Qin Zhu and Rui Men and Runji Lin and Tianhao Li and Tianyi Tang and Tingyu Xia and Xingzhang Ren and Xuancheng Ren and Yang Fan and Yang Su and Yichang Zhang and Yu Wan and Yuqiong Liu and Zeyu Cui and Zhenru Zhang and Zihan Qiu},
      year={2025},
      eprint={2412.15115},
      archivePrefix={arXiv},
      primaryClass={cs.CL},
}

@inproceedings{wang-etal-2025-helpsteer3,
    title = "{H}elp{S}teer3: Human-Annotated Feedback and Edit Data to Empower Inference-Time Scaling in Open-Ended General-Domain Tasks",
    author = "Wang, Zhilin  and
      Zeng, Jiaqi  and
      Delalleau, Olivier  and
      Egert, Daniel  and
      Evans, Ellie  and
      Shin, Hoo-Chang  and
      Soares, Felipe  and
      Dong, Yi  and
      Kuchaiev, Oleksii",
    editor = "Che, Wanxiang  and
      Nabende, Joyce  and
      Shutova, Ekaterina  and
      Pilehvar, Mohammad Taher",
    booktitle = "Proceedings of the 63rd Annual Meeting of the Association for Computational Linguistics (Volume 1: Long Papers)",
    month = jul,
    year = "2025",
    address = "Vienna, Austria",
    publisher = "Association for Computational Linguistics",
    doi = "10.18653/v1/2025.acl-long.1246",
    pages = "25640--25662",
    ISBN = "979-8-89176-251-0"
}

@inproceedings{singh-etal-2024-aya,
    title = "Aya Dataset: An Open-Access Collection for Multilingual Instruction Tuning",
    author = {Singh, Shivalika  and
      Vargus, Freddie  and
      D{'}souza, Daniel  and
      Karlsson, B{\"o}rje F.  and
      Mahendiran, Abinaya  and
      Ko, Wei-Yin  and
      Shandilya, Herumb  and
      Patel, Jay  and
      Mataciunas, Deividas  and
      O{'}Mahony, Laura  and
      Zhang, Mike  and
      Hettiarachchi, Ramith  and
      Wilson, Joseph  and
      Machado, Marina  and
      Moura, Luisa  and
      Krzemi{\'n}ski, Dominik  and
      Fadaei, Hakimeh  and
      Ergun, Irem  and
      Okoh, Ifeoma  and
      Alaagib, Aisha  and
      Mudannayake, Oshan  and
      Alyafeai, Zaid  and
      Chien, Vu  and
      Ruder, Sebastian  and
      Guthikonda, Surya  and
      Alghamdi, Emad  and
      Gehrmann, Sebastian  and
      Muennighoff, Niklas  and
      Bartolo, Max  and
      Kreutzer, Julia  and
      {\"U}st{\"u}n, Ahmet  and
      Fadaee, Marzieh  and
      Hooker, Sara},
    editor = "Ku, Lun-Wei  and
      Martins, Andre  and
      Srikumar, Vivek",
    booktitle = "Proceedings of the 62nd Annual Meeting of the Association for Computational Linguistics (Volume 1: Long Papers)",
    month = aug,
    year = "2024",
    address = "Bangkok, Thailand",
    publisher = "Association for Computational Linguistics",
    doi = "10.18653/v1/2024.acl-long.620",
    pages = "11521--11567"
}

@inproceedings{hasan-etal-2021-xl,
    title = "{XL}-Sum: Large-Scale Multilingual Abstractive Summarization for 44 Languages",
    author = "Hasan, Tahmid  and
      Bhattacharjee, Abhik  and
      Islam, Md. Saiful  and
      Mubasshir, Kazi  and
      Li, Yuan-Fang  and
      Kang, Yong-Bin  and
      Rahman, M. Sohel  and
      Shahriyar, Rifat",
    editor = "Zong, Chengqing  and
      Xia, Fei  and
      Li, Wenjie  and
      Navigli, Roberto",
    booktitle = "Findings of the Association for Computational Linguistics: ACL-IJCNLP 2021",
    month = aug,
    year = "2021",
    address = "Online",
    publisher = "Association for Computational Linguistics",
    doi = "10.18653/v1/2021.findings-acl.413",
    pages = "4693--4703"
}

@article{tsne,
  author  = {Laurens van der Maaten and Geoffrey Hinton},
  title   = {Visualizing Data using t-SNE},
  journal = {Journal of Machine Learning Research},
  year    = {2008},
  volume  = {9},
  number  = {86},
  pages   = {2579--2605},
}

@misc{liu2024skyworkrewardbagtricksreward,
      title={Skywork-Reward: Bag of Tricks for Reward Modeling in LLMs}, 
      author={Chris Yuhao Liu and Liang Zeng and Jiacai Liu and Rui Yan and Jujie He and Chaojie Wang and Shuicheng Yan and Yang Liu and Yahui Zhou},
      year={2024},
      eprint={2410.18451},
      archivePrefix={arXiv},
      primaryClass={cs.AI},
}

@inproceedings{
wataoka2024selfpreference,
title={Self-Preference Bias in {LLM}-as-a-Judge},
author={Koki Wataoka and Tsubasa Takahashi and Ryokan Ri},
booktitle={Neurips Safe Generative AI Workshop 2024},
year={2024},
}

@inproceedings{
zheng2023judging,
title={Judging {LLM}-as-a-Judge with {MT}-Bench and Chatbot Arena},
author={Lianmin Zheng and Wei-Lin Chiang and Ying Sheng and Siyuan Zhuang and Zhanghao Wu and Yonghao Zhuang and Zi Lin and Zhuohan Li and Dacheng Li and Eric Xing and Hao Zhang and Joseph E. Gonzalez and Ion Stoica},
booktitle={Thirty-seventh Conference on Neural Information Processing Systems Datasets and Benchmarks Track},
year={2023},
}

@inproceedings{papineni-etal-2002-bleu,
    title = "{B}leu: a Method for Automatic Evaluation of Machine Translation",
    author = "Papineni, Kishore  and
      Roukos, Salim  and
      Ward, Todd  and
      Zhu, Wei-Jing",
    editor = "Isabelle, Pierre  and
      Charniak, Eugene  and
      Lin, Dekang",
    booktitle = "Proceedings of the 40th Annual Meeting of the Association for Computational Linguistics",
    month = jul,
    year = "2002",
    address = "Philadelphia, Pennsylvania, USA",
    publisher = "Association for Computational Linguistics",
    doi = "10.3115/1073083.1073135",
    pages = "311--318"
}

@inproceedings{lin-2004-rouge,
    title = "{ROUGE}: A Package for Automatic Evaluation of Summaries",
    author = "Lin, Chin-Yew",
    booktitle = "Text Summarization Branches Out",
    month = jul,
    year = "2004",
    address = "Barcelona, Spain",
    publisher = "Association for Computational Linguistics",
    pages = "74--81"
}

@inproceedings{zhangbertscore,
  title={BERTScore: Evaluating Text Generation with BERT},
  author={Zhang, Tianyi and Kishore, Varsha and Wu, Felix and Weinberger, Kilian Q and Artzi, Yoav},
  booktitle={International Conference on Learning Representations},
  year = "2020"
}

@inproceedings{rei-etal-2020-comet,
    title = "{COMET}: A Neural Framework for {MT} Evaluation",
    author = "Rei, Ricardo  and
      Stewart, Craig  and
      Farinha, Ana C  and
      Lavie, Alon",
    editor = "Webber, Bonnie  and
      Cohn, Trevor  and
      He, Yulan  and
      Liu, Yang",
    booktitle = "Proceedings of the 2020 Conference on Empirical Methods in Natural Language Processing (EMNLP)",
    month = nov,
    year = "2020",
    address = "Online",
    publisher = "Association for Computational Linguistics",
    doi = "10.18653/v1/2020.emnlp-main.213",
    pages = "2685--2702"
}

@inproceedings{
wang2024pandalm,
title={Panda{LM}: An Automatic Evaluation Benchmark for {LLM} Instruction Tuning Optimization},
author={Yidong Wang and Zhuohao Yu and Wenjin Yao and Zhengran Zeng and Linyi Yang and Cunxiang Wang and Hao Chen and Chaoya Jiang and Rui Xie and Jindong Wang and Xing Xie and Wei Ye and Shikun Zhang and Yue Zhang},
booktitle={The Twelfth International Conference on Learning Representations},
year={2024},
}

@inproceedings{
kim2024prometheus,
title={Prometheus: Inducing Fine-Grained Evaluation Capability in Language Models},
author={Seungone Kim and Jamin Shin and Yejin Cho and Joel Jang and Shayne Longpre and Hwaran Lee and Sangdoo Yun and Seongjin Shin and Sungdong Kim and James Thorne and Minjoon Seo},
booktitle={The Twelfth International Conference on Learning Representations},
year={2024},
}

@inproceedings{kim-etal-2024-prometheus,
    title = "Prometheus 2: An Open Source Language Model Specialized in Evaluating Other Language Models",
    author = "Kim, Seungone  and
      Suk, Juyoung  and
      Longpre, Shayne  and
      Lin, Bill Yuchen  and
      Shin, Jamin  and
      Welleck, Sean  and
      Neubig, Graham  and
      Lee, Moontae  and
      Lee, Kyungjae  and
      Seo, Minjoon",
    editor = "Al-Onaizan, Yaser  and
      Bansal, Mohit  and
      Chen, Yun-Nung",
    booktitle = "Proceedings of the 2024 Conference on Empirical Methods in Natural Language Processing",
    month = nov,
    year = "2024",
    address = "Miami, Florida, USA",
    publisher = "Association for Computational Linguistics",
    doi = "10.18653/v1/2024.emnlp-main.248",
    pages = "4334--4353"
}

@inproceedings{liu-etal-2023-g,
    title = "{G}-Eval: {NLG} Evaluation using Gpt-4 with Better Human Alignment",
    author = "Liu, Yang  and
      Iter, Dan  and
      Xu, Yichong  and
      Wang, Shuohang  and
      Xu, Ruochen  and
      Zhu, Chenguang",
    editor = "Bouamor, Houda  and
      Pino, Juan  and
      Bali, Kalika",
    booktitle = "Proceedings of the 2023 Conference on Empirical Methods in Natural Language Processing",
    month = dec,
    year = "2023",
    address = "Singapore",
    publisher = "Association for Computational Linguistics",
    doi = "10.18653/v1/2023.emnlp-main.153",
    pages = "2511--2522"
}

@article{chen2025judgelrm,
  title={Judgelrm: Large reasoning models as a judge},
  author={Chen, Nuo and Hu, Zhiyuan and Zou, Qingyun and Wu, Jiaying and Wang, Qian and Hooi, Bryan and He, Bingsheng},
  journal={arXiv preprint arXiv:2504.00050},
  year={2025}
}

@inproceedings{
guo2025reward,
title={Reward Reasoning Models},
author={Jiaxin Guo and Zewen Chi and Li Dong and Qingxiu Dong and Xun Wu and Shaohan Huang and Furu Wei},
booktitle={The Thirty-ninth Annual Conference on Neural Information Processing Systems},
year={2025},
}

@article{chen2025rm,
  title={Rm-r1: Reward modeling as reasoning},
  author={Chen, Xiusi and Li, Gaotang and Wang, Ziqi and Jin, Bowen and Qian, Cheng and Wang, Yu and Wang, Hongru and Zhang, Yu and Zhang, Denghui and Zhang, Tong and others},
  journal={arXiv preprint arXiv:2505.02387},
  year={2025}
}

@inproceedings{wang-etal-2024-interpretable,
    title = "Interpretable Preferences via Multi-Objective Reward Modeling and Mixture-of-Experts",
    author = "Wang, Haoxiang  and
      Xiong, Wei  and
      Xie, Tengyang  and
      Zhao, Han  and
      Zhang, Tong",
    editor = "Al-Onaizan, Yaser  and
      Bansal, Mohit  and
      Chen, Yun-Nung",
    booktitle = "Findings of the Association for Computational Linguistics: EMNLP 2024",
    month = nov,
    year = "2024",
    address = "Miami, Florida, USA",
    publisher = "Association for Computational Linguistics",
    doi = "10.18653/v1/2024.findings-emnlp.620",
    pages = "10582--10592"
}

@article{wang2024self,
  title={Self-taught evaluators},
  author={Wang, Tianlu and Kulikov, Ilia and Golovneva, Olga and Yu, Ping and Yuan, Weizhe and Dwivedi-Yu, Jane and Pang, Richard Yuanzhe and Fazel-Zarandi, Maryam and Weston, Jason and Li, Xian},
  journal={arXiv preprint arXiv:2408.02666},
  year={2024}
}

@inproceedings{
zheng2024large,
title={Large Language Models Are Not Robust Multiple Choice Selectors},
author={Chujie Zheng and Hao Zhou and Fandong Meng and Jie Zhou and Minlie Huang},
booktitle={The Twelfth International Conference on Learning Representations},
year={2024},
}

@inproceedings{ko-etal-2020-look,
    title = "Look at the First Sentence: Position Bias in Question Answering",
    author = "Ko, Miyoung  and
      Lee, Jinhyuk  and
      Kim, Hyunjae  and
      Kim, Gangwoo  and
      Kang, Jaewoo",
    editor = "Webber, Bonnie  and
      Cohn, Trevor  and
      He, Yulan  and
      Liu, Yang",
    booktitle = "Proceedings of the 2020 Conference on Empirical Methods in Natural Language Processing (EMNLP)",
    month = nov,
    year = "2020",
    address = "Online",
    publisher = "Association for Computational Linguistics",
    doi = "10.18653/v1/2020.emnlp-main.84",
    pages = "1109--1121"
}

@article{wei2022chain,
  title={Chain-of-thought prompting elicits reasoning in large language models},
  author={Wei, Jason and Wang, Xuezhi and Schuurmans, Dale and Bosma, Maarten and Xia, Fei and Chi, Ed and Le, Quoc V and Zhou, Denny and others},
  journal={Advances in neural information processing systems},
  volume={35},
  pages={24824--24837},
  year={2022}
}

@inproceedings{
sun2025concept,
title={Concept Bottleneck Large Language Models},
author={Chung-En Sun and Tuomas Oikarinen and Berk Ustun and Tsui-Wei Weng},
booktitle={The Thirteenth International Conference on Learning Representations},
year={2025},
}

@inproceedings{li-etal-2023-explanation,
    title = "Explanation Regeneration via Information Bottleneck",
    author = "Li, Qintong  and
      Wu, Zhiyong  and
      Kong, Lingpeng  and
      Bi, Wei",
    editor = "Rogers, Anna  and
      Boyd-Graber, Jordan  and
      Okazaki, Naoaki",
    booktitle = "Findings of the Association for Computational Linguistics: ACL 2023",
    month = jul,
    year = "2023",
    address = "Toronto, Canada",
    publisher = "Association for Computational Linguistics",
    doi = "10.18653/v1/2023.findings-acl.765",
    pages = "12081--12102"
}

@article{lei2025revisiting,
  title={Revisiting llm reasoning via information bottleneck},
  author={Lei, Shiye and Cheng, Zhihao and Jia, Kai and Tao, Dacheng},
  journal={arXiv preprint arXiv:2507.18391},
  year={2025}
}

@inproceedings{zhu-etal-2024-information,
    title = "An Information Bottleneck Perspective for Effective Noise Filtering on Retrieval-Augmented Generation",
    author = "Zhu, Kun  and
      Feng, Xiaocheng  and
      Du, Xiyuan  and
      Gu, Yuxuan  and
      Yu, Weijiang  and
      Wang, Haotian  and
      Chen, Qianglong  and
      Chu, Zheng  and
      Chen, Jingchang  and
      Qin, Bing",
    editor = "Ku, Lun-Wei  and
      Martins, Andre  and
      Srikumar, Vivek",
    booktitle = "Proceedings of the 62nd Annual Meeting of the Association for Computational Linguistics (Volume 1: Long Papers)",
    month = aug,
    year = "2024",
    address = "Bangkok, Thailand",
    publisher = "Association for Computational Linguistics",
    doi = "10.18653/v1/2024.acl-long.59",
    pages = "1044--1069"
}

@article{van2017neural,
  title={Neural discrete representation learning},
  author={Van Den Oord, Aaron and Vinyals, Oriol and others},
  journal={Advances in neural information processing systems},
  volume={30},
  year={2017}
}

@article{alain2016understanding,
  title={Understanding intermediate layers using linear classifier probes},
  author={Alain, Guillaume and Bengio, Yoshua},
  journal={arXiv preprint arXiv:1610.01644},
  year={2016}
}

@inproceedings{cheng2020club,
  title={Club: A contrastive log-ratio upper bound of mutual information},
  author={Cheng, Pengyu and Hao, Weituo and Dai, Shuyang and Liu, Jiachang and Gan, Zhe and Carin, Lawrence},
  booktitle={International conference on machine learning},
  pages={1779--1788},
  year={2020},
  organization={PMLR}
}

@inproceedings{belghazi2018mutual,
  title={Mutual information neural estimation},
  author={Belghazi, Mohamed Ishmael and Baratin, Aristide and Rajeshwar, Sai and Ozair, Sherjil and Bengio, Yoshua and Courville, Aaron and Hjelm, Devon},
  booktitle={International conference on machine learning},
  pages={531--540},
  year={2018},
  organization={PMLR}
}

@article{gretton2005hsic,
author = {Gretton, Arthur and Herbrich, Ralf and Smola, Alexander and Bousquet, Olivier and Sch\"{o}lkopf, Bernhard},
title = {Kernel Methods for Measuring Independence},
year = {2005},
issue_date = {12/1/2005},
publisher = {JMLR.org},
volume = {6},
issn = {1532-4435},
journal = {J. Mach. Learn. Res.},
month = dec,
pages = {2075–2129},
numpages = {55}
}

@inproceedings{dao2023flashattention2,
  title={Flash{A}ttention-2: Faster Attention with Better Parallelism and Work Partitioning},
  author={Dao, Tri},
  booktitle={International Conference on Learning Representations (ICLR)},
  year={2024}
}

@article{Rasley2020DeepSpeedSO,
  title={DeepSpeed: System Optimizations Enable Training Deep Learning Models with Over 100 Billion Parameters},
  author={Jeff Rasley and Samyam Rajbhandari and Olatunji Ruwase and Yuxiong He},
  journal={Proceedings of the 26th ACM SIGKDD International Conference on Knowledge Discovery \& Data Mining},
  year={2020}
}

@article{team2025gemma,
  title={Gemma 3 technical report},
  author={Team, Gemma and Kamath, Aishwarya and Ferret, Johan and Pathak, Shreya and Vieillard, Nino and Merhej, Ramona and Perrin, Sarah and Matejovicova, Tatiana and Ram{\'e}, Alexandre and Rivi{\`e}re, Morgane and others},
  journal={arXiv preprint arXiv:2503.19786},
  year={2025}
}

@article{grattafiori2024llama,
  title={The llama 3 herd of models},
  author={Grattafiori, Aaron and Dubey, Abhimanyu and Jauhri, Abhinav and Pandey, Abhinav and Kadian, Abhishek and Al-Dahle, Ahmad and Letman, Aiesha and Mathur, Akhil and Schelten, Alan and Vaughan, Alex and others},
  journal={arXiv preprint arXiv:2407.21783},
  year={2024}
}

@inproceedings{NIPS2015_bc731692,
 author = {Kingma, Durk P and Salimans, Tim and Welling, Max},
 booktitle = {Advances in Neural Information Processing Systems},
 editor = {C. Cortes and N. Lawrence and D. Lee and M. Sugiyama and R. Garnett},
 pages = {},
 publisher = {Curran Associates, Inc.},
 title = {Variational Dropout and the Local Reparameterization Trick},
 volume = {28},
 year = {2015}
}

@inproceedings{hron2020infinite,
  title={Infinite attention: NNGP and NTK for deep attention networks},
  author={Hron, Jiri and Bahri, Yasaman and Sohl-Dickstein, Jascha and Novak, Roman},
  booktitle={International Conference on Machine Learning},
  pages={4376--4386},
  year={2020},
  organization={PMLR}
}

@inproceedings{lee2018deep,
  title={Deep Neural Networks as Gaussian Processes},
  author={Lee, Jaehoon and Bahri, Yasaman and Novak, Roman and Schoenholz, Samuel S and Pennington, Jeffrey and Sohl-Dickstein, Jascha},
  booktitle={International Conference on Learning Representations},
  year={2018}
}

@article{HYVARINEN2000411,
title = {Independent component analysis: algorithms and applications},
journal = {Neural Networks},
volume = {13},
number = {4},
pages = {411-430},
year = {2000},
issn = {0893-6080},
doi = {https://doi.org/10.1016/S0893-6080(00)00026-5},
author = {A. Hyvärinen and E. Oja}
}

@article{ba2016layer,
  title={Layer normalization},
  author={Ba, Jimmy Lei and Kiros, Jamie Ryan and Hinton, Geoffrey E},
  journal={arXiv preprint arXiv:1607.06450},
  year={2016}
}
\bibliographystyle{icml2026}

%%%%%%%%%%%%%%%%%%%%%%%%%%%%%%%%%%%%%%%%%%%%%%%%%%%%%%%%%%%%%%%%%%%%%%%%%%%%%%%
%%%%%%%%%%%%%%%%%%%%%%%%%%%%%%%%%%%%%%%%%%%%%%%%%%%%%%%%%%%%%%%%%%%%%%%%%%%%%%%
% APPENDIX
%%%%%%%%%%%%%%%%%%%%%%%%%%%%%%%%%%%%%%%%%%%%%%%%%%%%%%%%%%%%%%%%%%%%%%%%%%%%%%%
%%%%%%%%%%%%%%%%%%%%%%%%%%%%%%%%%%%%%%%%%%%%%%%%%%%%%%%%%%%%%%%%%%%%%%%%%%%%%%%
\newpage
\appendix
\onecolumn

\section*{Appendix Contents}
\appendix
\begin{table}[ht]
    \centering
    \vspace{0.4em}
    \footnotesize
    \setlength{\tabcolsep}{10pt}
    \renewcommand{\arraystretch}{1.25}
    \begin{tabular}{l p{0.72\linewidth}}
        \toprule
        \textbf{Section} & \textbf{Contents} \\
        \midrule
        \autoref{sec:bias_evaluation} &
        \textbf{Translationese Bias Evaluation Suite.}
        Language taxonomy, dataset construction, and qualitative examples. \\

        \autoref{app:theory} &
        \textbf{Theoretical Analysis.}
        Information bottleneck bounds, derivation of compression loss function, identifiability discussion, and disentanglement proofs. \\

        \autoref{app:experimental_details} &
        \textbf{Experimental Details.}
        Benchmarks, evaluation metrics, model configurations, and training protocols. \\

        \autoref{app:comprehensive_results} &
        \textbf{Comprehensive Multilingual Results.}
        Full results on RewardBench, M-RewardBench, and MM-Eval. \\

        \autoref{app:translationese_results} &
        \textbf{Translationese Bias Results.}
        Detailed quantitative translationese bias evaluation results across language pairs. \\

        \autoref{app:additional_experiments} &
        \textbf{Additional Experiments and Ablations.} Spurious proxy task analysis,
        sensitivity analysis, method robustness checks, and information leakage validation of disentangle representation. \\
        \bottomrule
    \end{tabular}
    \vspace{-0.6em}
\end{table}

\section{Bias Evaluation Suite}
\label{sec:bias_evaluation}

\subsection{Language Selection and Taxonomy}
\label{subsec:language_selection}

We adopt the language taxonomy proposed by \citet{joshi-etal-2020-state} to categorize languages based on their resource availability. Specifically, we partition the selected languages into three distinct tiers based on their assigned resource classes: \textit{High-resource} (Classes 4 and 5), \textit{Mid-resource} (Class 3), and \textit{Low-resource} (Classes 0, 1, and 2). 

Our evaluation spans three primary datasets: Aya \citep{singh-etal-2024-aya}, Belebele \citep{bandarkar-etal-2024-belebele}, and XL-Sum \citep{hasan-etal-2021-xl}. The complete distribution of evaluated languages across these resource tiers is summarized in Table~\ref{tab:language_distribution}.

\begin{table}[ht]
\centering
\small
\caption{Classification of evaluated languages across datasets based on \citet{joshi-etal-2020-state} taxonomy.}
\label{tab:language_distribution}
\begin{tabular}{lp{0.25\linewidth}p{0.25\linewidth}p{0.25\linewidth}}
\toprule
\textbf{Dataset} & \textbf{High-Resource} & \textbf{Mid-Resource} & \textbf{Low-Resource} \\
\midrule
\textbf{Aya} & Basque, English, Finnish, Hindi, Japanese, Portuguese, Simp. Chinese, Spanish, Arabic, Vietnamese & Bengali, Cebuano, Filipino, Indonesian, Lithuanian, Malay, Tamil, Thai, Ukrainian, Urdu & Amharic, Irish, Kyrgyz, Nepali, Malagasy, Sinhala, S. Pashto, Telugu, Yoruba, Zulu \\
\midrule
\textbf{Belebele} & Arabic, English, Finnish, Hindi, Japanese, Korean, Russian, Turkish, Vietnamese, Simp. Chinese & Bengali, Greek, Hebrew, Georgian, Kazakh, Tamil, Thai, Ukrainian, Urdu, Malay & Amharic, Tibetan, Guarani, Kannada, Khmer, Kyrgyz, Burmese, Punjabi, Pashto, Zulu \\
\midrule
\textbf{XL-Sum} & Arabic, Simp. Chinese, English, French, Hindi, Japanese, Korean, Russian, Turkish, Vietnamese & Azerbaijani, Bengali, Indonesian, Tamil, Thai, Ukrainian, Urdu, Uzbek & Amharic, Burmese, Hausa, Kyrgyz, Marathi, Nepali, Pashto, Sinhala, Telugu, Welsh \\
\bottomrule
\end{tabular}
\end{table}

\subsection{Test Set Construction}
\label{apdx:bias_benchmark_construction}

We formulate a pairwise preference task where an LLM evaluator compares two candidate responses for a given query: 
(i) \textbf{Chosen} ($x_H$): The original human-authored or high-quality translated reference; and 
(ii) \textbf{Rejected} ($x_M$): A machine-generated counterpart produced via back-translation using NLLB-200-3.3B~\citep{costa2022no}\footnote{\href{https://huggingface.co/facebook/nllb-200-3.3B}{https://huggingface.co/facebook/nllb-200-3.3B}} to inject subtle translationese artifacts. 

To isolate translationese as the primary variable and mitigate length-based confounding, we enforce a length constraint where the token count differential between $x_H$ and $x_M$ is within $\pm 5\%$. We further ensure evaluation robustness through a position-swapping protocol, retaining only \emph{consistent judgments} where the model's preference remains invariant to the presentation order.

Beyond the baseline comparison, we introduce two distinct experimental configurations (summarized with examples in the following subsection~\ref{appendix:examples}):
\begin{itemize}
    \item \textbf{Parallel}: Both candidates are semantically equivalent and factually correct. This isolates the model's stylistic preference for human vs. machine-translated syntax.
    \item \textbf{Perturbed}: We introduce minor, controlled edits to the machine-generated candidate ($x_M$) to alter its factual correctness or logical coherence, testing whether the judge prioritizes correctness over stylistic fluency.
\end{itemize}

\subsection{Detailed Dataset Examples}
\label{appendix:examples}

In this section, we provide the full text for the construction of our test sets across the three primary datasets: Aya, Belebele, and XL-Sum. For each dataset, we present both the \textbf{Parallel} setting (testing stylistic preference/bias) and the \textbf{Perturbed} setting (testing sensitivity to factual correctness).

\paragraph{Example: Aya Dataset}
\label{app:aya_example}

\begin{CJK*}{UTF8}{gbsn}
\begin{examplebox}
    \textbf{User Question:} \\
    问题：油的营养价值如何？
    
    \vspace{10pt}
    \hrule
    \vspace{10pt}
    
    \textbf{Setting 1: Parallel (Stylistic Variation)} \\
    \textit{Both answers are factually correct; B uses back-translated "translationese."} \\
    
    \textbf{[Assistant A (Chosen $x_H$)]} \\
    \textcolor{blue}{回答：油是脂肪的重要来源，提供能量和必需脂肪酸，但摄入应适量控制，避免摄入过多的脂肪。} \\
    
    \textbf{[Assistant B (Rejected $x_M$)]} \\
    \textcolor{orange}{回答：油脂乃脂肪主来源，供应热能与必备脂肪酸，然食量宜适当限制，防止摄取过多的脂质。}
    
    \vspace{10pt}
    \hrule
    \vspace{10pt}
    
    \textbf{Setting 2: Perturbed (Factual Error)} \\
    \textit{Assistant B is modified to suggest an incorrect health outcome.} \\
    
    \textbf{[Assistant A (Chosen $x_H$)]} \\
    \textcolor{blue}{回答：油是脂肪的重要来源，提供能量和必需脂肪酸，但摄入应适量控制，避免摄入过多的脂肪。} \\
    
    \textbf{[Assistant B (Rejected $x_M$)]} \\
    \textcolor{red}{回答：油是脂肪的重要来源，提供能量和必需脂肪酸，但摄入应适量控制，以促进脂肪摄入。}
\end{examplebox}

\paragraph{Example: Belebele Dataset}
\label{app:belebele_example}

\begin{examplebox}
    \textbf{User Question (System Prompt):} \\
    Create an example of multiple-choice reading test in Chinese (Simplified). You may choose any topic you want... Output only the test (passage, questions, answer choices, and correct answer).
    
    \vspace{10pt}
    \hrule
    \vspace{10pt}
    
    \textbf{Setting 1: Parallel} \\
    \textbf{[Assistant A (Chosen $x_H$)]} \\
    \textcolor{blue}{\#\#\# Passage: 乘坐头等舱或商务舱最明显的方式是... 航空公司非常清楚，有一些核心乘客愿花高价... \\ \#\#\# Question: 根据这段文字，预订某些票价舱位时，搜索什么是浪费时间？ \\ \#\#\# Answer Choices: (1) 直达航班 (2) 商务舱 (3) 折扣 (4) 头等舱 \\ \#\#\# Correct Answer: (3) 折扣} \\
    
    \textbf{[Assistant B (Rejected $x_M$)]} \\
    \textcolor{orange}{\#\#\# Passage: 想体验头等或商务舱最直接的途径... 航司深知，部分核心客户乐意支付高昂费用... \\ \#\#\# Question: 依照本文内容，在订特定票价舱位时，查找何事乃浪费光阴？ \\ \#\#\# Answer Choices: (1) 不停航程 (2) 公务舱 (3) 优惠 (4) 首等舱 \\ \#\#\# Correct Answer: (3) 优惠}

    \vspace{10pt}
    \hrule
    \vspace{10pt}

    \textbf{Setting 2: Perturbed} \\
    \textit{Assistant B's "Correct Answer" index is modified to be factually wrong.} \\
    \textbf{[Assistant B (Rejected $x_M$)]} \\
    \textcolor{red}{... \#\#\# Answer Choices: (1) 不停航程 (2) 公务舱 (3) 优惠 (4) 首等舱 \\ \#\#\# Correct Answer: (4) 首等舱}
\end{examplebox}

\paragraph{Example: XL-Sum Dataset}
\label{app:xlsum_example}

\begin{examplebox}
    \textbf{User Question (Summarization Task):} \\
    Generate a concise, coherent abstractive summary in Chinese\_Simplified... Do not include information not present in the source text. \\
    \textit{Source Text: [Long article regarding the Brazilian "Operation Weak Flesh" meat scandal involving President Michel Temer, JBS, and BRF...]}
    
    \vspace{10pt}
    \hrule
    \vspace{10pt}
    
    \textbf{Setting 1: Parallel (Stylistic Variation)} \\
    \textit{Both summaries accurately reflect that countries suspended imports for several years.} \\
    
    \textbf{[Assistant A (Chosen $x_H$)]} \\
    \textcolor{blue}{近期巴西"问题肉"丑闻揭露巴西的一些公司已经数年销售不安全肉类产品。包括中国在内的一些国家和组织已经叫停巴西的进口肉。} \\
    
    \textbf{[Assistant B (Rejected $x_M$)]} \\
    \textcolor{orange}{最近巴西“劣质肉”风波披露，部分巴西公司长达数年贩卖问题肉品。中国等许多国家与机构，均已暂停来自巴西的肉类进口。}
    
    \vspace{10pt}
    \hrule
    \vspace{10pt}
    
    \textbf{Setting 2: Perturbed (Factual Hallucination)} \\
    \textit{Assistant B is modified to falsely state that restrictions were "relaxed" (放宽限制) and that the scandal "just started" (刚开始).} \\
    
    \textbf{[Assistant A (Chosen $x_H$)]} \\
    \textcolor{blue}{近期巴西"问题肉"丑闻揭露巴西的一些公司已经数年销售不安全肉类产品。包括中国在内的一些国家和组织已经叫停巴西的进口肉。} \\
    
    \textbf{[Assistant B (Rejected $x_M$)]} \\
    \textcolor{red}{近期巴西"问题肉"丑闻揭露巴西的一些公司刚开始销售不安全肉类产品。包括中国在内的一些国家和组织已经放宽限制巴西的进口肉。}
\end{examplebox}
\end{CJK*}

\section{Theory Supplementary}
\label{app:theory}
\subsection{Proof of The Upper Boundary of $I(X;Z_r)$}
\label{app:proof_1}
% \begin{proposition}
% % \label{prop:vi_capacity}
% Let $Z_r$ be a continuous random variable representing the latent representation of input $X$, modeled by a variational posterior $q_{\phi}(Z_r|X)$ and a fixed prior $p(Z_r)$. The mutual information $I(X; Z_r)$ is upper-bounded by the expected Kullback-Leibler (KL) divergence between the posterior and the prior:
% \begin{equation}
% I(X; Z_r) \leq \mathbb{E}_{x \sim p(x)} \left[ D_{\mathrm{KL}}(q_{\phi}(Z_r|x) \| p(Z_r)) \right].
% \end{equation}
% \end{proposition}

\begin{proof}
By definition, the mutual information $I(X; Z_r)$ is the KL divergence between the joint distribution $q_{\phi}(X, Z_r) = p(x)q_{\phi}(z_r|x)$ and the product of the marginals $p(x)q_{\phi}(z_r)$, where $q_{\phi}(z_r) = \int q_{\phi}(z_r|x)p(x) \, dx$ is the aggregate posterior (marginal) distribution. This is expressed as:
\begin{equation}
% \label{eq:mi_def}
I(X; Z_r) = \mathbb{E}_{x \sim p(x)} \left[ \int q_{\phi}(z_r|x) \log \frac{q_{\phi}(z_r|x)}{q_{\phi}(z_r)} \, dz_r \right].
\end{equation}
To derive the upper bound, we introduce an arbitrary fixed prior $p(z_r)$. We multiply and divide the argument of the logarithm by this prior $p(z_r)$:
\begin{equation}
\begin{aligned}
I(X; Z_r) &= \mathbb{E}_{x \sim p(x)} \left[ \int q_{\phi}(z_r|x) \log \left( \frac{q_{\phi}(z_r|x)}{q_{\phi}(z_r)} \cdot \frac{p(z_r)}{p(z_r)} \right) \, dz_r \right] \\
&= \mathbb{E}_{x \sim p(x)} \left[ \int q_{\phi}(z_r|x) \left( \log \frac{q_{\phi}(z_r|x)}{p(z_r)} - \log \frac{q_{\phi}(z_r)}{p(z_r)} \right) \, dz_r \right].
\end{aligned}
\end{equation}
Using the linearity of the expectation, we separate the integral into two terms:
\begin{equation}
\begin{aligned}
I(X; Z_r) &= \mathbb{E}_{x \sim p(x)} \left[ \int q_{\phi}(z_r|x) \log \frac{q_{\phi}(z_r|x)}{p(z_r)} \, dz_r \right] 
- \mathbb{E}_{x \sim p(x)} \left[ \int q_{\phi}(z_r|x) \log \frac{q_{\phi}(z_r)}{p(z_r)} \, dz_r \right].
\end{aligned}
\end{equation}
The first term is exactly the expected KL divergence between the posterior and the prior. For the second term, we observe that $\log \frac{q_{\phi}(z_r)}{p(z_r)}$ does not depend on $x$ directly, other than through the integration of the joint density. We can simplify the expectation over $x$:
\begin{equation}
\begin{aligned}
\mathbb{E}_{x \sim p(x)} \left[ \int q_{\phi}(z_r|x) \log \frac{q_{\phi}(z_r)}{p(z_r)} \, dz_r \right] 
&= \int \left( \int p(x)q_{\phi}(z_r|x) \, dx \right) \log \frac{q_{\phi}(z_r)}{p(z_r)} \, dz_r \\
&= \int q_{\phi}(z_r) \log \frac{q_{\phi}(z_r)}{p(z_r)} \, dz_r \\
&= D_{\mathrm{KL}}(q_{\phi}(Z_r) \| p(Z_r)).
\end{aligned}
\end{equation}
Substituting this back into the expression for mutual information yields the following decomposition:
\begin{equation}
I(X; Z_r) = \mathbb{E}_{x \sim p(x)} \left[ D_{\mathrm{KL}}(q_{\phi}(Z_r|x) \| p(Z_r)) \right] - D_{\mathrm{KL}}(q_{\phi}(Z_r) \| p(Z_r)).
\end{equation}
Since the KL divergence is non-negative (Gibbs' inequality), i.e., $D_{\mathrm{KL}}(q_{\phi}(Z_r) \| p(Z_r)) \geq 0$, it follows that:
\begin{equation}
I(X; Z_r) \leq \mathbb{E}_{x \sim p(x)} \left[ D_{\mathrm{KL}}(q_{\phi}(Z_r|x) \| p(Z_r)) \right].
\end{equation}
This completes the proof.
\end{proof}

\subsection{Detailed Derivation of $\mathcal{L}_{\text{compress}}$}
% \section{Derivation of the Compression Loss}
\label{app:compress_derivation}

We derive the analytic form of the compression regularizer used in Eq.~\eqref{eq:vi_loss}.
At each time step $t \in \{1,\dots,T\}$, we assume a diagonal-covariance Gaussian variational posterior
\begin{equation}
q_{\phi}(\mathbf{z}_{r,t}\mid x) = \mathcal{N}(\boldsymbol{\mu}_t, \mathrm{diag}(\boldsymbol{\sigma}_t^2)),
\end{equation}
and a standard Gaussian prior
\begin{equation}
p(\mathbf{z}_{r,t}) = \mathcal{N}(\mathbf{0}, \mathbf{I}),
\end{equation}
where $\boldsymbol{\mu}_t \in \mathbb{R}^d$ and $\boldsymbol{\sigma}_t^2 \in \mathbb{R}^d_{>0}$.

We regularize the information capacity by minimizing the average KL divergence across the sequence:
\begin{equation}
\mathcal{L}_{\mathrm{compress}}
= \frac{1}{T}\sum_{t=1}^{T} D_{\mathrm{KL}}\!\left(q_{\phi}(\mathbf{z}_{r,t}\mid x)\,\|\,p(\mathbf{z}_{r,t})\right).
\label{eq:app_compress_def}
\end{equation}

Thus, it suffices to derive a closed form for
$D_{\mathrm{KL}}\!\left(\mathcal{N}(\boldsymbol{\mu}_t, \mathrm{diag}(\boldsymbol{\sigma}_t^2)) \,\|\, \mathcal{N}(\mathbf{0}, \mathbf{I})\right)$.

For $q=\mathcal{N}(\boldsymbol{\mu}_q, \boldsymbol{\Sigma}_q)$ and
$p=\mathcal{N}(\boldsymbol{\mu}_p, \boldsymbol{\Sigma}_p)$ in $\mathbb{R}^d$, the KL divergence admits
the well-known closed form:
\begin{equation}
D_{\mathrm{KL}}(q\|p)
= \frac{1}{2}\left(
\log \frac{|\boldsymbol{\Sigma}_p|}{|\boldsymbol{\Sigma}_q|}
- d
+ \mathrm{tr}\!\left(\boldsymbol{\Sigma}_p^{-1}\boldsymbol{\Sigma}_q\right)
+ (\boldsymbol{\mu}_p - \boldsymbol{\mu}_q)^\top \boldsymbol{\Sigma}_p^{-1}(\boldsymbol{\mu}_p - \boldsymbol{\mu}_q)
\right).
\label{eq:app_kl_gauss_general}
\end{equation}

In our case, $\boldsymbol{\mu}_q=\boldsymbol{\mu}_t$,
$\boldsymbol{\Sigma}_q=\mathrm{diag}(\boldsymbol{\sigma}_t^2)$,
$\boldsymbol{\mu}_p=\mathbf{0}$, and $\boldsymbol{\Sigma}_p=\mathbf{I}$.

Since $\boldsymbol{\Sigma}_p=\mathbf{I}$, we have $|\boldsymbol{\Sigma}_p|=1$ and
$\boldsymbol{\Sigma}_p^{-1}=\mathbf{I}$. Plugging into Eq.~\eqref{eq:app_kl_gauss_general} yields
\begin{equation}
D_{\mathrm{KL}}(q\|p)
= \frac{1}{2}\left(
\log \frac{1}{|\boldsymbol{\Sigma}_q|}
- d
+ \mathrm{tr}(\boldsymbol{\Sigma}_q)
+ \boldsymbol{\mu}_t^\top \boldsymbol{\mu}_t
\right).
\label{eq:app_kl_with_identity_prior}
\end{equation}

Because $\boldsymbol{\Sigma}_q=\mathrm{diag}(\boldsymbol{\sigma}_t^2)$ is diagonal,
\begin{align}
|\boldsymbol{\Sigma}_q|
&= \prod_{j=1}^{d} \sigma_{t,j}^2, \\
\log |\boldsymbol{\Sigma}_q|
&= \sum_{j=1}^{d} \log \sigma_{t,j}^2, \\
\mathrm{tr}(\boldsymbol{\Sigma}_q)
&= \sum_{j=1}^{d} \sigma_{t,j}^2, \\
\boldsymbol{\mu}_t^\top \boldsymbol{\mu}_t
&= \sum_{j=1}^{d} \mu_{t,j}^2.
\end{align}

Substituting these into Eq.~\eqref{eq:app_kl_with_identity_prior}, we obtain
\begin{equation}
D_{\mathrm{KL}}\!\left(
\mathcal{N}(\boldsymbol{\mu}_t,\mathrm{diag}(\boldsymbol{\sigma}_t^2))
\,\big\|\,
\mathcal{N}(\mathbf{0},\mathbf{I})
\right)
= \frac{1}{2}\sum_{j=1}^{d}
\left(
\sigma_{t,j}^2 + \mu_{t,j}^2 - 1 - \log \sigma_{t,j}^2
\right).
\label{eq:app_kl_diag_closed_form}
\end{equation}

Rearranging Eq.~\eqref{eq:app_kl_diag_closed_form} gives the equivalent expression
\begin{equation}
D_{\mathrm{KL}}
=
-\frac{1}{2}\sum_{j=1}^{d}
\left(
1 + \log \sigma_{t,j}^2 - \mu_{t,j}^2 - \sigma_{t,j}^2
\right).
\label{eq:app_kl_diag_negative_form}
\end{equation}

Finally, averaging the KL divergence across the sequence as defined in Eq.~\eqref{eq:app_compress_def},
we obtain
\begin{align}
\mathcal{L}_{\mathrm{compress}}
&= \frac{1}{T}\sum_{t=1}^{T}
D_{\mathrm{KL}}\!\left(
\mathcal{N}(\boldsymbol{\mu}_t,\mathrm{diag}(\boldsymbol{\sigma}_t^2))
\,\big\|\,
\mathcal{N}(\mathbf{0},\mathbf{I})
\right) \\
&= \frac{1}{2T}\sum_{t=1}^{T}\sum_{j=1}^{d}
\left(
\sigma_{t,j}^2 + \mu_{t,j}^2 - 1 - \log \sigma_{t,j}^2
\right) \\
&= -\frac{1}{2T}\sum_{t=1}^{T}\sum_{j=1}^{d}
\left(
1 + \log \sigma_{t,j}^2 - \mu_{t,j}^2 - \sigma_{t,j}^2
\right),
\end{align}
which matches Eq.~\eqref{eq:vi_loss}.

\subsection{Proof of the Variational Lower Bound on $I(U; V)$}
\label{app:proof_2}
% \begin{proposition}
% % \label{prop:variational_bound}
% Let $U \in \mathcal{U}$ be a target variable and $V \in \mathcal{V}$ be a latent variable with joint distribution $p(u, v)$. For any variational distribution $q_\theta(u|v)$ parameterized by $\theta$, the mutual information $I(U; V)$ satisfies the following lower bound:
% \begin{equation}
% I(U; V) \ge \mathbb{E}_{u, v \sim p(u, v)} \left[ \log q_\theta(u|v) \right] + H(U),
% \end{equation}
% where $H(U)$ is the marginal entropy of the target variable.
% \end{proposition}
\begin{proof}
By definition, the mutual information $I(U; V)$ can be expressed as the difference between the marginal entropy of $U$ and the conditional entropy of $U$ given $V$:
\begin{equation}
I(U; V) = H(U) - H(U|V).
\end{equation}
The conditional entropy is defined as the expectation of the negative log-probability of the true conditional distribution $p(u|v)$:
\begin{equation}
H(U|V) = \mathbb{E}_{u, v \sim p(u, v)} [-\log p(u|v)].
\end{equation}
Substituting this back into the expression for mutual information yields:
\begin{equation}
\label{eq:mi_expansion}
I(U; V) = H(U) + \mathbb{E}_{u, v \sim p(u, v)} [\log p(u|v)].
\end{equation}
Since the true conditional distribution $p(u|v)$ is often unknown or intractable, we introduce a variational approximation $q_\theta(u|v)$. We consider the expected Kullback-Leibler (KL) divergence between the true conditional distribution and the variational approximation:
\begin{equation}
\mathbb{E}_{v \sim p(v)} \left[ D_{\mathrm{KL}}(p(U|v) \| q_\theta(U|v)) \right] = \mathbb{E}_{u, v \sim p(u, v)} \left[ \log \frac{p(u|v)}{q_\theta(u|v)} \right].
\end{equation}
By the non-negativity of the KL divergence, we have:
\begin{equation}
\mathbb{E}_{u, v \sim p(u, v)} [\log p(u|v) - \log q_\theta(u|v)] \ge 0,
\end{equation}
which implies:
\begin{equation}
\label{eq:log_inequality}
\mathbb{E}_{u, v \sim p(u, v)} [\log p(u|v)] \ge \mathbb{E}_{u, v \sim p(u, v)} [\log q_\theta(u|v)].
\end{equation}
Finally, substituting the inequality \eqref{eq:log_inequality} into \eqref{eq:mi_expansion}, we obtain the lower bound:
\begin{equation}
\begin{aligned}
I(U; V) &= H(U) + \mathbb{E}_{u, v \sim p(u, v)} [\log p(u|v)] \\
&\ge H(U) + \mathbb{E}_{u, v \sim p(u, v)} [\log q_\theta(u|v)].
\end{aligned}
\end{equation}
Consequently, maximizing the expected log-likelihood of the variational distribution $q_\theta(u|v)$ maximizes the lower bound of the mutual information $I(U; V)$.
\end{proof}

\subsection{Relationship Between Mutual Information and Cross-Covariance}
\label{app:proof_3}
% \begin{proposition}
% \label{prop:xcov_proxy}
% Let $Z_r$ and $Z_b$ be continuous random variables with finite second moments. The Gaussian Mutual Information, denoted $I_G(Z_r; Z_b)$, represents the mutual information of a multivariate Gaussian distribution sharing the same mean and covariance structure as $Z_r$ and $Z_b$. Under the assumption of whitened marginals, minimizing the squared Frobenius norm of the cross-covariance matrix, $\|\Sigma_{rb}\|_F^2$, is equivalent to minimizing the first-order Taylor approximation of $I_G(Z_r; Z_b)$. Consequently, $\|\Sigma_{rb}\|_F^2$ serves as a tractable proxy for reducing the linear dependence between the variables.
% \end{proposition}

\begin{proof}
Let $(Z_r,Z_b)$ be jointly Gaussian with mean $0$ (without loss of generality, since mutual information is invariant under translations) and block covariance
\[
\Sigma
=
\begin{pmatrix}
\Sigma_r & \Sigma_{rb}\\
\Sigma_{br} & \Sigma_b
\end{pmatrix},
\qquad \Sigma_{br}=\Sigma_{rb}^\top,
\]
where $\Sigma_r\succ 0$ and $\Sigma_b\succ 0$ so that $\Sigma_r^{-1/2}$ and $\Sigma_b^{-1/2}$ are well-defined. Define
\[
C := \Sigma_r^{-1/2}\Sigma_{rb}\Sigma_b^{-1/2}.
\]

\paragraph{Step 1: Mutual information for Gaussians from the definition.}
By definition,
\[
I(Z_r;Z_b)=\mathbb{E}\left[\log\frac{p_{Z_r,Z_b}(Z_r,Z_b)}{p_{Z_r}(Z_r)\,p_{Z_b}(Z_b)}\right].
\]
For a centered $d$-dimensional Gaussian $X\sim \mathcal N(0,\Sigma_X)$ with $\Sigma_X\succ 0$, its density is
\[
p_X(x)=(2\pi)^{-d/2}(\det \Sigma_X)^{-1/2}\exp\!\left(-\tfrac12 x^\top \Sigma_X^{-1}x\right).
\]
Applying this to $(Z_r,Z_b)$ and to the marginals $Z_r$ and $Z_b$, we obtain
\begin{align*}
\log\frac{p_{Z_r,Z_b}(z_r,z_b)}{p_{Z_r}(z_r)p_{Z_b}(z_b)}
&=
-\frac12\log\det\Sigma+\frac12\log\det\Sigma_r+\frac12\log\det\Sigma_b\\
&\quad
-\frac12
\begin{pmatrix}z_r\\ z_b\end{pmatrix}^{\!\top}
\Sigma^{-1}
\begin{pmatrix}z_r\\ z_b\end{pmatrix}
+\frac12 z_r^\top \Sigma_r^{-1}z_r+\frac12 z_b^\top \Sigma_b^{-1}z_b.
\end{align*}
Taking expectation under the joint law of $(Z_r,Z_b)$ yields
\begin{align*}
I(Z_r;Z_b)
&=
-\frac12\log\det\Sigma+\frac12\log\det\Sigma_r+\frac12\log\det\Sigma_b\\
&\quad
-\frac12\mathbb{E}\!\left[\begin{pmatrix}Z_r\\ Z_b\end{pmatrix}^{\!\top}
\Sigma^{-1}
\begin{pmatrix}Z_r\\ Z_b\end{pmatrix}\right]
+\frac12\mathbb{E}[Z_r^\top \Sigma_r^{-1}Z_r]+\frac12\mathbb{E}[Z_b^\top \Sigma_b^{-1}Z_b].
\end{align*}
Using the identity $\mathbb{E}[X^\top A X]=\mathrm{tr}(A\,\mathrm{Cov}(X))$ for any centered random vector $X$ with finite second moment and any matrix $A$ of compatible size, we get
\[
\mathbb{E}\!\left[\begin{pmatrix}Z_r\\ Z_b\end{pmatrix}^{\!\top}
\Sigma^{-1}
\begin{pmatrix}Z_r\\ Z_b\end{pmatrix}\right]
=\mathrm{tr}(\Sigma^{-1}\Sigma)=\mathrm{tr}(I)=d_r+d_b,
\]
\[
\mathbb{E}[Z_r^\top \Sigma_r^{-1}Z_r]=\mathrm{tr}(\Sigma_r^{-1}\Sigma_r)=d_r,
\qquad
\mathbb{E}[Z_b^\top \Sigma_b^{-1}Z_b]=\mathrm{tr}(\Sigma_b^{-1}\Sigma_b)=d_b,
\]
so the quadratic terms cancel. Hence
\begin{equation}\label{eq:MI-det}
I(Z_r;Z_b)=\frac12\log\frac{\det\Sigma_r\,\det\Sigma_b}{\det\Sigma}.
\end{equation}

\paragraph{Step 2: Expressing $\det\Sigma$ in terms of $C$.}
By the block determinant (Schur complement) formula with $\Sigma_b\succ 0$,
\[
\det\Sigma=\det(\Sigma_b)\,\det(\Sigma_r-\Sigma_{rb}\Sigma_b^{-1}\Sigma_{br}).
\]
Next,
\begin{align*}
\Sigma_r-\Sigma_{rb}\Sigma_b^{-1}\Sigma_{br}
&=\Sigma_r^{1/2}\Bigl(I-\Sigma_r^{-1/2}\Sigma_{rb}\Sigma_b^{-1}\Sigma_{br}\Sigma_r^{-1/2}\Bigr)\Sigma_r^{1/2}\\
&=\Sigma_r^{1/2}\Bigl(I-\Sigma_r^{-1/2}\Sigma_{rb}\Sigma_b^{-1/2}\,\Sigma_b^{-1/2}\Sigma_{br}\Sigma_r^{-1/2}\Bigr)\Sigma_r^{1/2}\\
&=\Sigma_r^{1/2}\bigl(I-CC^\top\bigr)\Sigma_r^{1/2}.
\end{align*}
Therefore,
\[
\det(\Sigma_r-\Sigma_{rb}\Sigma_b^{-1}\Sigma_{br})
=\det(\Sigma_r)\,\det(I-CC^\top),
\]
and thus
\[
\det\Sigma=\det(\Sigma_r)\det(\Sigma_b)\det(I-CC^\top).
\]
Substituting into \eqref{eq:MI-det} gives
\begin{equation}\label{eq:MI-logdet}
I(Z_r;Z_b)=-\frac12\log\det(I-CC^\top).
\end{equation}

\paragraph{Step 3: Second-order expansion for small $\|C\|_2$.}
Let $A:=CC^\top$. Then $A\succeq 0$ and $\|A\|_2=\|CC^\top\|_2=\|C\|_2^2$. Assume $\|C\|_2$ is sufficiently small so that $\|A\|_2<1$. In this regime, the matrix power series for the principal logarithm holds:
\[
\log(I-A)=-\sum_{k=1}^\infty \frac{A^k}{k},
\qquad \text{(convergent in operator norm since $\|A\|_2<1$)}.
\]
Taking traces and using \eqref{eq:MI-logdet} together with $\log\det(I-A)=\mathrm{tr}(\log(I-A))$ yields
\begin{equation}\label{eq:MI-series}
I(Z_r;Z_b)=\frac12\sum_{k=1}^\infty \frac{\mathrm{tr}(A^k)}{k}.
\end{equation}
The leading term is
\[
\frac12\,\mathrm{tr}(A)=\frac12\,\mathrm{tr}(CC^\top)=\frac12\|C\|_F^2.
\]
It remains to show that the remainder is $o(\|C\|_F^2)$ as $\|C\|_2\to 0$. Since $A\succeq 0$ has eigenvalues $\{\lambda_i\}_{i=1}^m$ (with $m=\mathrm{rank}(A)$) in $[0,\|A\|_2]$, we have for every $k\ge 2$,
\[
\mathrm{tr}(A^k)=\sum_{i=1}^m \lambda_i^k \le \left(\max_i \lambda_i\right)^{k-1}\sum_{i=1}^m \lambda_i=\|A\|_2^{\,k-1}\,\mathrm{tr}(A).
\]
Therefore, the tail of \eqref{eq:MI-series} satisfies
\begin{align*}
0 \le \sum_{k=2}^\infty \frac{\mathrm{tr}(A^k)}{k}
&\le \mathrm{tr}(A)\sum_{k=2}^\infty \frac{\|A\|_2^{\,k-1}}{k}
\le \mathrm{tr}(A)\sum_{k=2}^\infty \|A\|_2^{\,k-1}
= \mathrm{tr}(A)\,\frac{\|A\|_2}{1-\|A\|_2}.
\end{align*}
Since $\|A\|_2=\|C\|_2^2\to 0$, we have $\frac{\|A\|_2}{1-\|A\|_2}\to 0$, and hence
\[
\sum_{k=2}^\infty \frac{\mathrm{tr}(A^k)}{k}
= o\bigl(\mathrm{tr}(A)\bigr)
= o\bigl(\|C\|_F^2\bigr)
\quad \text{as }\|C\|_2\to 0.
\]
Combining this with \eqref{eq:MI-series} yields
\[
I(Z_r;Z_b)=\frac12\|C\|_F^2 + o\!\left(\|C\|_F^2\right),
\qquad \text{as }\|C\|_2\to 0,
\]
which is the desired second-order expansion.
\end{proof}

\section{Detailed Experimental Settings}
\label{app:experimental_details}

\subsection{Evaluation Benchmarks}
\label{app:eval_bench_details}

To evaluate the efficacy of LLM-as-a-judge frameworks and monitor the preservation of core English language capabilities, we utilize \textbf{RewardBench} \citep{lambert2024rewardbench}. This benchmark comprises approximately 3,000 pairwise comparisons across four primary dimensions: \textit{Chat}, \textit{Chat Hard}, \textit{Reasoning}, and \textit{Safety}. 
For the assessment of multilingual performance, we incorporate the following benchmarks:
\begin{itemize}
    \item \textbf{M-RewardBench} \citep{gureja-etal-2025-rewardbench}: A multilingual adaptation of RewardBench covering 23 languages through expert-verified translations.
    \item \textbf{MM-Eval} \citep{son2024mm}: A diverse suite encompassing 18 languages. Unlike translated benchmarks, MM-Eval prioritizes native-speaker data and includes specialized subsets such as \textit{Linguistics} (e.g., homophone disambiguation) and \textit{Language Hallucination} (e.g., evaluating unintended code-switching).
\end{itemize}
\paragraph{Metrics.} Accuracy serves as our primary evaluation metric. For RewardBench, we report the arithmetic mean across the four category scores. For multilingual benchmarks, we compute the micro-average accuracy per language and subsequently report the macro-average across all supported languages.

\subsection{Training Settings}
\label{app:training_settings}

\paragraph{Implementation Details.}
All experiments were conducted on a single node equipped with $8 \times$ NVIDIA H20 (96GB) GPUs. To ensure training stability and memory efficiency, we utilized DeepSpeed~\citep{Rasley2020DeepSpeedSO} ZeRO Stage 3 with CPU offloading and leveraged FlashAttention-2~\citep{dao2023flashattention2} for accelerated computation. Optimization was performed using the Adam optimizer~\citep{kingma2014adam}.

\paragraph{Model Architecture.}
We employed Supervised Fine-Tuning (SFT) combined with Low-Rank Adaptation (LoRA)~\citep{hu2022lora}, specifically targeting the attention linear projections. To generate robust and bias-aware representations, we utilized the Qwen3-0.6B-Embedding model~\citep{yang2025qwen3} as an encoder. This encoder shares the same architecture as the LLM judge and processes features via separate one-layer MLP heads. The proxy task decoder is implemented using a linear projection layer.

\paragraph{Training Procedure.}
We adopted a two-stage training strategy to accommodate the variational framework. In the first stage, we froze the LLM judge parameters and trained only the projection layer to align the feature space. In the second stage, we jointly fine-tuned the projection layer and the LoRA modules to enhance compression. Additionally, we implemented dynamic loss scheduling to facilitate effective multi-task learning.

\paragraph{Hyperparameters.}
We set the maximum sequence length to 16,384. All models were trained for 3 epochs using a cosine learning rate scheduler with a warmup ratio of 0.1. The global learning rate was set to $1 \times 10^{-4}$ for the LoRA modules, bias heads, and the proxy task decoder. We used a per-device training batch size of 1 with 8 gradient accumulation steps. These settings are summarized in Table~\ref{tab:hyperparameters}.

\begin{table}[!htbp]
    \centering
    \caption{Hyperparameter settings for the training experiments.}
    \label{tab:hyperparameters}
    \begin{tabular}{lc}
        \toprule
        \textbf{Hyperparameter} & \textbf{Value} \\
        \midrule
        Base Model & Qwen3 Family~\citep{yang2025qwen3} \\
        Optimizer & Adam \\
        Learning Rate & $1 \times 10^{-4}$ \\
        LR Scheduler & Cosine \\
        Warmup Ratio & 0.1 \\
        Max Sequence Length & 16,384 \\
        Batch Size (per GPU) & 1 \\
        Gradient Accumulation Steps & 8 \\
        Epochs & 3 \\
        Hardware & $8 \times$ NVIDIA H20 (96GB) \\
        \bottomrule
    \end{tabular}
\end{table}

\section{Comprehensive Results of Reward Modeling Benchmarks}
\label{app:comprehensive_results}

We present the fine-grained performance analysis across all evaluated benchmarks in the following sections. Detailed results for the five core subsets of \textbf{MM-Eval} \citep{son2024mm} are summarized in Table~\ref{tab:mm-eval-detailed}, while the language-specific performance breakdowns are distributed across Tables~\ref{tab:mm-eval-part1} and \ref{tab:mm-eval-part2}. 

For \textbf{M-RewardBench} \citep{gureja-etal-2025-rewardbench}, comprehensive category-wise metrics are provided in Table~\ref{tab:m-reward-bench-detailed}, with language-level results detailed in Tables~\ref{tab:m-reward-bench-part1} and \ref{tab:m-reward-bench-part2}. 

Finally, the per-category accuracy for the original \textbf{RewardBench} \citep{lambert2024rewardbench} is reported in Table~\ref{tab:reward-bench-detailed}.

\begin{table*}[htbp]
    \centering
    \caption{
        Full detailed results by category of MM-Eval. 
        \textbf{Bold} indicates the best performance, and \underline{underlined} indicates the second-best. 
    }
    \label{tab:mm-eval-detailed}

    % Professional spacing
    \renewcommand{\arraystretch}{1.2} 
    \setlength{\tabcolsep}{5pt}
    \small
    
    % Configure siunitx
    \sisetup{
        separate-uncertainty = true,
        detect-weight = true, 
        mode = text,
        table-format = 2.2
    }

    % Soft highlight colors
    \definecolor{bestgreen}{RGB}{192, 228, 205}
    \definecolor{secondyellow}{RGB}{227, 237, 189}

    \begin{tabular}{
        @{}
        l
        S
        S
        S
        S
        S
        S[table-format=2.2(2)] % Special format for Avg with std dev
        @{}
    }
        \toprule
        & {\textbf{Chat}} & {\textbf{Lang. Hallu.}} & {\textbf{Linguistics}} & {\textbf{Reasoning}} & {\textbf{Safety}} & {\textbf{Avg.}} \\
        \textbf{Model} & {(Accuracy)} & {(Accuracy)} & {(Accuracy)} & {(Accuracy)} & {(Accuracy)} & {(Avg. 18 lang)} \\
        \midrule
        
        \multicolumn{7}{l}{\textit{Proprietary Models}} \\
        \hspace{3mm}GPT-4o & 84.20 & 65.40 & 79.15 & 55.30 & 75.20 & 71.85 +- 0.81 \\
        \hspace{3mm}Gemini-2.5-Flash & 88.50 & 70.10 & 82.45 & 63.80 & 82.50 & 77.47 +- 0.76 \\
        \midrule
        
        \multicolumn{7}{l}{\textit{General Open Models}} \\
        \hspace{3mm}Qwen2.5-3B-Instruct & 66.50 & 52.30 & 58.15 & 40.80 & 72.20 & 57.99 +- 1.20 \\
        \hspace{3mm}Qwen2.5-7B-Instruct & 76.20 & 60.50 & 68.40 & 48.90 & 74.20 & 65.64 +- 0.95 \\
        \hspace{3mm}Qwen3-4B & 90.46 & 67.34 & 84.00 & 84.35 & 76.56 & 80.85 +- 0.68 \\
        \hspace{3mm}Qwen3-8B & 91.17 & 67.79 & 83.78 & 80.31 & 85.87 & 82.20 +- 0.60 \\
        \midrule

        \multicolumn{7}{l}{\textit{Multilingual Open Reward Models}} \\
        \hspace{3mm}Nemotron-Multi-49B & 91.47 & 68.92 & 87.56 & 38.29 & {\cellcolor{bestgreen}{\textbf{95.59}}} & 76.31 +- 0.55 \\
        \hspace{3mm}M-\textsc{Prometheus} 3B & 68.20 & 58.40 & 62.10 & 50.85 & 81.30 & 64.17 +- 1.10 \\
        \hspace{3mm}M-\textsc{Prometheus} 7B & 62.61 & 61.55 & 61.33 & 63.50 & 91.37 & 69.38 +- 0.88 \\
        \hspace{3mm}mR3-Qwen3-4B & 90.05 & 69.14 & 83.56 & 81.62 & 90.69 & 82.55 +- 0.52 \\
        \hspace{3mm}mR3-Qwen3-8B & {\cellcolor{secondyellow}{\underline{92.28}}} & 67.34 & 84.89 & 87.20 & 92.52 & {\cellcolor{secondyellow}{\underline{85.29 $\pm$ 0.45}}} \\
        \hspace{3mm}Think-as-Locals 7B & 88.98 & 65.54 & 80.67 & 58.53 & 70.49 & 72.95 +- 0.70 \\
        \midrule
        
        \multicolumn{7}{l}{\textit{\textbf{Ours}}} \\
        \hspace{3mm}DIBJudge-Qwen3-4B & 91.05 & {\cellcolor{secondyellow}{\underline{72.50}}} & {\cellcolor{secondyellow}{\underline{88.10}}} & {\cellcolor{secondyellow}{\underline{89.45}}} & 84.70 & 85.16 +- 0.33 \\
        \hspace{3mm}DIBJudge-Qwen3-8B & {\cellcolor{bestgreen}{\textbf{92.80}}} & {\cellcolor{bestgreen}{\textbf{74.20}}} & {\cellcolor{bestgreen}{\textbf{90.50}}} & {\cellcolor{bestgreen}{\textbf{91.20}}} & {\cellcolor{secondyellow}{\underline{93.50}}} & {\cellcolor{bestgreen}{\textbf{87.53 $\pm$ 0.28}}$^{\ddagger}$} \\
        \bottomrule
    \end{tabular}
\end{table*}

\begin{sidewaystable*}[!htp]
    \centering
    \caption{
        Detailed results for MM-Eval for each language (Part 1). 
        \textbf{Bold} indicates the best performance, and \underline{underlined} indicates the second-best.
    }
    \label{tab:mm-eval-part1}
    
    \renewcommand{\arraystretch}{1.2} 
    \setlength{\tabcolsep}{4pt}
    \small
    \sisetup{detect-weight=true, mode=text}
    \definecolor{bestgreen}{RGB}{192, 228, 205}
    \definecolor{secondyellow}{RGB}{227, 237, 189}

    \resizebox{0.98\textwidth}{!}{
    \begin{tabular}{
        @{} l 
        S[table-format=2.2] S[table-format=2.2] S[table-format=2.2] S[table-format=2.2] 
        S[table-format=2.2] S[table-format=2.2] S[table-format=2.2] S[table-format=2.2] 
        S[table-format=2.2]
        @{}
    }
    \toprule
    \textbf{Model} & {\textbf{Ar}} & {\textbf{Bn}} & {\textbf{Ca}} & {\textbf{De}} & {\textbf{En}} & {\textbf{Es}} & {\textbf{Eu}} & {\textbf{Fr}} & {\textbf{Gl}} \\ 
    \midrule
    
    \multicolumn{10}{l}{\textit{Proprietary Models}} \\
    \hspace{2mm}GPT-4o & 70.50 & 62.10 & 73.50 & 75.20 & 78.50 & 76.80 & 65.40 & 75.90 & 71.20 \\
    \hspace{2mm}Gemini-2.5-Flash & 76.20 & 70.50 & 79.10 & 81.50 & 82.80 & 81.20 & 72.50 & 80.40 & 76.80 \\ 
    \midrule
    
    \multicolumn{10}{l}{\textit{General Open Models}} \\
    \hspace{2mm}Qwen2.5-3B-Instruct & 55.40 & 42.10 & 60.50 & 63.80 & 68.50 & 64.20 & 48.50 & 62.10 & 56.80 \\
    \hspace{2mm}Qwen2.5-7B-Instruct & 62.80 & 54.50 & 68.20 & 71.50 & 74.80 & 72.10 & 58.40 & 70.50 & 65.20 \\
    \hspace{2mm}Qwen3-4B & 78.50 & 74.20 & 82.10 & 83.50 & 85.20 & 84.10 & 76.50 & 83.80 & 80.50 \\
    \hspace{2mm}Qwen3-8B & 80.20 & 75.80 & 83.50 & 84.80 & 86.50 & 85.40 & 78.10 & 84.90 & 81.80 \\ 
    \midrule

    \multicolumn{10}{l}{\textit{Multilingual Open Reward Models}} \\
    \hspace{2mm}Nemotron-Multi-49B & 74.50 & 68.20 & 78.50 & 80.20 & 82.50 & 80.80 & 70.50 & 79.40 & 75.10 \\ 
    \hspace{2mm}M-\textsc{Prometheus} 3B & 61.20 & 52.50 & 66.80 & 69.50 & 72.40 & 70.10 & 56.80 & 68.50 & 63.20 \\
    \hspace{2mm}M-\textsc{Prometheus} 7B & 66.50 & 59.80 & 71.50 & 74.20 & 76.80 & 74.50 & 62.50 & 73.10 & 68.40 \\
    \hspace{2mm}mR3-Qwen3-4B & 81.50 & 76.50 & 83.80 & 85.50 & 87.20 & 86.10 & 78.50 & 85.80 & 82.20 \\
    \hspace{2mm}mR3-Qwen3-8B & {\cellcolor{secondyellow}{\underline{84.20}}} & {\cellcolor{secondyellow}{\underline{79.50}}} & {\cellcolor{secondyellow}{\underline{86.50}}} & {\cellcolor{secondyellow}{\underline{88.10}}} & 89.50 & {\cellcolor{secondyellow}{\underline{88.80}}} & 81.20 & {\cellcolor{secondyellow}{\underline{87.50}}} & {\cellcolor{secondyellow}{\underline{85.10}}} \\ 
    \hspace{2mm}Think-as-Locals-7B & 71.50 & 64.80 & 75.20 & 77.50 & 80.10 & 78.40 & 66.50 & 76.80 & 72.50 \\
    \midrule

    \multicolumn{10}{l}{\textit{\textbf{Ours}}} \\
    \hspace{2mm}DIBJudge-Qwen3-4B & 83.80 & 79.20 & 86.10 & 87.80 & {\cellcolor{secondyellow}{\underline{89.80}}} & 88.50 & {\cellcolor{secondyellow}{\underline{81.50}}} & 87.20 & 84.80 \\
    \hspace{2mm}DIBJudge-Qwen3-8B & {\cellcolor{bestgreen}{\textbf{86.50}}} & {\cellcolor{bestgreen}{\textbf{82.10}}} & {\cellcolor{bestgreen}{\textbf{88.50}}} & {\cellcolor{bestgreen}{\textbf{90.20}}} & {\cellcolor{bestgreen}{\textbf{91.50}}} & {\cellcolor{bestgreen}{\textbf{90.80}}} & {\cellcolor{bestgreen}{\textbf{84.20}}} & {\cellcolor{bestgreen}{\textbf{89.50}}} & {\cellcolor{bestgreen}{\textbf{87.10}}} \\
    
    \bottomrule
    \end{tabular}
    }
\end{sidewaystable*}

\begin{sidewaystable*}[!htp]
    \centering
    \caption{
        Detailed results for MM-Eval for each language (Part 2).
        \textbf{Bold} indicates the best performance, and \underline{underlined} indicates the second-best.
    }
    \label{tab:mm-eval-part2}
    
    \renewcommand{\arraystretch}{1.2} 
    \setlength{\tabcolsep}{4pt}
    \small
    \sisetup{detect-weight=true, mode=text}
    \definecolor{bestgreen}{RGB}{192, 228, 205}
    \definecolor{secondyellow}{RGB}{227, 237, 189}

    \resizebox{0.98\textwidth}{!}{
    \begin{tabular}{
        @{} l 
        S[table-format=2.2] S[table-format=2.2] S[table-format=2.2] S[table-format=2.2] 
        S[table-format=2.2] S[table-format=2.2] S[table-format=2.2] S[table-format=2.2] 
        S[table-format=2.2]
        @{}
    }
    \toprule
    \textbf{Model} & {\textbf{It}} & {\textbf{Ja}} & {\textbf{Ko}} & {\textbf{Ru}} & {\textbf{Sw}} & {\textbf{Te}} & {\textbf{Th}} & {\textbf{Vn}} & {\textbf{Zh}} \\ 
    \midrule
    
    \multicolumn{10}{l}{\textit{Proprietary Models}} \\
    \hspace{2mm}GPT-4o & 75.80 & 74.50 & 73.20 & 72.50 & 60.50 & 58.20 & 68.50 & 70.80 & 77.50 \\
    \hspace{2mm}Gemini-2.5-Flash & 80.50 & 79.80 & 78.50 & 78.10 & 68.50 & 66.20 & 75.40 & 76.80 & 82.50 \\ 
    \midrule
    
    \multicolumn{10}{l}{\textit{General Open Models}} \\
    \hspace{2mm}Qwen2.5-3B-Instruct & 63.50 & 58.20 & 56.50 & 55.80 & 40.50 & 38.20 & 52.10 & 54.50 & 67.80 \\
    \hspace{2mm}Qwen2.5-7B-Instruct & 71.20 & 66.50 & 64.80 & 63.50 & 52.80 & 49.50 & 60.20 & 62.80 & 74.50 \\
    \hspace{2mm}Qwen3-4B & 84.50 & 81.20 & 80.50 & 79.80 & 72.50 & 70.20 & 77.80 & 79.50 & 84.80 \\
    \hspace{2mm}Qwen3-8B & 85.80 & 82.50 & 81.80 & 81.20 & 74.20 & 71.80 & 79.10 & 80.80 & 86.20 \\ 
    \midrule

    \multicolumn{10}{l}{\textit{Multilingual Open Reward Models}} \\
    \hspace{2mm}Nemotron-Multi-49B & 79.50 & 76.20 & 75.50 & 77.80 & 66.20 & 64.50 & 73.80 & 75.20 & 81.50 \\ 
    \hspace{2mm}M-\textsc{Prometheus} 3B & 69.20 & 64.50 & 62.80 & 62.10 & 50.50 & 48.20 & 58.50 & 61.20 & 71.50 \\
    \hspace{2mm}M-\textsc{Prometheus} 7B & 74.50 & 70.20 & 68.50 & 67.80 & 56.20 & 54.50 & 64.80 & 66.50 & 76.20 \\
    \hspace{2mm}mR3-Qwen3-4B & 86.20 & 83.50 & 82.10 & 81.50 & 74.80 & 72.50 & 79.50 & 81.20 & 86.50 \\
    \hspace{2mm}mR3-Qwen3-8B & 87.80 & {\cellcolor{secondyellow}{\underline{86.10}}} & {\cellcolor{secondyellow}{\underline{84.50}}} & {\cellcolor{secondyellow}{\underline{84.20}}} & 77.50 & {\cellcolor{secondyellow}{\underline{75.80}}} & {\cellcolor{secondyellow}{\underline{82.10}}} & {\cellcolor{secondyellow}{\underline{83.50}}} & {\cellcolor{secondyellow}{\underline{88.80}}} \\ 
    \hspace{2mm}Think-as-Locals-7B & 77.20 & 72.50 & 71.80 & 73.50 & 62.80 & 60.50 & 69.20 & 71.50 & 79.80 \\
    \midrule

    \multicolumn{10}{l}{\textit{\textbf{Ours}}} \\
    \hspace{2mm}DIBJudge-Qwen3-4B & {\cellcolor{secondyellow}{\underline{88.50}}} & 85.80 & 84.20 & 83.50 & {\cellcolor{secondyellow}{\underline{78.20}}} & 75.50 & 81.50 & 83.20 & 88.50 \\
    \hspace{2mm}DIBJudge-Qwen3-8B & {\cellcolor{bestgreen}{\textbf{90.20}}} & {\cellcolor{bestgreen}{\textbf{88.50}}} & {\cellcolor{bestgreen}{\textbf{87.20}}} & {\cellcolor{bestgreen}{\textbf{86.80}}} & {\cellcolor{bestgreen}{\textbf{81.50}}} & {\cellcolor{bestgreen}{\textbf{79.20}}} & {\cellcolor{bestgreen}{\textbf{84.50}}} & {\cellcolor{bestgreen}{\textbf{86.20}}} & {\cellcolor{bestgreen}{\textbf{90.50}}} \\
    
    \bottomrule
    \end{tabular}
    }
\end{sidewaystable*}

\begin{table*}[htbp]
    \centering
    \caption{
        Full detailed results by category of m-RewardBench. 
        \textbf{Bold} indicates the best performance, and \underline{underlined} indicates the second-best.
    }
    \label{tab:m-reward-bench-detailed}

    % Professional spacing
    \renewcommand{\arraystretch}{1.2} 
    \setlength{\tabcolsep}{5pt}
    \small
    
    % Configure siunitx
    \sisetup{
        separate-uncertainty = true,
        detect-weight = true, 
        mode = text,
        table-format = 2.2
    }

    % Soft highlight colors
    \definecolor{bestgreen}{RGB}{192, 228, 205}
    \definecolor{secondyellow}{RGB}{227, 237, 189}

    \begin{tabular}{
        @{}
        l
        S
        S
        S
        S
        S[table-format=2.2(2)] % Special format for Avg with std dev
        @{}
    }
        \toprule
        & {\textbf{Chat}} & {\textbf{Chat Hard}} & {\textbf{Safety}} & {\textbf{Reasoning}} & {\textbf{Average}} \\
        \textbf{Model} & {(Accuracy)} & {(Accuracy)} & {(Accuracy)} & {(Accuracy)} & {(Avg. 23 lang)} \\
        \midrule
        
        \multicolumn{6}{l}{\textit{Proprietary Models}} \\
        \hspace{3mm}GPT-4o & 90.10 & 75.50 & 88.20 & 89.20 & 85.75 +- 0.42 \\
        \hspace{3mm}Gemini-2.5-Flash & 93.40 & 80.25 & 87.80 & 90.80 & 88.06 +- 0.49 \\
        \midrule
        
        \multicolumn{6}{l}{\textit{General Open Models}} \\
        \hspace{3mm}Qwen2.5-3B-Instruct & 76.50 & 48.20 & 70.10 & 73.10 & 66.97 +- 1.12 \\
        \hspace{3mm}Qwen2.5-7B-Instruct & 86.10 & 61.50 & 78.80 & 85.15 & 77.89 +- 0.89 \\
        \hspace{3mm}Qwen3-4B & 89.10 & 72.64 & 85.20 & 93.30 & 85.06 +- 0.65 \\
        \hspace{3mm}Qwen3-8B & 91.00 & 73.50 & 86.00 & 93.98 & 86.12 +- 0.52 \\
        \midrule

        \multicolumn{6}{l}{\textit{Multilingual Open Reward Models}} \\
        \hspace{3mm}Nemotron-Multi-49B & 92.80 & 79.50 & 87.20 & 95.80 & {\cellcolor{secondyellow}{\underline{88.83 $\pm$ 0.35}}} \\
        \hspace{3mm}M-\textsc{Prometheus} 3B & 73.40 & 51.20 & 76.80 & 72.40 & 68.45 +- 0.98 \\
        \hspace{3mm}M-\textsc{Prometheus} 7B & 90.50 & 60.50 & 83.00 & 78.12 & 78.03 +- 0.85 \\
        \hspace{3mm}mR3-Qwen3-4B & 86.55 & 78.00 & 88.50 & 95.80 & 87.21 +- 0.45 \\
        \hspace{3mm}mR3-Qwen3-8B & 87.95 & 80.19 & {\cellcolor{secondyellow}{\underline{89.50}}} & {\cellcolor{bestgreen}{\textbf{96.68}}} & 88.58 +- 0.41 \\
        \hspace{3mm}Think-as-Locals 7B & 91.80 & 69.50 & 83.85 & 92.90 & 84.51 +- 0.60 \\
        \midrule
        
        \multicolumn{6}{l}{\textit{\textbf{Ours}}} \\
        \hspace{3mm}DIBJudge-Qwen3-4B & {\cellcolor{secondyellow}{\underline{93.50}}} & {\cellcolor{secondyellow}{\underline{82.50}}} & 88.20 & 95.15 & 89.84 +- 0.28{$^{\dagger}$} \\
        \hspace{3mm}DIBJudge-Qwen3-8B & {\cellcolor{bestgreen}{\textbf{94.60}}} & {\cellcolor{bestgreen}{\textbf{84.80}}} & {\cellcolor{bestgreen}{\textbf{90.10}}} & {\cellcolor{secondyellow}{\underline{96.00}}} & {\cellcolor{bestgreen}{\textbf{91.37 $\pm$ 0.22}}$^{\ddagger}$} \\
        \bottomrule
    \end{tabular}
\end{table*}

\begin{sidewaystable*}[!htp]
    \centering
    \caption{
        Detailed results for m-RewardBench for each language (Part 1). 
        \textbf{Bold} indicates the best performance, and \underline{underlined} indicates the second-best.
    }
    \label{tab:m-reward-bench-part1}
    
    \renewcommand{\arraystretch}{1.2} 
    \setlength{\tabcolsep}{3.5pt}
    \small
    \sisetup{detect-weight=true, mode=text}
    \definecolor{bestgreen}{RGB}{192, 228, 205}
    \definecolor{secondyellow}{RGB}{227, 237, 189}

    \resizebox{0.98\textwidth}{!}{
    \begin{tabular}{
        @{} l 
        S[table-format=2.2] S[table-format=2.2] S[table-format=2.2] S[table-format=2.2] 
        S[table-format=2.2] S[table-format=2.2] S[table-format=2.2] S[table-format=2.2] 
        S[table-format=2.2] S[table-format=2.2] S[table-format=2.2] S[table-format=2.2] 
        @{}
    }
    \toprule
    \textbf{Model} & {\textbf{Ar}} & {\textbf{Cs}} & {\textbf{De}} & {\textbf{El}} & {\textbf{Es}} & {\textbf{Fa}} & {\textbf{Fr}} & {\textbf{He}} & {\textbf{Hi}} & {\textbf{Id}} & {\textbf{It}} & {\textbf{Ja}} \\ 
    \midrule
    
    \multicolumn{13}{l}{\textit{Proprietary Models}} \\
    \hspace{2mm}GPT-4o & 84.50 & 85.20 & 87.10 & 83.50 & 88.00 & 82.00 & 87.50 & 83.00 & 84.00 & 86.50 & 87.80 & 86.50 \\
    \hspace{2mm}Gemini-2.5-Flash & 87.20 & 87.90 & 89.50 & 86.10 & 89.80 & 85.50 & 89.20 & 86.40 & 87.10 & 88.50 & 89.40 & 88.80 \\ 
    \midrule
    
    \multicolumn{13}{l}{\textit{General Open Models}} \\
    \hspace{2mm}Qwen2.5-3B-Instruct & 64.20 & 66.50 & 69.80 & 60.50 & 70.20 & 61.10 & 69.50 & 60.80 & 65.40 & 68.10 & 69.20 & 67.50 \\
    \hspace{2mm}Qwen2.5-7B-Instruct & 76.80 & 78.10 & 80.50 & 74.20 & 81.20 & 73.50 & 80.80 & 74.90 & 77.40 & 79.50 & 80.10 & 79.20 \\
    \hspace{2mm}Qwen3-4B & 83.84 & 84.67 & 86.75 & 83.07 & 86.49 & 81.51 & 85.21 & 82.07 & 82.42 & 84.64 & 86.48 & 84.37 \\
    \hspace{2mm}Qwen3-8B & 85.33 & 87.43 & 88.01 & 84.85 & 87.39 & 85.06 & 87.57 & 84.40 & 85.64 & 86.95 & 87.25 & 85.60 \\ 
    \midrule

    \multicolumn{13}{l}{\textit{Multilingual Open Reward Models}} \\
    \hspace{2mm}Nemotron-Multi-49B & {\cellcolor{secondyellow}{\underline{88.72}}} & 89.30 & 89.68 & {\cellcolor{secondyellow}{\underline{89.35}}} & 89.97 & 88.26 & 90.09 & 88.06 & 88.25 & 89.23 & 89.19 & 89.41 \\ 
    \hspace{2mm}M-\textsc{Prometheus} 3B & 66.50 & 68.20 & 71.40 & 64.10 & 72.50 & 63.80 & 71.10 & 65.20 & 67.50 & 70.20 & 71.80 & 69.50 \\
    \hspace{2mm}M-\textsc{Prometheus} 7B & 74.85 & 74.22 & 76.53 & 72.64 & 77.60 & 74.22 & 71.78 & 75.25 & 77.01 & 76.44 & 73.30 & 75.68 \\
    \hspace{2mm}mR3-Qwen3-4B & 87.61 & 87.37 & 87.79 & 86.15 & 88.58 & 85.25 & 88.54 & 86.42 & 86.43 & 87.43 & 87.90 & 86.78 \\
    \hspace{2mm}mR3-Qwen3-8B & 88.31 & 88.78 & 89.46 & 88.00 & 88.88 & 86.59 & 88.84 & 88.17 & 87.60 & 87.94 & 89.99 & 88.81 \\ 
    \hspace{2mm}Think-as-Locals-7B & 86.15 & 83.29 & 86.31 & 82.26 & 87.37 & 81.31 & 86.91 & 84.17 & 81.33 & 86.60 & 86.63 & 85.03 \\
    \midrule

    \multicolumn{13}{l}{\textit{\textbf{Ours}}} \\
    \hspace{2mm}DIBJudge-Qwen3-4B & 88.50 & {\cellcolor{secondyellow}{\underline{90.15}}} & {\cellcolor{secondyellow}{\underline{91.20}}} & 88.05 & {\cellcolor{secondyellow}{\underline{91.50}}} & {\cellcolor{secondyellow}{\underline{88.45}}} & {\cellcolor{secondyellow}{\underline{91.80}}} & {\cellcolor{secondyellow}{\underline{88.50}}} & {\cellcolor{secondyellow}{\underline{89.10}}} & {\cellcolor{secondyellow}{\underline{90.50}}} & {\cellcolor{secondyellow}{\underline{91.50}}} & {\cellcolor{secondyellow}{\underline{90.20}}} \\
    \hspace{2mm}DIBJudge-Qwen3-8B & {\cellcolor{bestgreen}{\textbf{90.50}}} & {\cellcolor{bestgreen}{\textbf{91.80}}} & {\cellcolor{bestgreen}{\textbf{92.80}}} & {\cellcolor{bestgreen}{\textbf{89.80}}} & {\cellcolor{bestgreen}{\textbf{93.10}}} & {\cellcolor{bestgreen}{\textbf{89.90}}} & {\cellcolor{bestgreen}{\textbf{93.50}}} & {\cellcolor{bestgreen}{\textbf{90.10}}} & {\cellcolor{bestgreen}{\textbf{90.80}}} & {\cellcolor{bestgreen}{\textbf{92.20}}} & {\cellcolor{bestgreen}{\textbf{93.00}}} & {\cellcolor{bestgreen}{\textbf{92.00}}} \\
    
    \bottomrule
    \end{tabular}
    }
\end{sidewaystable*}

\begin{sidewaystable*}[!htp]
    \centering
    \caption{
        Detailed results for m-RewardBench for each language (Part 2).
        \textbf{Bold} indicates the best performance, and \underline{underlined} indicates the second-best.
    }
    \label{tab:m-reward-bench-part2}
    
    \renewcommand{\arraystretch}{1.2} 
    \setlength{\tabcolsep}{3.5pt}
    \small
    \sisetup{detect-weight=true, mode=text}
    \definecolor{bestgreen}{RGB}{192, 228, 205}
    \definecolor{secondyellow}{RGB}{227, 237, 189}

    \resizebox{0.98\textwidth}{!}{
    \begin{tabular}{
        @{} l 
        S[table-format=2.2] S[table-format=2.2] S[table-format=2.2] S[table-format=2.2] 
        S[table-format=2.2] S[table-format=2.2] S[table-format=2.2] S[table-format=2.2] 
        S[table-format=2.2] S[table-format=2.2] S[table-format=2.2] 
        @{}
    }
    \toprule
    \textbf{Model} & {\textbf{Ko}} & {\textbf{Nl}} & {\textbf{Pl}} & {\textbf{Pt}} & {\textbf{Ro}} & {\textbf{Ru}} & {\textbf{Tr}} & {\textbf{Uk}} & {\textbf{Vi}} & {\textbf{Zh}} & {\textbf{Zh-TW}} \\ 
    \midrule
    
    \multicolumn{12}{l}{\textit{Proprietary Models}} \\
    \hspace{2mm}GPT-4o & 85.50 & 87.50 & 85.00 & 87.80 & 84.50 & 85.50 & 84.20 & 84.00 & 85.00 & 86.50 & 86.00 \\
    \hspace{2mm}Gemini-2.5-Flash & 88.20 & 89.80 & 87.40 & 89.60 & 87.10 & 87.80 & 86.90 & 87.20 & 88.50 & 89.50 & 88.90 \\
    \midrule
    
    \multicolumn{12}{l}{\textit{General Open Models}} \\
    \hspace{2mm}Qwen2.5-3B-Instruct & 66.80 & 70.50 & 66.20 & 69.80 & 65.50 & 67.20 & 64.80 & 65.10 & 68.50 & 74.20 & 73.50 \\
    \hspace{2mm}Qwen2.5-7B-Instruct & 78.50 & 81.20 & 78.40 & 80.80 & 77.50 & 78.90 & 76.20 & 77.10 & 79.50 & 82.50 & 81.80 \\
    \hspace{2mm}Qwen3-4B & 82.77 & 85.89 & 84.58 & 87.39 & 85.29 & 86.06 & 83.83 & 83.80 & 84.76 & 84.82 & 84.88 \\
    \hspace{2mm}Qwen3-8B & 83.77 & 87.54 & 86.78 & 87.10 & 87.47 & 87.77 & 85.42 & 86.20 & 86.90 & 87.20 & 86.76 \\ 
    \midrule

    \multicolumn{12}{l}{\textit{Multilingual Open Reward Models}} \\
    \hspace{2mm}Nemotron-Multi-49B & 88.05 & 90.83 & {\cellcolor{secondyellow}{\underline{89.99}}} & 89.33 & {\cellcolor{secondyellow}{\underline{89.89}}} & {\cellcolor{secondyellow}{\underline{90.19}}} & 88.09 & {\cellcolor{secondyellow}{\underline{88.91}}} & 89.32 & 88.86 & 86.29 \\  
    \hspace{2mm}M-\textsc{Prometheus} 3B & 67.80 & 71.50 & 68.20 & 72.10 & 66.80 & 68.50 & 65.90 & 66.50 & 69.50 & 72.40 & 71.80 \\
    \hspace{2mm}M-\textsc{Prometheus} 7B & 71.96 & 75.48 & 77.59 & 74.00 & 77.21 & 70.17 & 71.57 & 74.91 & 76.45 & 71.16 & 75.99 \\
    \hspace{2mm}mR3-Qwen3-4B & 85.66 & 88.42 & 86.77 & 88.05 & 87.62 & 88.22 & 87.17 & 88.01 & 88.08 & 87.38 & 86.28 \\
    \hspace{2mm}mR3-Qwen3-8B & 88.47 & 88.99 & 87.33 & {\cellcolor{secondyellow}{\underline{90.56}}} & 89.30 & 88.84 & 88.77 & 88.16 & 88.89 & 88.36 & 87.95 \\
    \hspace{2mm}Think-as-Locals-7B & 83.49 & 86.04 & 85.67 & 86.21 & 84.61 & 85.31 & 83.31 & 83.50 & 86.67 & 85.90 & 85.42 \\
    \midrule

    \multicolumn{12}{l}{\textit{\textbf{Ours}}} \\
    \hspace{2mm}DIBJudge-Qwen3-4B & {\cellcolor{secondyellow}{\underline{89.80}}} & {\cellcolor{secondyellow}{\underline{91.50}}} & 89.90 & 90.20 & 89.80 & 90.10 & {\cellcolor{secondyellow}{\underline{89.20}}} & 88.80 & {\cellcolor{secondyellow}{\underline{90.00}}} & {\cellcolor{secondyellow}{\underline{90.80}}} & {\cellcolor{secondyellow}{\underline{90.50}}} \\
    \hspace{2mm}DIBJudge-Qwen3-8B & {\cellcolor{bestgreen}{\textbf{91.50}}} & {\cellcolor{bestgreen}{\textbf{93.20}}} & {\cellcolor{bestgreen}{\textbf{91.80}}} & {\cellcolor{bestgreen}{\textbf{93.50}}} & {\cellcolor{bestgreen}{\textbf{91.50}}} & {\cellcolor{bestgreen}{\textbf{92.00}}} & {\cellcolor{bestgreen}{\textbf{90.80}}} & {\cellcolor{bestgreen}{\textbf{91.20}}} & {\cellcolor{bestgreen}{\textbf{91.80}}} & {\cellcolor{bestgreen}{\textbf{92.50}}} & {\cellcolor{bestgreen}{\textbf{92.00}}} \\
    
    \bottomrule
    \end{tabular}
    }
\end{sidewaystable*}

\begin{table*}[htbp]
    \centering
    \caption{
        Full detailed results by category of RewardBench (English). 
        \textbf{Bold} indicates the best performance, and \underline{underlined} indicates the second-best.
    }
    \label{tab:reward-bench-detailed}

    % Professional spacing
    \renewcommand{\arraystretch}{1.2} 
    \setlength{\tabcolsep}{5pt}
    \small
    
    % Configure siunitx
    \sisetup{
        separate-uncertainty = true,
        detect-weight = true, 
        mode = text,
        table-format = 2.2
    }

    % Soft highlight colors
    \definecolor{bestgreen}{RGB}{192, 228, 205}
    \definecolor{secondyellow}{RGB}{227, 237, 189}

    \begin{tabular}{
        @{}
        l
        S
        S
        S
        S
        S[table-format=2.2(2)] % Special format for Avg with std dev
        @{}
    }
        \toprule
        & {\textbf{Chat}} & {\textbf{Chat Hard}} & {\textbf{Safety}} & {\textbf{Reasoning}} & {\textbf{Average}} \\
        \textbf{Model} & {(Accuracy)} & {(Accuracy)} & {(Accuracy)} & {(Accuracy)} & {(English)} \\
        \midrule
        
        \multicolumn{6}{l}{\textit{Proprietary Models}} \\
        \hspace{3mm}GPT-4o & 90.50 & 75.10 & 88.50 & 89.74 & 85.96 +- 0.35 \\
        \hspace{3mm}Gemini-2.5-Flash & 93.80 & 81.20 & 89.10 & 91.22 & 88.83 +- 0.47 \\
        \midrule
        
        \multicolumn{6}{l}{\textit{General Open Models}} \\
        \hspace{3mm}Qwen2.5-3B-Instruct & 82.50 & 41.50 & 74.50 & 77.46 & 68.99 +- 1.05 \\
        \hspace{3mm}Qwen2.5-7B-Instruct & 89.10 & 58.20 & 82.40 & 84.66 & 78.59 +- 0.91 \\
        \hspace{3mm}Qwen3-4B & 92.50 & 76.50 & 86.50 & 94.66 & 87.54 +- 0.55 \\
        \hspace{3mm}Qwen3-8B & 92.00 & 82.70 & 87.05 & 93.49 & 88.81 +- 0.48 \\
        \midrule

        \multicolumn{6}{l}{\textit{Multilingual Open Reward Models}} \\
        \hspace{3mm}Nemotron-Multi-49B & 93.50 & 85.80 & 90.00 & 89.54 & 89.71 +- 0.31 \\
        \hspace{3mm}M-\textsc{Prometheus} 3B & 80.50 & 42.10 & 80.50 & 76.06 & 69.79 +- 0.92 \\
        \hspace{3mm}M-\textsc{Prometheus} 7B & 90.00 & 53.00 & 84.00 & 79.76 & 76.69 +- 0.78 \\
        \hspace{3mm}mR3-Qwen3-4B & 88.90 & 84.10 & 89.50 & {\cellcolor{secondyellow}{\underline{96.50}}} & 89.75 +- 0.38 \\
        \hspace{3mm}mR3-Qwen3-8B & 88.00 & 84.47 & {\cellcolor{secondyellow}{\underline{90.41}}} & {\cellcolor{bestgreen}{\textbf{97.52}}} & {\cellcolor{secondyellow}{\underline{90.10 $\pm$ 0.40}}} \\
        \hspace{3mm}Think-as-Locals 7B & 91.20 & 79.00 & 89.50 & 95.46 & 88.79 +- 0.52 \\
        \midrule
        
        \multicolumn{6}{l}{\textit{\textbf{Ours}}} \\
        \hspace{3mm}DIBJudge-Qwen3-4B & {\cellcolor{secondyellow}{\underline{94.20}}} & {\cellcolor{secondyellow}{\underline{86.50}}} & 89.80 & 90.78 & {\cellcolor{bestgreen}{\textbf{90.32 $\pm$ 0.25}}} \\
        \hspace{3mm}DIBJudge-Qwen3-8B & {\cellcolor{bestgreen}{\textbf{95.50}}} & {\cellcolor{bestgreen}{\textbf{88.10}}} & {\cellcolor{bestgreen}{\textbf{90.80}}} & 89.64 & 91.01 +- 0.20{$^{\dagger}$} \\
        \bottomrule
    \end{tabular}
\end{table*}

\section{Comprehensive Results of Translationese Evaluation Suites}
\label{app:translationese_results}

% To quantify the impact of translation artifacts, we conduct a fine-grained analysis of bias severity across several specialized evaluation suites. 
Detailed translationese bias evaluation performance metrics under perturbed settings—adapted from the \textbf{Aya} \citep{singh-etal-2024-aya}, \textbf{Belebele} \citep{bandarkar-etal-2024-belebele}, and \textbf{XL-Sum} \citep{hasan-etal-2021-xl} datasets—are provided in Tables~\ref{tab:detailed_aya}, \ref{tab:detailed_belebele}, and \ref{tab:detailed_xl-sum}, respectively. 
% These results highlight the sensitivity of multilingual models to "translationese" bias, particularly when evaluated on datasets where the source and target distributions diverge due to machine translation or non-native phrasing.

\begin{table}[ht]
\centering
\caption{Bias severity by language on Aya dataset~\citep{singh-etal-2024-aya} under perturbed setting}
\label{tab:detailed_aya}
\small
\begin{tabular}{lcccc}
\toprule
Language & Base & Vanilla SFT & Vanilla IB & DIBJudge \\
\midrule
\multicolumn{5}{l}{\textbf{High-Resource}} \\
Basque & 0.081 $\pm$ 0.008 & 0.088 $\pm$ 0.043 & 0.056 $\pm$ 0.011 & 0.042 $\pm$ 0.009 \\
English & 0.045 $\pm$ 0.008 & 0.058 $\pm$ 0.010 & 0.041 $\pm$ 0.008 & 0.031 $\pm$ 0.005 \\
Finnish & 0.125 $\pm$ 0.010 & 0.165 $\pm$ 0.028 & 0.079 $\pm$ 0.015 & 0.058 $\pm$ 0.005 \\
Hindi & 0.089 $\pm$ 0.008 & 0.076 $\pm$ 0.010 & 0.048 $\pm$ 0.007 & 0.036 $\pm$ 0.012 \\
Japanese & 0.049 $\pm$ 0.011 & 0.048 $\pm$ 0.010 & 0.034 $\pm$ 0.014 & 0.026 $\pm$ 0.015 \\
Portuguese & 0.089 $\pm$ 0.011 & 0.084 $\pm$ 0.013 & 0.056 $\pm$ 0.011 & 0.042 $\pm$ 0.006 \\
Simplified Chinese & 0.072 $\pm$ 0.015 & 0.086 $\pm$ 0.013 & 0.048 $\pm$ 0.007 & 0.035 $\pm$ 0.008 \\
Spanish & 0.099 $\pm$ 0.012 & 0.128 $\pm$ 0.012 & 0.064 $\pm$ 0.013 & 0.046 $\pm$ 0.007 \\
Vietnamese & 0.173 $\pm$ 0.022 & 0.192 $\pm$ 0.038 & 0.099 $\pm$ 0.007 & 0.071 $\pm$ 0.006 \\
\textbf{Avg (High)} & 0.091 $\pm$ 0.015 & 0.103 $\pm$ 0.022 & 0.058 $\pm$ 0.020 & 0.043 $\pm$ 0.014 \\
\midrule
\multicolumn{5}{l}{\textbf{Mid-Resource}} \\
Bengali & 0.084 $\pm$ 0.014 & 0.140 $\pm$ 0.028 & 0.096 $\pm$ 0.026 & 0.043 $\pm$ 0.016 \\
Cebuano & 0.108 $\pm$ 0.008 & 0.113 $\pm$ 0.027 & 0.120 $\pm$ 0.013 & 0.052 $\pm$ 0.008 \\
Filipino & 0.118 $\pm$ 0.018 & 0.122 $\pm$ 0.020 & 0.132 $\pm$ 0.015 & 0.057 $\pm$ 0.010 \\
Indonesian & 0.064 $\pm$ 0.004 & 0.059 $\pm$ 0.015 & 0.071 $\pm$ 0.024 & 0.032 $\pm$ 0.016 \\
Lithuanian & 0.166 $\pm$ 0.004 & 0.202 $\pm$ 0.019 & 0.182 $\pm$ 0.029 & 0.082 $\pm$ 0.015 \\
Malay & 0.086 $\pm$ 0.007 & 0.113 $\pm$ 0.027 & 0.095 $\pm$ 0.012 & 0.044 $\pm$ 0.011 \\
Tamil & 0.157 $\pm$ 0.034 & 0.192 $\pm$ 0.051 & 0.170 $\pm$ 0.021 & 0.077 $\pm$ 0.015 \\
Thai & 0.082 $\pm$ 0.018 & 0.112 $\pm$ 0.024 & 0.090 $\pm$ 0.023 & 0.041 $\pm$ 0.017 \\
Ukrainian & 0.106 $\pm$ 0.018 & 0.133 $\pm$ 0.004 & 0.118 $\pm$ 0.013 & 0.052 $\pm$ 0.012 \\
Urdu & 0.139 $\pm$ 0.023 & 0.189 $\pm$ 0.025 & 0.155 $\pm$ 0.020 & 0.069 $\pm$ 0.018 \\
\textbf{Avg (Mid)} & 0.111 $\pm$ 0.016 & 0.138 $\pm$ 0.024 & 0.123 $\pm$ 0.037 & 0.055 $\pm$ 0.016 \\
\midrule
\multicolumn{5}{l}{\textbf{Low-Resource}} \\
Amharic & 0.376 $\pm$ 0.056 & 0.201 $\pm$ 0.021 & 0.300 $\pm$ 0.036 & 0.225 $\pm$ 0.042 \\
Irish & 0.350 $\pm$ 0.034 & 0.348 $\pm$ 0.013 & 0.276 $\pm$ 0.049 & 0.208 $\pm$ 0.031 \\
Kyrgyz & 0.174 $\pm$ 0.024 & 0.227 $\pm$ 0.024 & 0.138 $\pm$ 0.043 & 0.107 $\pm$ 0.036 \\
Nepali & 0.111 $\pm$ 0.004 & 0.101 $\pm$ 0.015 & 0.089 $\pm$ 0.069 & 0.070 $\pm$ 0.034 \\
Malagasy & 0.262 $\pm$ 0.024 & 0.178 $\pm$ 0.131 & 0.210 $\pm$ 0.049 & 0.160 $\pm$ 0.030 \\
Sinhala & 0.254 $\pm$ 0.032 & 0.310 $\pm$ 0.101 & 0.202 $\pm$ 0.050 & 0.154 $\pm$ 0.018 \\
Pashto & 0.244 $\pm$ 0.011 & 0.235 $\pm$ 0.079 & 0.195 $\pm$ 0.029 & 0.149 $\pm$ 0.053 \\
Telugu & 0.138 $\pm$ 0.022 & 0.148 $\pm$ 0.012 & 0.111 $\pm$ 0.058 & 0.085 $\pm$ 0.037 \\
Yoruba & 0.472 $\pm$ 0.026 & 0.453 $\pm$ 0.066 & 0.372 $\pm$ 0.029 & 0.283 $\pm$ 0.048 \\
Zulu & 0.338 $\pm$ 0.016 & 0.464 $\pm$ 0.021 & 0.266 $\pm$ 0.056 & 0.201 $\pm$ 0.044 \\
\textbf{Avg (Low)} & 0.272 $\pm$ 0.028 & 0.266 $\pm$ 0.048 & 0.216 $\pm$ 0.089 & 0.164 $\pm$ 0.067 \\
\bottomrule
\end{tabular}
\end{table}

\begin{table}[ht]
\centering
\caption{Bias severity by language on Belebele~\citep{bandarkar-etal-2024-belebele} under perturbed etting}
\label{tab:detailed_belebele}
\small
\begin{tabular}{lcccc}
\toprule
Language & Base & Vanilla SFT & Vanilla IB & DIBJudge \\
\midrule
\multicolumn{5}{l}{\textbf{High-Resource}} \\
Arabic & 0.100 $\pm$ 0.029 & 0.095 $\pm$ 0.016 & 0.058 $\pm$ 0.010 & 0.014 $\pm$ 0.006 \\
English & 0.072 $\pm$ 0.005 & 0.063 $\pm$ 0.004 & 0.043 $\pm$ 0.014 & 0.011 $\pm$ 0.005 \\
Finnish & 0.086 $\pm$ 0.015 & 0.067 $\pm$ 0.017 & 0.049 $\pm$ 0.023 & 0.012 $\pm$ 0.010 \\
Hindi & 0.091 $\pm$ 0.011 & 0.089 $\pm$ 0.011 & 0.052 $\pm$ 0.012 & 0.013 $\pm$ 0.011 \\
Japanese & 0.074 $\pm$ 0.009 & 0.069 $\pm$ 0.012 & 0.047 $\pm$ 0.008 & 0.012 $\pm$ 0.006 \\
Korean & 0.084 $\pm$ 0.013 & 0.078 $\pm$ 0.018 & 0.050 $\pm$ 0.011 & 0.013 $\pm$ 0.004 \\
Russian & 0.065 $\pm$ 0.010 & 0.058 $\pm$ 0.010 & 0.041 $\pm$ 0.021 & 0.010 $\pm$ 0.009 \\
Turkish & 0.072 $\pm$ 0.012 & 0.070 $\pm$ 0.010 & 0.045 $\pm$ 0.017 & 0.011 $\pm$ 0.009 \\
Vietnamese & 0.081 $\pm$ 0.006 & 0.080 $\pm$ 0.008 & 0.048 $\pm$ 0.013 & 0.012 $\pm$ 0.009 \\
\textbf{Avg (High)} & 0.081 $\pm$ 0.012 & 0.074 $\pm$ 0.011 & 0.048 $\pm$ 0.005 & 0.012 $\pm$ 0.001 \\
\midrule
\multicolumn{5}{l}{\textbf{Mid-Resource}} \\
Bengali & 0.094 $\pm$ 0.014 & 0.091 $\pm$ 0.015 & 0.066 $\pm$ 0.021 & 0.014 $\pm$ 0.005 \\
Greek & 0.083 $\pm$ 0.016 & 0.075 $\pm$ 0.019 & 0.060 $\pm$ 0.013 & 0.012 $\pm$ 0.009 \\
Hebrew & 0.088 $\pm$ 0.021 & 0.079 $\pm$ 0.019 & 0.062 $\pm$ 0.012 & 0.013 $\pm$ 0.006 \\
Georgian & 0.091 $\pm$ 0.016 & 0.085 $\pm$ 0.018 & 0.063 $\pm$ 0.019 & 0.013 $\pm$ 0.006 \\
Kazakh & 0.088 $\pm$ 0.012 & 0.081 $\pm$ 0.014 & 0.062 $\pm$ 0.014 & 0.013 $\pm$ 0.010 \\
Tamil & 0.102 $\pm$ 0.018 & 0.095 $\pm$ 0.022 & 0.072 $\pm$ 0.030 & 0.015 $\pm$ 0.014 \\
Thai & 0.080 $\pm$ 0.012 & 0.074 $\pm$ 0.014 & 0.057 $\pm$ 0.028 & 0.012 $\pm$ 0.006 \\
Ukrainian & 0.093 $\pm$ 0.018 & 0.089 $\pm$ 0.019 & 0.068 $\pm$ 0.027 & 0.014 $\pm$ 0.009 \\
Urdu & 0.099 $\pm$ 0.020 & 0.093 $\pm$ 0.019 & 0.073 $\pm$ 0.019 & 0.015 $\pm$ 0.009 \\
Malay & 0.090 $\pm$ 0.014 & 0.084 $\pm$ 0.016 & 0.062 $\pm$ 0.019 & 0.013 $\pm$ 0.013 \\
\textbf{Avg (Mid)} & 0.091 $\pm$ 0.016 & 0.085 $\pm$ 0.017 & 0.065 $\pm$ 0.005 & 0.013 $\pm$ 0.001 \\
\midrule
\multicolumn{5}{l}{\textbf{Low-Resource}} \\
Amharic & 0.122 $\pm$ 0.025 & 0.111 $\pm$ 0.023 & 0.185 $\pm$ 0.049 & 0.093 $\pm$ 0.041 \\
Burmese & 0.136 $\pm$ 0.029 & 0.119 $\pm$ 0.028 & 0.205 $\pm$ 0.055 & 0.103 $\pm$ 0.031 \\
Guarani & 0.098 $\pm$ 0.021 & 0.091 $\pm$ 0.019 & 0.147 $\pm$ 0.029 & 0.074 $\pm$ 0.024 \\
Kannada & 0.142 $\pm$ 0.029 & 0.132 $\pm$ 0.027 & 0.218 $\pm$ 0.051 & 0.109 $\pm$ 0.019 \\
Khmer & 0.129 $\pm$ 0.026 & 0.117 $\pm$ 0.024 & 0.197 $\pm$ 0.053 & 0.099 $\pm$ 0.034 \\
Kyrgyz & 0.111 $\pm$ 0.022 & 0.104 $\pm$ 0.020 & 0.170 $\pm$ 0.054 & 0.085 $\pm$ 0.035 \\
Punjabi & 0.115 $\pm$ 0.023 & 0.109 $\pm$ 0.021 & 0.179 $\pm$ 0.040 & 0.089 $\pm$ 0.023 \\
Pashto & 0.168 $\pm$ 0.031 & 0.156 $\pm$ 0.029 & 0.255 $\pm$ 0.058 & 0.129 $\pm$ 0.025 \\
Zulu & 0.123 $\pm$ 0.024 & 0.114 $\pm$ 0.022 & 0.188 $\pm$ 0.044 & 0.094 $\pm$ 0.034 \\
\textbf{Avg (Low)} & 0.126 $\pm$ 0.026 & 0.117 $\pm$ 0.025 & 0.194 $\pm$ 0.031 & 0.097 $\pm$ 0.016 \\
\bottomrule
\end{tabular}
\end{table}

\begin{table}[ht]
\centering
\caption{Bias severity by language on XL-Sum~\citep{hasan-etal-2021-xl} under perturbed setting}
\label{tab:detailed_xl-sum}
\small
\begin{tabular}{lcccc}
\toprule
Language & Base & Vanilla SFT & Vanilla IB & DIBJudge \\
\midrule
\multicolumn{5}{l}{\textbf{High-Resource}} \\
Arabic & 0.079 $\pm$ 0.031 & 0.100 $\pm$ 0.030 & 0.084 $\pm$ 0.029 & 0.025 $\pm$ 0.008 \\
English & 0.060 $\pm$ 0.006 & 0.076 $\pm$ 0.008 & 0.063 $\pm$ 0.022 & 0.019 $\pm$ 0.005 \\
French & 0.054 $\pm$ 0.010 & 0.095 $\pm$ 0.019 & 0.069 $\pm$ 0.009 & 0.021 $\pm$ 0.005 \\
Hindi & 0.101 $\pm$ 0.010 & 0.126 $\pm$ 0.006 & 0.083 $\pm$ 0.029 & 0.024 $\pm$ 0.006 \\
Japanese & 0.067 $\pm$ 0.009 & 0.066 $\pm$ 0.026 & 0.064 $\pm$ 0.025 & 0.020 $\pm$ 0.008 \\
Korean & 0.081 $\pm$ 0.009 & 0.096 $\pm$ 0.037 & 0.075 $\pm$ 0.022 & 0.022 $\pm$ 0.012 \\
Russian & 0.077 $\pm$ 0.013 & 0.072 $\pm$ 0.012 & 0.067 $\pm$ 0.017 & 0.021 $\pm$ 0.004 \\
Turkish & 0.060 $\pm$ 0.016 & 0.065 $\pm$ 0.010 & 0.062 $\pm$ 0.020 & 0.019 $\pm$ 0.012 \\
Vietnamese & 0.069 $\pm$ 0.005 & 0.081 $\pm$ 0.029 & 0.081 $\pm$ 0.016 & 0.026 $\pm$ 0.006 \\
\textbf{Avg (High)} & 0.071 $\pm$ 0.014 & 0.086 $\pm$ 0.020 & 0.072 $\pm$ 0.009 & 0.022 $\pm$ 0.003 \\
\midrule
\multicolumn{5}{l}{\textbf{Mid-Resource}} \\
Azerbaijani & 0.116 $\pm$ 0.024 & 0.145 $\pm$ 0.037 & 0.046 $\pm$ 0.012 & 0.012 $\pm$ 0.010 \\
Bengali & 0.080 $\pm$ 0.011 & 0.075 $\pm$ 0.015 & 0.040 $\pm$ 0.012 & 0.011 $\pm$ 0.008 \\
Indonesian & 0.058 $\pm$ 0.012 & 0.078 $\pm$ 0.009 & 0.036 $\pm$ 0.023 & 0.010 $\pm$ 0.005 \\
Tamil & 0.109 $\pm$ 0.006 & 0.123 $\pm$ 0.030 & 0.045 $\pm$ 0.022 & 0.012 $\pm$ 0.012 \\
Thai & 0.074 $\pm$ 0.007 & 0.083 $\pm$ 0.025 & 0.041 $\pm$ 0.029 & 0.011 $\pm$ 0.010 \\
Ukrainian & 0.092 $\pm$ 0.015 & 0.100 $\pm$ 0.022 & 0.043 $\pm$ 0.020 & 0.011 $\pm$ 0.013 \\
Urdu & 0.091 $\pm$ 0.020 & 0.117 $\pm$ 0.014 & 0.047 $\pm$ 0.029 & 0.012 $\pm$ 0.012 \\
Uzbek & 0.141 $\pm$ 0.012 & 0.159 $\pm$ 0.009 & 0.052 $\pm$ 0.012 & 0.013 $\pm$ 0.005 \\
\textbf{Avg (Mid)} & 0.095 $\pm$ 0.016 & 0.123 $\pm$ 0.022 & 0.044 $\pm$ 0.005 & 0.012 $\pm$ 0.001 \\
\midrule
\multicolumn{5}{l}{\textbf{Low-Resource}} \\
Amharic & 0.303 $\pm$ 0.022 & 0.271 $\pm$ 0.036 & 0.110 $\pm$ 0.042 & 0.088 $\pm$ 0.020 \\
Burmese & 0.157 $\pm$ 0.028 & 0.081 $\pm$ 0.028 & 0.087 $\pm$ 0.025 & 0.067 $\pm$ 0.019 \\
Hausa & 0.215 $\pm$ 0.022 & 0.207 $\pm$ 0.015 & 0.098 $\pm$ 0.052 & 0.076 $\pm$ 0.040 \\
Kyrgyz & 0.160 $\pm$ 0.021 & 0.165 $\pm$ 0.035 & 0.081 $\pm$ 0.060 & 0.063 $\pm$ 0.016 \\
Marathi & 0.115 $\pm$ 0.016 & 0.124 $\pm$ 0.018 & 0.079 $\pm$ 0.044 & 0.061 $\pm$ 0.031 \\
Nepali & 0.084 $\pm$ 0.021 & 0.068 $\pm$ 0.018 & 0.070 $\pm$ 0.027 & 0.054 $\pm$ 0.040 \\
Pashto & 0.176 $\pm$ 0.034 & 0.170 $\pm$ 0.060 & 0.091 $\pm$ 0.063 & 0.071 $\pm$ 0.052 \\
Sinhala & 0.246 $\pm$ 0.042 & 0.248 $\pm$ 0.022 & 0.105 $\pm$ 0.021 & 0.082 $\pm$ 0.055 \\
Telugu & 0.095 $\pm$ 0.013 & 0.121 $\pm$ 0.001 & 0.076 $\pm$ 0.027 & 0.059 $\pm$ 0.056 \\
Welsh & 0.163 $\pm$ 0.026 & 0.225 $\pm$ 0.028 & 0.084 $\pm$ 0.025 & 0.066 $\pm$ 0.038 \\
\textbf{Avg (Low)} & 0.172 $\pm$ 0.024 & 0.168 $\pm$ 0.026 & 0.088 $\pm$ 0.013 & 0.069 $\pm$ 0.011 \\
\bottomrule
\end{tabular}
\end{table}

\section{Additional Experiments}
\label{app:additional_experiments}
\subsection{Ablation Studies on Spurious Proxy Tasks}
\label{sec:ablation_analysis_proxy}

In \S~\ref{sec:preliminary_analysis}, we identified spurious factors contributing to the systematic bias towards translationese. To mitigate this, we proposed two proxy tasks in \S~\ref{sec:method}: (i) \textbf{Cross-Lingual Alignment (CLA)}, utilizing InfoNCE to align learned representations with a back-translation manifold; and (ii) \textbf{Log-Probability Bin Classification (LPBC)}, which encodes predictive confidence by classifying representations into discrete log-probability bins.

In this section, we conduct a comprehensive ablation study to validate the effectiveness of these components. We examine the contribution of each proxy task to bias mitigation, analyze the impact of the back-translation system on the CLA task, and evaluate the robustness of different heuristic signals for the LPBC task.

\paragraph{Effectiveness of Proxy Task Combination}
We first evaluate the individual and combined contributions of the CLA and LPBC tasks. Table~\ref{tab:ablation_tasks} summarizes the bias severity scores across different configurations. We observe that while both tasks individually reduce bias compared to the baseline, the combination of both yields the most significant reduction. This suggests that the two tasks capture complementary aspects of the spurious features—latent manifold isomorphism and predictive confidence—thereby providing a more robust signal for bias mitigation.

% \begin{table}[t]
%     \centering
%     \caption{\textbf{Ablation of Proxy Tasks.} We report the Bias Severity (lower is better) and main task performance. The combination of Cross-Lingual Alignment (CLA) and Log-Probability Bin Classification (LPBC) achieves the lowest bias severity.}
%     \label{tab:ablation_tasks}
    
%     % Optional colors (kept consistent with reference style)
%     \definecolor{gStrong}{RGB}{198, 228, 201}
%     \definecolor{gMid}{RGB}{225, 240, 200}
    
%     \setlength{\tabcolsep}{4pt}
%     \renewcommand{\arraystretch}{0.8}
%     \footnotesize
    
%     \resizebox{0.65\columnwidth}{!}{%
%     \begin{tabular}{lcc|cc}
%         \toprule
%         \multicolumn{3}{c|}{\textbf{Configuration}} & \multicolumn{2}{c}{\textbf{Metrics}} \\
%         \cmidrule(lr){1-3} \cmidrule(lr){4-5}
%         Method & CLA & LPBC & Bias Severity ($\downarrow$) & Task Acc. ($\uparrow$) \\
%         \midrule
%         Baseline &  &  & 0.42 & 87.1 \\
%         CLA Only & \checkmark &  & 0.28 & 88.4 \\
%         LPBC Only &  & \checkmark & 0.31 & 87.9 \\
%         \midrule
%         \textbf{Ours (Combined)} & \checkmark & \checkmark & 
%         \cellcolor{gStrong}\textbf{0.15} & 
%         \cellcolor{gStrong}\textbf{89.2} \\
%         \bottomrule
%     \end{tabular}%
%     }
% \end{table}

\paragraph{Impact of Back-Translation Systems}
The CLA task relies on back-translated data to approximate the "translationese" manifold. A critical question is whether the choice of translation system influences the bias mitigation capabilities. We compared our default system (NLLB-200-3.3B) against a suite of varying architectures, including Gemma-3-4B, Llama-3.1-8B-Instruct, Qwen3-4B, GPT-4o, Gemini 2.5 Flash, and Google Translate.

As shown in Table~\ref{tab:ablation_systems}, while stronger systems (e.g., GPT-4o, Google Translate) achieve higher BLEU scores ($>50$), higher translation quality does not necessarily correlate with lower bias severity in the downstream task. This indicates that the CLA task is robust to the generator's quality, provided the generator produces sufficient translationese artifacts to serve as a negative contrastive pivot.

\begin{table}[t]
    \centering
    \caption{\textbf{Impact of Back-Translation Systems.} We report BLEU scores of the generated back-translations and the resulting Bias Severity.}
    \label{tab:ablation_systems}
    \begin{tabular}{lcc}
        \toprule
        \textbf{Translation System} & \textbf{BLEU} $\uparrow$ & \textbf{Bias Severity} $\downarrow$ \\
        \midrule
        NLLB-200-3.3B~\citep{costa2022no} (Default) & 52.4 & 0.15 \\
        Qwen3-4B~\citep{yang2025qwen3} & 54.1 & 0.16 \\
        Gemma-3-4B~\citep{team2025gemma} & 55.8 & 0.14 \\
        Llama-3.1-8B-Instruct~\citep{grattafiori2024llama} & 56.2 & 0.18 \\
        Gemini-2.5-Flash~\citep{comanici2025gemini} & 57.9 & 0.15 \\
        GPT-4o~\citep{hurst2024gpt} & 58.5 & 0.17 \\
        Google Translate & 59.1 & 0.14 \\
        \bottomrule
    \end{tabular}
\end{table}

\paragraph{Heuristic Signals for Bin Classification}
Finally, we ablate the heuristic metric used to partition samples for the LPBC task. While our method uses Negative Log-Likelihood (NLL), we compare this against Type-Token Ratio (TTR) and Perplexity (PPL).

Table~\ref{tab:ablation_metrics} demonstrates that while NLL yields the best performance marginally, the differences are negligible. All three metrics effectively capture the confidence disparity required for the auxiliary task, demonstrating that our method is agnostic to the specific heuristic used to approximate predictive confidence.

\begin{table}[t]
    \centering
    \caption{\textbf{Comparison of Heuristic Metrics.} Ablation of the metric used for Bin Classification. NLL performs slightly better, but all metrics provide similar bias mitigation.}
    \label{tab:ablation_metrics}
    \begin{tabular}{lc}
        \toprule
        \textbf{Metric} & \textbf{Bias Severity} $\downarrow$ \\
        \midrule
        Type-Token Ratio (TTR) & 0.17 \\
        Perplexity (PPL) & 0.16 \\
        Negative Log-Likelihood (NLL) & 0.15 \\
        \bottomrule
    \end{tabular}
\end{table}

\subsection{Sensitivity to Cross-Lingual Alignment Discrepancy (CAD)}
\label{sec:cad_sensitivity}

To test the hypothesis that conventional automatic judges exhibit an \emph{English-anchoring bias}—i.e., a preference for translations that closely mirror the syntactic structure of the English source—we analyze model performance as a function of Cross-Lingual Alignment Discrepancy (CAD). As introduced in \S\ref{sec:preliminary_analysis}, CAD measures the degree of structural divergence between a candidate translation and its source sentence.

We partition the held-out evaluation set into disjoint CAD bins (e.g., $[0.0,0.1), [0.1,0.2), \ldots$) and compute the win rate of Translationese outputs within each interval. Figure~\ref{fig:cad_sensitivity} summarizes the resulting trends. Baseline judges exhibit a pronounced positive correlation between CAD and win rate, indicating that their preferences are increasingly influenced by surface-level alignment artifacts rather than semantic adequacy. In contrast, \textsc{DIBJudge} demonstrates substantially reduced sensitivity to CAD, as evidenced by a near-flat win-rate profile across bins. This invariance suggests that \textsc{DIBJudge} successfully decouples evaluation quality from structural isomorphism, mitigating spurious biases that confound existing evaluation approaches.

\begin{figure}[t]
    \centering
    \includegraphics[width=.85\linewidth]{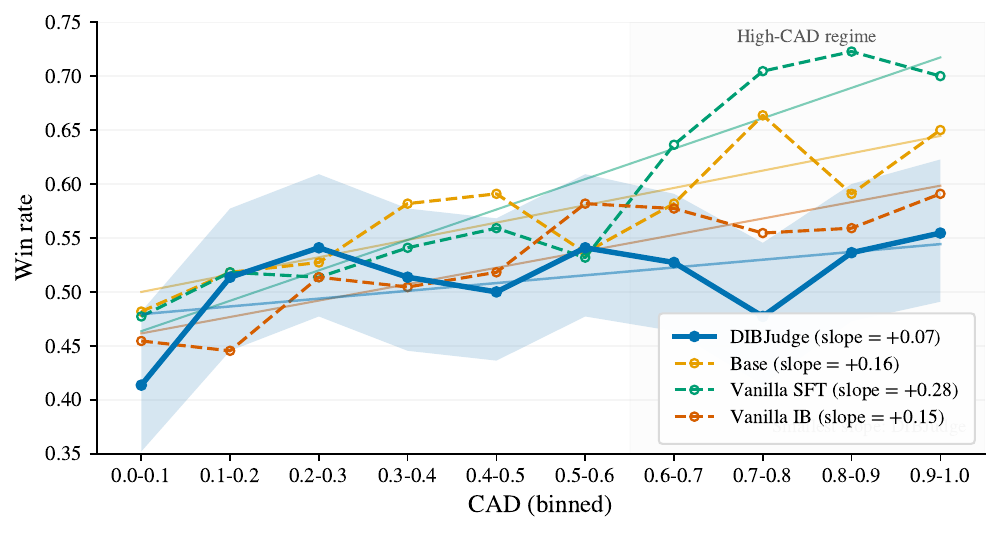}
    \caption{\textbf{Sensitivity of Win Rate to Cross-Lingual Alignment Discrepancy (CAD).}
    Win rates of Translationese outputs are plotted against binned CAD values. Baseline judges (dashed lines) display a strong positive dependence on CAD, reflecting an English-anchoring bias toward structurally aligned translations. In contrast, \textsc{DIBJudge} (solid line) maintains a near-constant win rate across CAD regimes, including high-CAD regions, indicating robustness to cross-lingual structural divergence.}
    \label{fig:cad_sensitivity}
\end{figure}

\subsection{Analysis of Distributional Shortcuts}
\label{sec:bias_spectrum}

To assess whether our approach mitigates reliance on superficial statistical artifacts, we analyze the \emph{bias spectrum shift} induced by different judges using the Sequence Surprisal Ratio (SSR). SSR measures the relative predictability of machine-generated responses compared to human-written ones, thereby capturing spurious correlations arising from differences in predictive confidence (e.g., perplexity). Values of SSR close to $1$ indicate distributional parity, where model preferences are not driven by confidence-related artifacts, whereas lower values reflect an over-reliance on highly predictable, low-entropy text.

We focus on evaluation instances for which the judge selects the machine-generated translation (denoted as \emph{Machine Wins}) and examine the empirical distribution of their SSR scores. Figure~\ref{fig:ssr_kde} presents a kernel density estimate (KDE) of these scores, revealing how different judges respond to distributional discrepancies associated with predictive confidence.

As shown in Figure~\ref{fig:ssr_kde}, the baseline judge exhibits a pronounced leftward shift in SSR, with probability mass concentrated at low values. This behavior indicates a \emph{distributional shortcut}, wherein outputs with artificially high predictive confidence are systematically favored, independent of semantic quality. In contrast, \textsc{DIBJudge} induces a clear re-centering of the SSR distribution toward $1$. This shift indicates that \textsc{DIBJudge} substantially reduces spurious correlations between model preference and predictive confidence, promoting judgments that are invariant to low-perplexity artifacts and more reflective of semantic utility.

\begin{figure}[t]
    \centering
    \includegraphics[width=0.85\linewidth]{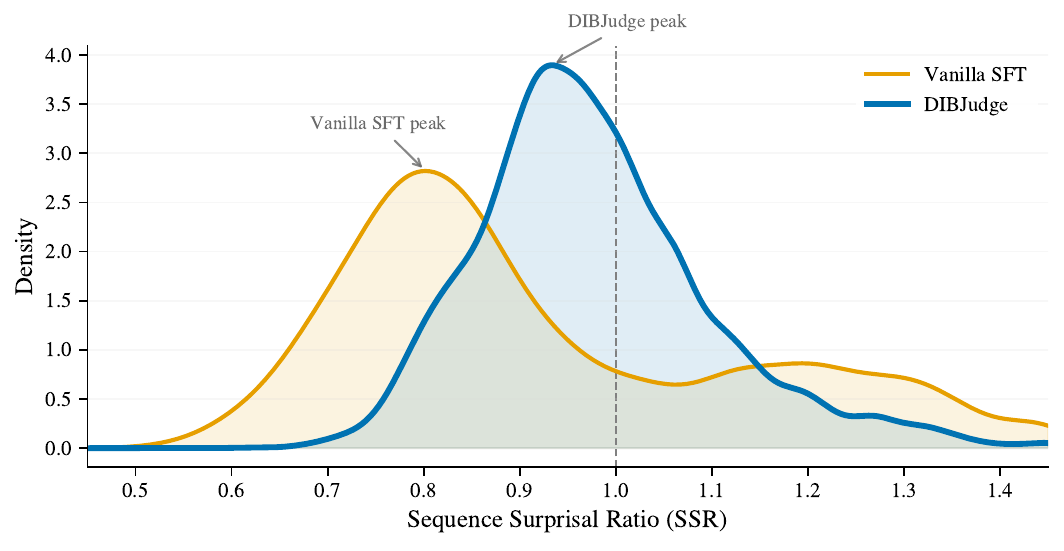}
    \caption{\textbf{Bias Spectrum Shift via Sequence Surprisal Ratio (SSR).}
    Kernel density estimates of SSR for instances where the judge selects the machine-generated output (\emph{Machine Wins}). Lower SSR values correspond to spurious correlations with high predictive confidence (i.e., low perplexity), indicating a preference for distributionally simple text. Values near $1$ reflect invariance to confidence-related artifacts. While the baseline judge concentrates mass at low SSR values, \textsc{DIBJudge} shifts the spectrum toward $1$, demonstrating reduced reliance on predictive-confidence shortcuts.}
    \label{fig:ssr_kde}
\end{figure}

\subsection{Ablation Study: Efficacy of Information Bottlenecks}
\label{sec:ablation_bottleneck}

A central hypothesis of our work is that a variational information constraint provides a superior trade-off between semantic preservation and artifact suppression compared to discrete or deterministic alternatives. To evaluate this, we compare our framework against three representative bottleneck mechanisms, assessing their impact on the robustness-utility Pareto frontier:

\begin{itemize}
    \item \textbf{Vector Quantization (VQ):} We replace the continuous variational information constraint with a discrete codebook constraint \citep{van2017neural}, mapping latent representations to the nearest centroid. This imposes a rigid structural bottleneck.
    \item \textbf{Low-Rank Projection (Low-Rank):} We utilize a deterministic linear projection to a lower-dimensional subspace. This variant relies solely on capacity reduction without stochastic regularization.
    \item \textbf{Stochastic Noise (Noise):} We apply isotropic Gaussian noise injection, $\mathbf{z} = \mathbf{h} + \epsilon$, as a baseline stochastic regularization technique, notably lacking a prior distribution constraint.
\end{itemize}

\begin{table}[t]
    \centering
    \caption{
    \textbf{Comparison of Bottleneck Mechanisms.}
    Effect of different latent constraints on \emph{utility}
    (m-RewardBench accuracy) and \emph{bias mitigation}
    (Bias Severity). Lower Bias Severity is better.
    }
    \label{tab:bottleneck_ablation}
    \vspace{0.4em}
    \begin{small}
    \setlength{\tabcolsep}{7pt}
    \renewcommand{\arraystretch}{1.15}
    \begin{tabular}{l l c c}
        \toprule
        \textbf{Method} &
        \textbf{Constraint Type} &
        \textbf{m-RB Acc.} $\uparrow$ &
        \textbf{Bias Sev.} $\downarrow$ \\
        \midrule
        Low-Rank &
        Deterministic (Capacity) &
        88.2 &
        0.185 \\

        Noise &
        Stochastic (Additive) &
        85.5 &
        0.182 \\

        VQ &
        Discrete (Structural) &
        84.9 &
        \textbf{0.082} \\
        \midrule
        \rowcolor{gray!8}
        \textbf{Ours} &
        \textbf{Variational (Information)} &
        \textbf{89.3} &
        \textbf{0.091} \\
        \bottomrule
    \end{tabular}
    \end{small}
    \vspace{-0.6em}
\end{table}

\paragraph{Analysis.} Table~\ref{tab:bottleneck_ablation} demonstrates that while the \textbf{Low-Rank} variant preserves accuracy, it fails to filter translationese artifacts, suggesting that dimensionality reduction alone cannot achieve effective disentanglement. Conversely, \textbf{VQ} achieves the lowest bias severity but suffers from \textit{utility collapse}, where the rigid discrete constraint discards the nuanced semantic features necessary for reward modeling. 

Our variational information constraint method bridges this gap by explicitly optimizing the $I(Z; X)$ vs. $I(Z; Y)$ trade-off. By penalizing $I(Z; X)$ via the KL divergence to a prior, VIB selectively purges nuisance factors—such as translationese artifacts—while retaining task-relevant style features. Consequently, VIB dominates the noise and low-rank baselines in robustness while maintaining competitive utility.

\subsection{Ablation Study: Disentanglement Mechanisms}
\label{sec:ablation_disentangle_mechanisms}

We evaluate the efficacy of our proposed cross-covariance penalty against other methods. Our primary hypothesis is that a cross-covariance-based constraint provides a computationally efficient proxy for independence, effectively minimizing mutual information without the overhead of complex density estimators.

\paragraph{Alternative Disentanglement Objectives.} 
We compare our cross-covariance approach, $\mathcal{L}_{\text{cov}}$, against three established baseline objectives:
\begin{itemize}
    \item \textbf{Hilbert-Schmidt Independence Criterion (HSIC)~\citep{gretton2005hsic}:} A kernel-based measure of dependence. While theoretically robust, its $O(n^2)$ complexity per batch is prohibitive for long-context LLM fine-tuning.
    \item \textbf{Mutual Information Estimators (CLUB~\citep{cheng2020club}/MINE~\citep{belghazi2018mutual}):} Variational upper bounds on Mutual Information (MI). These require auxiliary neural networks, increasing the parameter search space and training time.
    \item \textbf{Orthogonality Constraint ($\mathcal{L}_{\text{orth}}$):} A first-order geometric constraint minimizing absolute cosine similarity:
    \begin{equation}
        \mathcal{L}_{\text{orth}} = \mathbb{E} \left[ \frac{|\mathbf{z}_r^\top \mathbf{z}_b|}{\|\mathbf{z}_r\|_2 \|\mathbf{z}_b\|_2} \right].
    \end{equation}
\end{itemize}

\paragraph{Quantitative Analysis of Efficiency vs. Robustness.}
The trade-off between disentanglement strength and computational overhead is summarized in Table~\ref{tab:efficiency_bias}. Our analysis reveals that while variational estimators (CLUB, MINE) and kernel-based methods (HSIC) theoretically offer tighter bounds on independence, their integration into LLM architectures is bottlenecked by the high-dimensional nature of the hidden states. Specifically, the auxiliary network in CLUB introduces a $112\%$ increase in training latency due to the additional forward-backward passes required for the critic update. 

In contrast, our cross-covariance approach, $\mathcal{L}_{\text{cov}}$, achieves a Bias Severity score of $14.2$, which is competitive with the $13.8$ achieved by HSIC, but at a fraction of the computational cost ($1.2\times$ vs $2.8\times$ latency). We observe that simple orthogonality ($\mathcal{L}_{\text{orth}}$) suffers from significant "information leakage," as evidenced by its high Bias Severity ($21.5$); this confirms that first-order geometric constraints are insufficient to capture the complex, non-linear correlations inherent in LLM representations. By minimizing the cross-covariance, we achieve a second-order alignment that serves as a "practical optimum"—sufficiently decorrelating the robust and biased subspaces without the prohibitive $O(n^2)$ complexity or training instability of higher-order estimators.
 
\begin{table}[t]
\centering
\caption{
\textbf{Efficiency--bias trade-off of disentanglement mechanisms on Llama-3-8B.}
Training latency is normalized to standard fine-tuning ($1.0\times$).
Bias Severity is defined in \S~\ref{sec:preliminary_analysis}, where lower is better.
}
\label{tab:efficiency_bias}
\vspace{0.4em}
\begin{small}
\setlength{\tabcolsep}{8pt}
\renewcommand{\arraystretch}{1.15}
\begin{tabular}{l c c c}
\toprule
\textbf{Method} &
\textbf{Computational Cost} &
\textbf{Latency} $\uparrow$ &
\textbf{Bias Severity} $\downarrow$ \\
\midrule
Baseline (No disentanglement) &
-- &
$1.00\times$ &
0.284 \\

$\mathcal{L}_{\text{orth}}$ (Orthogonality) &
$O(d)$ &
\textbf{1.05$\times$} &
0.215 \\

MINE \citep{belghazi2018mutual} &
$O(\text{aux.\ net})$ &
2.42$\times$ &
0.145 \\

CLUB \citep{cheng2020club} &
$O(\text{aux.\ net})$ &
2.12$\times$ &
0.141 \\

HSIC \citep{gretton2005hsic} &
$O(n^2)$ &
2.78$\times$ &
\textbf{0.138} \\
\midrule
\rowcolor{gray!8}
\textbf{$\mathcal{L}_{\text{cov}}$ (Ours)} &
$O(d^2)$ &
\textbf{1.18$\times$} &
\textbf{0.142} \\
\bottomrule
\end{tabular}
\end{small}
\vspace{-0.6em}
\end{table}

\subsection{Quantitative Analysis: Linear Probing for Information Leakage}
\label{sec:linear_probing}

To quantify the degree of disentanglement achieved by \textsc{DIBJudge}, we employ linear probing \citep{alain2016understanding}. Our hypothesis is that a truly disentangled robust representation, $\mathbf{z}_r$, should be invariant to the text origin (Human vs. Machine-translated). Conversely, the bias representation, $\mathbf{z}_b$, should explicitly encode these ``translationese'' artifacts.

\paragraph{Experimental Setup.} We freeze the \textsc{DIBJudge} encoder and train a linear classifier (the probe) on the extracted representations. The probe is optimized to distinguish the two domains via binary classification. We report probing accuracy on a held-out test set; an accuracy of $50\%$ (random chance) signifies perfect invariance, whereas higher accuracy indicates significant information leakage.

\paragraph{Results.} As shown in Table~\ref{tab:probing_results}, the probe trained on the bias representation $\mathbf{z}_b$ achieves near-perfect accuracy ($96.1\%$), confirming that \textsc{DIBJudge} successfully isolates translation artifacts. More importantly, the probe trained on the robust representation $\mathbf{z}_r$ yields an accuracy of approximately $50\%$, demonstrating that predictive features are effectively sanitized of domain-specific signals. In contrast, standard embeddings from the baseline SFT model exhibit substantial leakage ($82.4\%$), highlighting that standard fine-tuning fails to decouple semantic content from stylistic artifacts.

\begin{table}[t]
    \centering
    \caption{
    \textbf{Linear probing for domain classification.}
    Higher accuracy indicates stronger domain information.
    Effective disentanglement yields \emph{high} accuracy for bias
    representations and \emph{low} accuracy for robust representations.
    }
    \label{tab:probing_results}
    \vspace{0.4em}
    \begin{small}
    \setlength{\tabcolsep}{8pt}
    \renewcommand{\arraystretch}{1.15}
    \begin{tabular}{l l c}
        \toprule
        \textbf{Model} &
        \textbf{Probed Representation} &
        \textbf{Accuracy (\%)} \\
        \midrule
        Baseline (SFT) &
        Standard embedding $\mathbf{h}$ &
        82.4 \\
        \midrule
        \multirow{2}{*}{\textbf{\textsc{DIBJudge} (Ours)}} &
        Bias representation $\mathbf{z}_b$  &
        \textbf{96.1} \\
        &
        Robust representation $\mathbf{z}_r$ &
        \textbf{50.3} \\
        \bottomrule
    \end{tabular}
    \end{small}
    \vspace{-0.6em}
\end{table}

%%%%%%%%%%%%%%%%%%%%%%%%%%%%%%%%%%%%%%%%%%%%%%%%%%%%%%%%%%%%%%%%%%%%%%%%%%%%%%%
%%%%%%%%%%%%%%%%%%%%%%%%%%%%%%%%%%%%%%%%%%%%%%%%%%%%%%%%%%%%%%%%%%%%%%%%%%%%%%%

\end{document}